\pdfoutput=1

\documentclass[twoside,11pt]{article}

\usepackage{url}
\usepackage{subfig}
\usepackage[ruled]{algorithm}
\usepackage[noend]{algpseudocode}
\usepackage{amsmath}
\usepackage{amsfonts}
\usepackage{amssymb}
\usepackage{amsthm}
\usepackage{enumerate}
\usepackage{blkarray}
\usepackage{array}
\usepackage{comment}
\usepackage{graphicx}
\usepackage{natbib}
\usepackage{fancyhdr}
\usepackage{authblk}

\newcommand{\scll}{0.4}
\graphicspath{{./fig/}}
\newcommand{\ct}{\citet}
\newcommand{\cp}{\citep}
\newcommand{\ca}{\citeauthor}
\newcommand{\cy}{\citeyear}
\newcommand{\ctp}[1]{\ca{#1}'s~\citeyearpar{#1}}
\newcommand{\cwp}[1]{\ca{#1}, \cy{#1}}

\newcommand{\kera}{\ensuremath{\kappa^{a}_{\tau}}}
\newcommand{\kerb}{\ensuremath{\bar{\kappa}_{\bar{\tau}}}}

\newcommand{\Sc}{\ensuremath{\mathbb{S}}}
\newcommand{\Sb}{\ensuremath{\bar{S}}}
\newcommand{\Sca}{\ensuremath{S^{a}}}
\newcommand{\Scb}{\ensuremath{S^{b}}}

\newcommand{\E}[1]{\ensuremath{E\left[#1\right]}}

\newcommand{\xia}{\ensuremath{{s}^{a}_{i}}}
\newcommand{\xja}{\ensuremath{{s}^{a}_{j}}}

\newcommand{\xka}{\ensuremath{{s}^{a}_{k}}}
\newcommand{\xta}{\ensuremath{{s}^{a}_{t}}}
\newcommand{\xla}{\ensuremath{{s}^{a}_{l}}}
\newcommand{\xib}{\ensuremath{{s}^{b}_{i}}}

\newcommand{\yia}{\ensuremath{\hat{{s}}^{a}_{i}}}
\newcommand{\yib}{\ensuremath{\hat{{s}}^{b}_{i}}}

\newcommand{\yka}{\ensuremath{\hat{{s}}^{a}_{k}}}

\newcommand{\yta}{\ensuremath{\hat{{s}}^{a}_{t}}}

\newcommand{\ria}{\ensuremath{{r}^{a}_{i}}}
\newcommand{\rja}{\ensuremath{{r}^{a}_{j}}}
\newcommand{\rka}{\ensuremath{{r}^{a}_{k}}}
\newcommand{\mk}{\ensuremath{\phi}}
\newcommand{\mkb}{\ensuremath{\bar{\phi}}}
\newcommand{\gauss}{\ensuremath{\mathrm{k}}}
\newcommand{\gaussa}{\ensuremath{\gauss_{\tau}}}
\newcommand{\gaussb}{\ensuremath{\bar{\gauss}_{\bar{\tau}}}}

\newcommand{\rs}{\ensuremath{\bar{s}}}
\newcommand{\Pban}{\matii{\bar{P}}{_{S_1}}{}}
\newcommand{\Pbann}{\matii{\bar{P}}{_{S_1 \cup S_2}}{}}

\newcommand{\Pta}{\matii{\tilde{P}}{a}{}}

\newcommand{\Pxat}{\matii{\check{P}}{a}{t}}

\newcommand{\R}{\ensuremath{\mathbb{R}}}

\newcommand{\mat}[1]{\ensuremath{\mathbf{#1}}}
\newcommand{\matii}[3]{\ensuremath{\mathbf{#1}^{#2}_{#3}}}

\renewcommand{\P}{\mat{P}}
\newcommand{\D}{\mat{D}}
\newcommand{\K}{\mat{K}}
\newcommand{\Dt}{\matii{D}{}{t}}
\renewcommand{\E}{\mat{E}}
\newcommand{\Pa}{\matii{P}{a}{}}

\newcommand{\Da}{\matii{D}{a}{}}
\newcommand{\Ka}{\matii{K}{a}{}}
\newcommand{\Kat}{\matii{K}{a}{t}}

\newcommand{\Dda}{\matii{\dot{D}}{a}{}}
\newcommand{\Kda}{\matii{\dot{K}}{a}{}}
\newcommand{\ra}{\matii{r}{a}{}}

\newcommand{\rba}{\matii{\bar{r}}{a}{}}
\newcommand{\rbat}{\matii{\bar{r}}{a}{t}}

\newcommand{\rta}{\matii{\tilde{r}}{a}{}}

\newcommand{\rxat}{\matii{\check{r}}{a}{t}}

\newcommand{\rban}{\matii{\bar{r}}{_{S_1}}{}}
\newcommand{\rbann}{\matii{\bar{r}}{_{S_1 \cup S_2}}{}}

\newcommand{\rca}{\matii{\hat{r}}{a}{}}
\newcommand{\rcat}{\matii{\hat{r}}{a}{t}}

\newcommand{\rxa}{\matii{\check{r}}{a}{}}

\newcommand{\Pb}{\mat{\bar{P}}}
\newcommand{\Pba}{\matii{\bar{P}}{a}{}}
\newcommand{\Pbat}{\matii{\bar{P}}{a}{t}}

\newcommand{\Pca}{\matii{\hat{P}}{a}{}}
\newcommand{\Pcat}{\matii{\hat{P}}{a}{t}}

\newcommand{\Pxa}{\matii{\check{P}}{a}{}}
\newcommand{\Pcb}{\matii{\hat{P}}{b}{}}

\renewcommand{\u}{\mat{u}}
\newcommand{\vo}{\matii{v}{*}{}}
\newcommand{\vbo}{\matii{\bar{v}}{*}{}}
\newcommand{\vto}{\matii{\tilde{v}}{*}{}}

\newcommand{\vho}{\matii{\check{v}}{*}{}}

\newcommand{\vco}{\matii{\hat{v}}{*}{}}
\newcommand{\vxo}{\matii{\check{v}}{*}{}}

\newcommand{\Qco}{\matii{\hat{Q}}{*}{}}
\newcommand{\Qxo}{\matii{\check{Q}}{*}{}}

\newcommand{\Q}{\mat{Q}}
\newcommand{\Qo}{\matii{Q}{*}{}}
\newcommand{\Rd}{\ensuremath{\Gamma}}
\renewcommand{\v}{\mat{v}}
\newcommand{\Ex}{\ensuremath{\Delta}}

\newcommand{\qa}{\matii{q}{a}{}}
\newcommand{\hT}{\ensuremath{\check{T}}}
\newcommand{\hEx}{\ensuremath{\check{\Delta}}}
\newcommand{\bT}{\ensuremath{\bar{T}}}

\newcommand{\La}{\matii{L}{a}{}}
\newcommand{\cE}{\ensuremath{\mathcal{E}}}
\renewcommand{\H}{\mat{H}}
\renewcommand{\t}{\mathrm{\intercal}}

\newcommand{\rdda}{\ensuremath{\bar{\bar{\mathbf{r}}}^{a}}}

\newcommand{\Pdda}{\ensuremath{\bar{\bar{\mathbf{P}}}^{a}}}

\newcommand{\vddo}{\ensuremath{\bar{\bar{\mathbf{v}}}^{*}}}

\newcommand{\ddM}{\ensuremath{\bar{\bar{M}}}}
\newcommand{\cM}{\ensuremath{\hat{M}}}
\newcommand{\xM}{\ensuremath{\check{M}}}
\newcommand{\cS}{\ensuremath{\hat{S}}}
\newcommand{\cQ}{\ensuremath{\hat{Q}}}
\newcommand{\cV}{\ensuremath{\hat{V}}}
\newcommand{\cT}{\ensuremath{\hat{T}}}

\newcommand{\tQ}{\ensuremath{\tilde{Q}}}
\newcommand{\tV}{\ensuremath{\tilde{V}}}

\newcommand{\ddT}{\ensuremath{\bar{\bar{{T}}}}}

\newcommand{\ddEx}{\ensuremath{\bar{\bar{\Delta}}}}
\newcommand{\vt}{\mat{\tilde{v}}}
\newcommand{\Qb}{\mat{\bar{Q}}}
\newcommand{\Qbo}{\matii{\bar{Q}}{*}{}}
\newcommand{\Qto}{\matii{\tilde{Q}}{*}{}}

\newcommand{\Qddo}{\bar{\bar{\mathbf{Q}}}^{*}}

\newcommand{\qbao}{\matii{\bar{q}}{a}{*}}

\newcommand{\qddao}{\bar{\bar{\mathbf{q}}}^{a}_{*}}
\newcommand{\wa}{\matii{w}{a}{}}

\renewcommand{\th}{\ensuremath{^{\mathrm{th}}}}
\newcommand{\ith}{\ensuremath{i^{\mathrm{th}}}}
\newcommand{\jth}{\ensuremath{j^{\mathrm{th}}}}
\newcommand{\ath}{\ensuremath{a^{\mathrm{th}}}}

\newcommand{\tth}{\ensuremath{t^{\mathrm{th}}}}

\newcommand{\mne}{\ensuremath{\eta_{\min}}}

\newcommand{\hM}{\ensuremath{\check{M}}}

\newcommand{\bM}{\ensuremath{\bar{M}}}

\newcommand{\norm}[1]{\ensuremath{\| #1 \|}}
\newcommand{\infnorm}[1]{\ensuremath{\left\| #1\right\|_{\infty}}}
\newcommand{\infnorms}[1]{\ensuremath{\| #1\|_{\infty}}}

\newcommand{\mmin}[1]{\underset{#1}{\operatorname{min}}\;}
\newcommand{\mmax}[1]{\underset{#1}{\operatorname{max}}\;}

\newcommand{\nnz}{\ensuremath{\times}}
\newcommand{\la}{\ensuremath{\leftarrow}}

\newcommand{\argmax}[1]{\ensuremath{\mathop{\mathrm{argmax}}_{#1}}}

\newcommand{\dista}{\ensuremath{\psi}}

\newcommand{\ikbsf}{$i$KBSF}

\newcommand{\Vpi}{\mbox{$V^{\pi}$}}
\newcommand{\vpi}{\matii{v}{\pi}{}}

\newcommand{\T}{\ensuremath{T}}

\newcommand{\Qpi}{\mbox{$Q^{\pi}$}}

\newcommand{\dif}{\ensuremath{\mathrm{dif}}}

\newcommand{\st}{\ensuremath{s_{(t)}}}
\newcommand{\stp}{\ensuremath{s_{(t+1)}}}

\newcommand{\rtp}{\ensuremath{r_{(t+1)}}}

\newcommand{\taub}{\ensuremath{\bar{\tau}}}

\newcommand{\aij}{\ensuremath{z}}
\newcommand{\mkp}[1]{\mkb\left(#1\right)}
\newcommand{\maxv}{\ensuremath{\bar{\varsigma}^{\aij}_{\max}}}
\newcommand{\minv}{\ensuremath{{\varsigma}^{\aij}_{\min}}}
\newcommand{\rr}{\mat{{r}}}

\newcommand{\difajik}{\ensuremath{\varsigma^{\aij}_{k}}}

\newcommand{\nk}{\ensuremath{\mu}}
\newcommand{\nkb}{\ensuremath{\bar{\mu}}}

\newcommand{\ts}{\hspace{0.3mm}}

\newcommand{\kerk}{\ensuremath{\bar{\kappa}_{0}}}
\newcommand{\gaussk}{\ensuremath{\bar{\gauss}_{0}}}

\newcommand{\reptheo}[2]{\noindent{\bf #1}{\it #2}\vspace{2mm}}

\newtheorem{definition}{Definition}
\newtheorem{proposition}{Proposition}
\newtheorem{lemma}{Lemma}

    \makeatletter
    \renewcommand*{\@fnsymbol}[1]{\ensuremath{\ifcase#1\or \dagger\or
\ddagger\or
       \mathsection\or \mathparagraph\or \|\or **\or \dagger\dagger
       \or \ddagger\ddagger \else\@ctrerr\fi}}
    \makeatother

\long\def\symbolfootnote[#1]#2{\begingroup%
\def\thefootnote{\fnsymbol{footnote}}\thanks[#1]{#2}\endgroup}

\newcommand{\dims}{\ensuremath{d_{\Sc}}}

\newcommand{\kfunca}{\ensuremath{\hat{\mathcal{P}}_{\Sca}}}
\newcommand{\kfuncb}{\ensuremath{\bar{\mathcal{P}}_{\Sca}}}

\newcommand{\maxinf}[2]{\ensuremath{\mmax{a}{\infnorms{#1 - #2}}}}
\newcommand{\levelD}{\ensuremath{\sigma(\D)}}
\newcommand{\levelDc}{\ensuremath{\sigma(\Dt)}}
\newcommand{\boundV}{\ensuremath{\xi_{v}}}
\newcommand{\boundVV}{\ensuremath{\xi_{v}^{\prime}}}

\newcommand{\A}{\mat{A}}
\newcommand{\sprime}{\ensuremath{s^{_{\prime}}}}

\newcommand{\neib}{\ensuremath{{rs}}}
\newcommand{\ddb}{\ensuremath{{{dist}}}}

\begin{document}

\title{\vspace{-15mm}{\bf Practical Kernel-Based Reinforcement 
Learning}\thanks{
Parts of the material presented in this technical report have 
appeared before in two papers published in the {\sl Neural Information
Processing Systems}
conference~(NIPS,
\ca{barreto2011reinforcement}, \cy{barreto2011reinforcement}, 
\cy{barreto2012online}).
The current manuscript is a substantial extension of
the aforementioned works.}
}

\author{Andr\'{e} M. S. Barreto}
\affil{\vspace{-3mm} Laborat\'{o}rio Nacional de Computa\c{c}\~{a}o Cient\'{i}fica \\
Petr\'{o}polis, Brazil \vspace{-6mm}}

\author{Doina Precup}
\author{Joelle Pineau }
\affil{\vspace{-3mm}  McGill University \\ Montreal, Canada \vspace{-5mm}}


\date{}
\maketitle

\thispagestyle{empty}

\begin{abstract}
\emph{Kernel-based reinforcement learning} (KBRL) stands out
among approximate reinforcement learning algorithms for its strong theoretical
guarantees. By casting the learning problem as a local kernel approximation,
KBRL provides a way of computing a decision policy which
is statistically consistent and converges to a unique solution.
Unfortunately, the model constructed by KBRL grows with the number of sample
transitions, resulting in a computational cost that precludes 
its application to large-scale or on-line domains.
In this paper we introduce an algorithm that turns KBRL into a practical
reinforcement learning tool.
\emph{Kernel-based stochastic factorization} (KBSF) builds on a simple idea:
when a transition probability matrix is represented as the product of two
stochastic matrices, one can swap the factors of the multiplication to obtain
another transition matrix, potentially much smaller than the original, which
retains some fundamental properties of its precursor.
KBSF exploits such an insight to compress the information contained in KBRL's
model into an approximator of fixed size. This makes it possible to
build an approximation that takes into account both the difficulty of the
problem and the associated computational cost.
KBSF's computational complexity is linear in the number of sample
transitions, which is the best one can do without discarding data.
Moreover, the algorithm's simple mechanics allow for a fully incremental
implementation that makes the amount of memory used 
independent of the number of sample transitions.
The result is a kernel-based reinforcement learning algorithm that can be
applied to large-scale problems in both off-line and on-line regimes.
We derive upper bounds for the distance between the value
functions computed by KBRL and KBSF using the same data. We also prove that it
is possible to control the magnitude
of the variables appearing in our bounds, which means that, 
given enough computational resources, we can make KBSF's value function
as close as desired to the value function that would be
computed by KBRL using the same set of sample transitions.
The potential of our algorithm is demonstrated in
an extensive empirical study in which KBSF is applied to difficult tasks 
based on real-world data. 
Not only does KBSF solve problems that had never been solved before, 
it also significantly outperforms other 
state-of-the-art reinforcement learning algorithms on the tasks studied.
\end{abstract}

\section{Introduction}
\label{sec:introduction}

Reinforcement learning provides a conceptual framework with the potential to
materialize a long-sought goal in artificial intelligence: the construction of
situated agents that learn how to behave from direct interaction with the 
environment~\cp{sutton98reinforcement}.
But such an endeavor does not come without its challenges; among them, 
extrapolating the field's basic machinery to large-scale domains has been a
particularly persistent obstacle.

It has long been recognized that virtually any real-world application of
reinforcement learning must involve some form of approximation.
Given the mature stage of the supervised-learning theory, and 
considering the multitude of approximation techniques available today, 
this realization may not come across as a particularly worrisome issue at
first glance. However, it is well known that the sequential nature of the
reinforcement learning problem renders the incorporation of function
approximators non-trivial~\cp{bertsekas96neuro-dynamic}. 

Despite the difficulties,  in the last two decades
the collective effort of the reinforcement learning community has given rise to
many reliable approximate algorithms~\cp{szepesvari2010algorithms}.
Among them, \ctp{ormoneit2002kernelbased} kernel-based reinforcement learning
(KBRL) stands out for two reasons. First, unlike
other approximation schemes, KBRL always converges to a unique solution. Second,
KBRL is consistent in the statistical sense, meaning that adding more data
always improves the quality of the
resulting policy and eventually leads to optimal performance.

Unfortunately, the good theoretical properties of KBRL come at
a price: since the model constructed by this algorithm grows with the number of
sample transitions, the cost of computing a decision policy 
quickly becomes prohibitive as more data become available.
Such a computational burden severely limits the applicability of KBRL. 
This may help explain why, in spite of its nice theoretical guarantees,
kernel-based learning has not been widely adopted as a practical reinforcement
learning tool.

This paper presents an algorithm that can potentially change this situation.
\emph{Kernel-based stochastic factorization} (KBSF) builds on a simple idea:
when a transition probability matrix is represented as the product of two
stochastic matrices, one can swap the factors of the multiplication to obtain
another transition matrix, potentially much smaller than the original, which
retains some fundamental properties of its precursor~\cp{barreto2011computing}.
KBSF exploits this insight to compress the information contained in KBRL's
model into an approximator of fixed size. 
In other words, KBSF builds a model, whose size is independent of
the number of sample transitions, which serves as an approximation of 
the model that would be constructed by KBRL. Since 
the size of the model becomes a parameter of the algorithm, KBSF essentially
detaches the structure of KBRL's
approximator from its configuration. This extra flexibility makes it possible to
build an approximation that takes into account \emph{both} the difficulty of the
problem \emph{and} the computational cost of 
finding a policy using the constructed model.

KBSF's computational complexity is linear in the number of sample
transitions, which is the best one can do without throwing data away.
Moreover, we show in the paper that the amount of memory used by  our
algorithm is independent of the number of sample transitions. 
Put together, these two properties make it
possible to apply KBSF to large-scale problems in both off-line and on-line
regimes. To illustrate this possibility in practice, we present an
extensive empirical study in which KBSF is applied to difficult control tasks 
based on real-world data, some of which had never been solved 
before. KBSF outperforms least-squares policy iteration and fitted
$Q$-iteration on several off-line problems and SARSA on a difficult
on-line task.

We also show that KBSF is a sound algorithm from a theoretical point of view. 
Specifically, we derive results bounding the distance between the value
function computed by our algorithm and the one computed by KBRL
using the same data. We also prove that it is possible to control the magnitude
of the variables appearing in our bounds, which means that we can make the
difference between KBSF's and KBRL's solutions arbitrarily small.

We start the paper presenting some background material in
Section~\ref{sec:back}. Then, in Section~\ref{sec:sf}, we introduce the
stochastic-factorization trick, the insight underlying the development of our
algorithm. KBSF itself is presented in Section~\ref{sec:kbsf}. This section is
divided in two parts, one theoretical and one practical.
In Section~\ref{sec:theory_batch} we present theoretical results showing 
not only that the difference between KBSF's and KBRL's value functions
is bounded, but also that such a difference can be controlled.
Section~\ref{sec:empirical_batch} brings
experiments with KBSF on four reinforcement-learning problems: single and
double pole-balancing, HIV drug schedule domain, and epilepsy suppression task.
In Section~\ref{sec:ikbsf} we introduce the incremental version of our
algorithm, which can be applied to on-line problems.
This section follows the same structure of Section~\ref{sec:kbsf}, with
theoretical results followed by experiments. Specifically, in
Section~\ref{sec:theory_inc} we extend the results of
Section~\ref{sec:theory_batch} to the on-line scenario, and in
Section~\ref{sec:empirical_inc} we present experiments on the triple
pole-balancing and helicopter tasks. In Section~\ref{sec:discussion} we discuss
the impact of deviating from theoretical assumptions over KBSF's performance,
and also present a practical guide on how to configure our algorithm to solve a
reinforcement learning problem. In Section~\ref{sec:previous} we summarize
related works and situate KBSF in the context of kernel-based learning.
Finally, in Section~\ref{sec:conclusion} we present the main conclusions
regarding the current research and discuss some possibilities of
future work.

\section{Background}
\label{sec:back}

We consider the standard framework of reinforcement learning, in which an agent
interacts with an environment and tries to maximize the amount of
reward collected in the long run~\cp{sutton98reinforcement}. 
The interaction between agent and environment happens at
discrete time steps: at each instant $t$ the agent occupies a state 
$\st \in S$
and must choose an action $a$ from a finite set $A$. The sets $S$ and $A$ are
called the state and action spaces, respectively. 
The execution of action $a$ in state $\st$ moves the agent to a new
state $\stp$, where a new action must be selected, and so on. 
Each transition  has a certain probability of
occurrence and is associated with a reward 
$r \in \R$.  The goal of the agent is to find a policy $\pi:S \mapsto A$,
that is, a mapping from states to actions, that maximizes the expected
return. Here we define the return from time $t$ as:
\begin{equation}
\label{eq:sum_rd}
R_{(t)} = r_{(t+1)} + \gamma r_{(t+2)} + \gamma^{2} r_{(t+3)} +
... = \sum_{i=1}^{\infty} \gamma^{i-1} r_{(t+i)},
\end{equation}
where \rtp\ is the reward received at the transition from state $\st$ to state
$\stp$.
The parameter $\gamma \in [0,1)$ is the
discount factor, which determines the relative importance of individual
rewards depending on how far in the future they are received. 

\subsection{Markov Decision Processes}
\label{sec:mdp}

As usual, we assume that the interaction between agent and environment 
can be modeled as a \emph{Markov decision process} (MDP,
\cwp{puterman94markov}). 
An MDP is a tuple $M \equiv (S,A,P^{a},R^{a},\gamma)$, where $P^{a}$ and
$R^{a}$ describe the dynamics of the task at hand. For each action $a \in
A$, $P^{a}(\cdot|s)$ defines the next-state distribution upon taking action $a$
in state $s$. The reward received at transition $s \xrightarrow{a}
\sprime$ is given by $R^{a}(s,\sprime)$, with
$\left|R^{a}(s,\sprime)\right| \le R_{\max} < \infty$. Usually, one is
interested in the expected reward resulting from the execution of action
$a$ in state $s$, that is, 
$r^{a}(s) = E_{\sprime \sim P^{a}(\cdot|s)}\{R^{a}(s, \sprime)\}$.

Once the interaction between agent and environment has been modeled
as an MDP, a natural way of searching for an optimal policy is to
resort to \emph{dynamic programming}~\cp{bellman57dynamic}. 
Central to the theory of dynamic-programming is the concept of a 
\emph{value function}. The value of state
$s$ under a policy $\pi$, denoted by $\Vpi(s)$, is the expected return
the agent will receive from $s$ when following $\pi$, that is, 
$\Vpi(s) = E_{\pi}\{R_{(t)} | \st = s\}$
(here the expectation is over all possible sequences of rewards
in~(\ref{eq:sum_rd}) when the agent follows $\pi$).
Similarly, the value of the state-action pair $(s,a)$ under policy
$\pi$ is defined as 
$\Qpi(s,a) = E_{\sprime \sim P^{a}(\cdot|s)}\{R^{a}(s,\sprime) + \gamma \Vpi(\sprime)\}
= 
r^{a}(s) + \gamma E_{\sprime \sim P^{a}(\cdot|s)}\{\Vpi(\sprime)\}$.

The notion of value function makes it possible to impose a
partial ordering over decision policies. In particular, a policy $\pi'$ is
considered to be at least as good as another policy $\pi$ if
$V^{\pi'}(s) \ge V^{\pi}(s)$ for all $s \in S$.
The goal of dynamic programming is to find an optimal policy
$\pi^{*}$  that performs no worse than any other. It is well known that
there always exists at least one such policy for a given
MDP~\cp{puterman94markov}.
When there is more than one optimal policy, they all share the same
value function~$V^{*}$.

When both the state and action spaces are finite, an MDP can be
represented in matrix
form: each function $P^{a}$ becomes a matrix $\Pa \in \R^{|S| \times |S|}$, with
$p_{ij}^{a} = P^{a}(s_{j}|s_{i})$, and each function $r^{a}$ becomes a vector
$\ra \in \R^{|S|}$, where $r^{a}_{i}=r^{a}(s_{i})$. Similarly, $V^{\pi}$ can be
represented as a vector $\vpi \in \R^{|S|}$ and $Q^{\pi}$ can be seen as a
matrix $\matii{Q}{\pi}{} \in \R^{|S| \times |A|}$.
Throughout the paper we will use the conventional and matrix notations
interchangeably, depending on the context. When using the latter, 
vectors will be denoted by small boldface letters and matrices will be denoted
by capital boldface letters.

When the MDP is finite, dynamic
programming can be used to find an optimal decision-policy $\pi^{*} \in A^{|S|}$
in time polynomial in the number of states $|S|$ and actions 
$|A|$~\cp{ye2011simplex}. 
Let $\v \in \R^{|S|}$ and let $\Q \in \R^{|S| \times |A|}$. Define the
operator $\Rd: \R^{|S| \times |A|} \mapsto \R^{|S|}$ such that 
$\Rd\Q = \v$ if and only if 
$v_{i} = \max_{j} q_{ij}$ for all $i$. Also, 
given an MDP $M$, define 
$\Ex:  \R^{|S|} \mapsto \R^{|S| \times |A|}$ such that  
$\Ex\v = \Q$ if and only if  $q_{ia} = r^{a}_{i} + \gamma \sum_{j=1}^{|S|}
p^{a}_{ij} v_{j} \text{ for all } i \text{ and all } a$. 
The \emph{Bellman operator} of the MDP $M$ is given by $\T \equiv \Rd \Ex$.
A fundamental result in dynamic programming states that, starting from 
$\v^{(0)} = \mat{0}$, the expression 
$\v^{(t)} = \T \v^{(t-1)} = \Rd \Q^{(t)}$
gives the
optimal $t$-step value function, and as $t \rightarrow \infty$ 
the vector $\v^{(t)}$ approaches \vo. 
At any point, the optimal $t$-step policy can be obtained by selecting 
$\pi^{(t)}_{i} \in \argmax{j}{q^{(t)}_{ij}}$~\cp{puterman94markov}.

In contrast with dynamic programming, in reinforcement learning it is assumed
that the MDP is unknown,
and the agent must learn a policy based on transitions sampled from
the environment. 
If the process of learning a decision policy is based on a fixed set of
sample transitions, we call it \emph{batch} reinforcement learning. 
On the other hand, in
\emph{on-line} reinforcement learning the computation of a decision
policy takes place concomitantly with the collection of
data~\cp{sutton98reinforcement}.

\subsection{Kernel-Based Reinforcement Learning}
\label{sec:kbrl}

Kernel-based reinforcement learning (KBRL) 
is a batch algorithm that uses a finite model approximation
to solve a continuous MDP  $M \equiv (\Sc, A, P^{a}, R^{a}, \gamma)$, where $\Sc
\subseteq [0,1]^{\dims}$~\cp{ormoneit2002kernelbased}. 
Let $\Sca \equiv \{(\xka,\rka,\yka)| k = 1, 2, ..., n_{a}\}$
be sample transitions associated with action $a \in A$, where $\xka, \yka \in
\Sc$ and $\rka \in \R$.
Let $\mk: \R^{+} \mapsto \R^{+}$ be a Lipschitz continuous function
satisfying $\int_{0}^{1} \mk(x) dx = 1$.
Let $\gaussa(s,\sprime)$ be a kernel function defined as
\begin{equation}
\label{eq:gaussa}
\gaussa(s,\sprime) = 
{\mk \left(\frac{\norm{s - \sprime}}{ \tau} \right)},
\end{equation}
where $\tau \in \R$ and  
$\norm{\cdot}$ is a norm in $\R^{\dims}$ (for concreteness, the reader may
think of $\gaussa(s,\sprime)$ as the Gaussian
kernel, although the definition also encompasses other functions). 
Finally, define the normalized kernel function associated with action $a$ as
\begin{equation}
\label{eq:kera}
\kera(s,s_{i}^{a}) = 
\frac{\gaussa(s, s^{a}_{i})}{\sum_{j=1}^{n_{a}}
\gaussa(s, s^{a}_{j})} .
\end{equation}
KBRL uses~(\ref{eq:kera}) to build a finite MDP whose state space $\cS$ is 
composed solely of the $n = \sum_{a} n_{a}$ states
\yia\ (if a given state $s \in \Sc$ occurs more than
once in the set of sample transitions, each occurrence will be treated as a
distinct state in the finite MDP). The transition functions of KBRL's model, 
$\hat{P}^{a}:\cS \times \cS \mapsto [0,1]$,
are given by:
\begin{equation}
\label{eq:mdp_kbrl_P}
\hat{P}^{a}\left(\yib|s\right) = 
\left\{\begin{array}{l}
       \kera(s, \xib), \text{ if } a = b,\\
		 0, \mbox{ otherwise,}
      \end{array}
\right.
\end{equation}
where $a,b \in A$. Similarly, the reward functions of the MDP constructed
by KBRL, $\hat{R}^{a}: \cS \times \cS \mapsto \R$, are
\begin{equation}
\label{eq:mdp_kbrl_R}
\hat{R}^{a}(s, \yib) = 
\left\{\begin{array}{l}
       \ria, \mbox{ if } a = b,\\
		 0, \mbox{ otherwise. }
      \end{array}
\right.
\end{equation}
Based on~(\ref{eq:mdp_kbrl_P}) and~(\ref{eq:mdp_kbrl_R}) we can define the 
transition matrices and expected-reward vectors of KBRL's MDP.
The matrices \Pca\ are derived directly from the definition of 
$\hat{P}^{a}(\yib|s)$. The vectors of expected rewards \rca\ are computed as
follows.
Let $\rr \equiv [(\matii{{r}}{1}{})^{\t},
\matii{({r}}{2}{})^{\t}, ..., (\matii{{r}}{|A|}{})^{\t}]^{\t} \in \R^{n}$,
where $\ra \in \R^{n_{a}}$ are the vectors composed of the sampled rewards
$r^{a}_{i}$.
Since $R^{a}(s,\yib)$ does not depend on the start state $s$,
we can write
\begin{equation}
\label{eq:rca}
\rca = \Pca \rr.
\end{equation}
KBRL's MDP is thus given by $\cM \equiv (\hat{S}, A, \Pca, \rca, \gamma)$.

Once \cM\ has been defined, one can use dynamic programming to compute its
optimal value function $\hat{V}^{*}$.
Then, the value of any state-action pair of the continuous MDP can be determined
as: 
\begin{equation}
\label{eq:kbrl_q}
\cQ(s,a) = \sum_{i=1}^{n_{a}} \kera(s, \xia) \left[\ria +
\gamma \hat{V}^{*}(\yia) \right],
\end{equation}
where $s \in \Sc$ and $a \in A$.
\ct{ormoneit2002kernelbased} have shown that, if $n_{a} \rightarrow
\infty$ for all $a \in A$ and the kernel's width $\tau$ shrink at an
``admissible'' rate, the probability of choosing a suboptimal action based
on $\cQ(s,a)$ converges to zero (see their Theorem~4).

As discussed, using dynamic programming one can compute the optimal value
function of \cM\ in time polynomial in the number of sample transitions $n$
(which is also the number of states in \cM).
However, since each application of the Bellman operator $\cT$ is $O(n^2|A|)$,
the computational cost of such a procedure can easily become
prohibitive in practice. 
Thus, the use of KBRL leads to a dilemma: on the one hand one wants as much data
as possible to
describe the dynamics of the task, but on the other hand
the number of transitions should be small enough to allow for the numerical
solution of the resulting model. In the following sections we describe 
a practical approach to weight the relative
importance of these two conflicting objectives.

\section{Stochastic Factorization}
\label{sec:sf}

A \emph{stochastic matrix} has only nonnegative elements and each of its rows
sums to $1$. That said, we can introduce the concept that will serve as a
cornerstone for the rest of the paper:

\begin{definition}
 Given a stochastic matrix $\P \in \R^{n \times p}$, the relation $\P=\D\K$
is called a \emph{stochastic factorization} of \P\ if $\D \in \R^{n \times m}$
and $\K \in \R^{m \times p}$ are also stochastic matrices. The integer $m > 0$
is the \emph{order} of the factorization.
\end{definition}

This mathematical construct has been explored before.
For example,~\ct{cohen91nonnegative} briefly discuss it as a special case of
nonnegative matrix factorization,
while~\ct{cutler94archetypal} focus on slightly modified
versions of the stochastic factorization for statistical data analysis. However,
in this paper we will focus on a useful property of this type of factorization
that has only recently 
been noted~\cp{barreto2011computing}.

\subsection{Stochastic-Factorization Trick}

Let $\P \in \R^{n \times n}$ be a transition matrix, that is, a
square stochastic matrix, and let $\P=\D\K$ be an order $m$ 
stochastic factorization. In this case, one can see the elements of \D\ and \K\
as probabilities of transitions between the states $s_{i}$ and  a set of $m$
artificial states $\bar{s}_{h}$. Specifically, the elements in
each row of \D\ can be interpreted as probabilities of transitions from the 
original states to the artificial states, while the rows of \K\
can be seen as probabilities of transitions in the opposite
direction.
Under this interpretation, each element 
$p_{ij}=\sum_{h=1}^{m} d_{ih} k_{hj}$ is the sum of the probabilities associated
with $m$ two-step transitions: from state $s_i$ to each artificial state 
$\bar{s}_h$ and
from these back to state $s_{j}$. In other words, $p_{ij}$ is the accumulated
probability of all possible paths from $s_{i}$ to $s_{j}$ with a stopover in one
of the artificial states $\bar{s}_{h}$. Following similar reasoning, it is not
difficult to see that by \emph{swapping} the factors of a stochastic
factorization, that is, by
switching from $\D\K$ to $\K\D$, one obtains the transition probabilities
between the artificial states $\bar{s}_{h}$, $\mat{\bar{P}}=\K\D$. If $m < n$, 
$\Pb \in \R^{m \times m}$ will be a compact version of \P. 
Figure~\ref{fig:mat_trans} illustrates this idea for the
case in which $n=3$ and $m=2$.

\begin{figure*}
    \centering
    \begin{tabular}{cccc}
	 \includegraphics[scale=0.35]{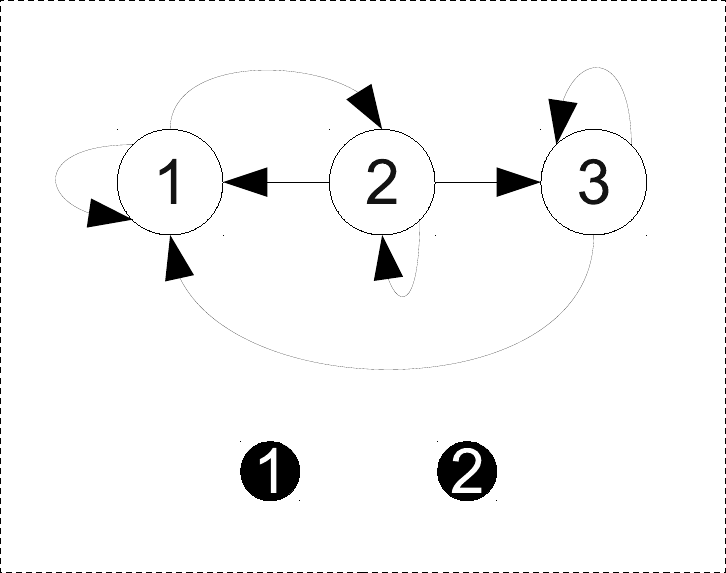} &
	 \includegraphics[scale=0.35]{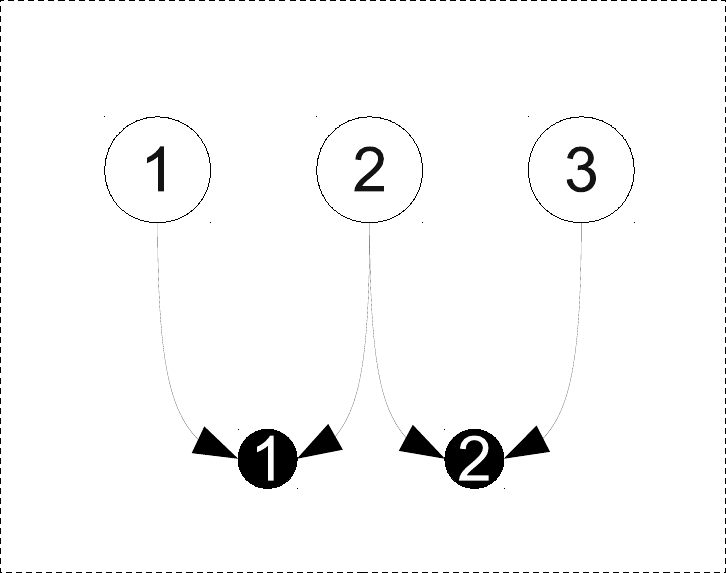} &
	 \includegraphics[scale=0.35]{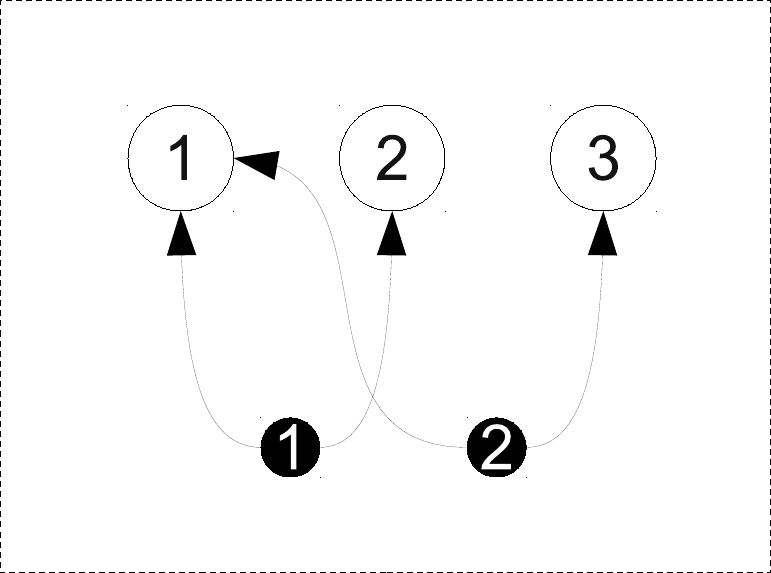} &
     	 \includegraphics[scale=0.35]{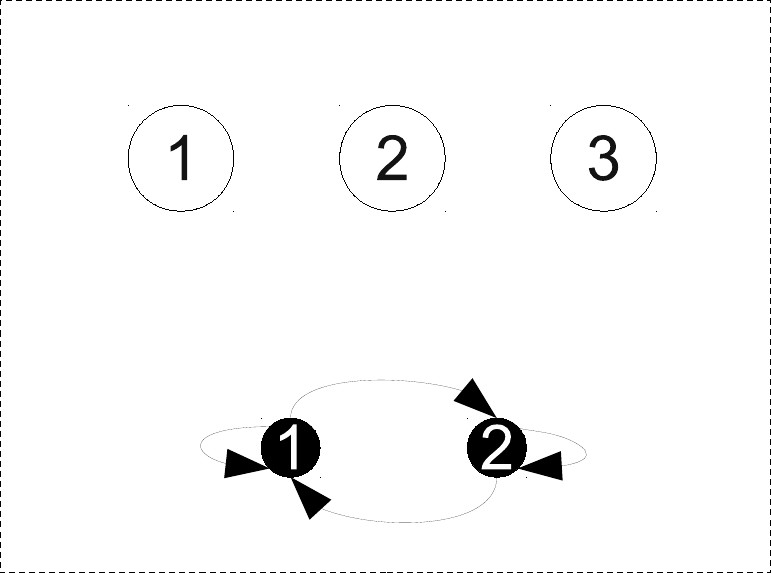} \\ \\
	  $ \P = \left[\begin{array}{ccc}
	             \nnz & \nnz & 0 \\
	             \nnz & \nnz & \nnz \\
	             \nnz & 0 & \nnz \\
	      \end{array}\right]$ &
	   $ \D = \left[\begin{array}{cc}
	             \nnz & 0 \\
	             \nnz & \nnz  \\
	              0 &  \nnz \\
	      \end{array}\right]$ &
	  $ \K = \left[\begin{array}{ccc}
	             \nnz & \nnz & 0 \\
	             \nnz & 0 & \nnz \\
	      \end{array}\right]$ &
	  $ \Pb = \left[\begin{array}{cc}
	             \nnz & \nnz  \\
	             \nnz & \nnz \\
	      \end{array}\right]$ \\
    \end{tabular}
\caption{Reducing the dimension of a transition model from $n=3$ states to $m=2$
artificial states. The original states $s_{i}$ are represented as big white
circles; small black circles depict artificial states $\rs_{h}$. The symbol
`$\nnz$' is used to represent nonzero elements. \label{fig:mat_trans}}
\end{figure*}

The stochasticity of \Pb\ follows immediately from the same property of \D\ and
\K. What is perhaps more surprising is the fact that this matrix 
shares some fundamental characteristics with the original matrix \P.
Specifically, it is possible to show that: $(i)$ for each
recurrent class in \P\ there is a corresponding class in \Pb\ with the same
period and, given some simple assumptions about the factorization, $(ii)$ 
\P\ is irreducible if and only if \Pb\ is irreducible and $(iii)$ 
\P\ is regular if and only if \Pb\ is regular (for details, see
the article by~\cwp{barreto2011computing}).
We will refer to this insight as the
``\emph{stochastic-factorization trick}'':

\begin{center}
\fbox{\begin{minipage}{\columnwidth}{\sl
Given a stochastic factorization of a transition matrix, $\P=\D\K$, 
\emph{swapping} the factors of the factorization yields another transition
matrix $\Pb = \K\D$, potentially much smaller than the original, which retains 
the basic topology and properties of \P.
}\end{minipage}}\vspace{1mm} 
\end{center}

Given the strong connection between $\P \in \R^{n \times n}$ and 
$\Pb \in \R^{m \times m}$, the idea of replacing the former by the
latter comes almost inevitably. The motivation for this would
be, of course, to save computational resources when $m < n$. 
For example, \ct{barreto2011computing} have shown that it is possible to
recover the stationary distribution of \P\ through a linear transformation of
the corresponding distribution of \Pb. In this paper we will use the
stochastic-factorization trick to reduce the computational
cost of KBRL. The strategy will be to summarize the information contained in
KBRL's MDP in a model of fixed size.

\subsection{Reducing a Markov Decision Process}
\label{sec:red_mdp}

The idea of using stochastic factorization to reduce 
dynamic programming's computational requirements is straightforward:
given factorizations of the transition matrices \Pa, we can apply our
trick to obtain a reduced MDP that will be solved 
in place of the original one. In the most general scenario, we would have one 
independent factorization $\Pa = \Da\Ka$ for each action $a \in A$.
However, in the current work we will focus on the particular case 
in which there is a single matrix \D, which will
prove to be convenient both mathematically and computationally.

Obviously, in order to apply the stochastic-factorization trick to an MDP, we
have to first \emph{compute} the matrices involved in the factorization.
Unfortunately, such a procedure can be computationally demanding, exceeding the
number of operations necessary to calculate
\vo~\cp{vavasis2009complexity,barreto2012policy}. 
Thus, in practice we may have to replace the exact 
factorizations $\Pa = \D\Ka$ with
approximations $\Pa \approx \D\Ka$. The following proposition bounds the error
in the value-function approximation resulting from the application of
our trick to approximate stochastic 
factorizations:

\begin{proposition}
\label{teo:bound_sf}
Let $M \equiv(S,A,\Pa,\ra,\gamma)$ be
a finite MDP with $|S| = n$ and 
$0 \le \gamma < 1$. Let $\D \in \R^{n \times m}$ be a stochastic matrix
and, for each $a \in A$, let $\Ka \in  \R^{m \times n}$ be stochastic 
and let $\rba$ be a vector in $\R^{m}$. Define the MDP $\bM
\equiv(\Sb,A,\Pba,\rba,\gamma)$, with 
$|\Sb| = m$ and $\Pba = \Ka \D$. Then,
{\small
 \begin{equation}
\label{eq:bound_sf}
\infnorm{\vo - \Rd\D\Qbo} \le
\boundV\equiv 
\frac{1}{1-\gamma} 
\maxinf{\ra}{\D\rba} +
\frac{\bar{R}_{\dif}}{(1-\gamma)^{2}} 
\left(
\frac{\gamma}{2}
\maxinf{\Pa}{\D\Ka}
+
\levelD   
\right),
\end{equation}
}
where
$\infnorm{\cdot}$ is the maximum norm, 
$\bar{R}_{\dif} = \mmax{a,i}{\bar{r}^{a}_{i}} - \mmin{a,i}{\bar{r}^{a}_{i}}$,
and $\levelD = \mmax{i}{(1 - \mmax{j}{d_{ij})}}$.\footnote{We recall that
$\infnorm{\cdot}$ induces the following norm over the space of matrices:
$\infnorm{\mat{A}} = \mathop{\max}_{i} \sum_{j} |a_{ij}|$.}
\end{proposition}
\begin{proof}
Let $\check{M} \equiv (S, A, \Pxa, \rxa, \gamma)$,
with $\Pxa = \D\Ka$ and $\rxa = \D\rba$. From the triangle inequality, we know
that
\begin{equation}
\label{eq:triangle}
\infnorm{\vo - \Rd\D\Qbo} \le \infnorm{\vo - \vxo} + \infnorm{\vxo -
\Rd\D\Qbo},
\end{equation}
where \vxo\ is the optimal value function of \xM.
Our strategy will be to bound $\infnorm{\vo - \vxo}$ and 
$ \infnorm{\vxo - \Rd\D\Qbo}$.
In order to find an upper bound for \infnorm{\vo - \vxo}, we 
apply~\ctp{whitt78approximations}
Theorem~3.1 and Corollary~(b) of his Theorem~6.1,
with all mappings between $M$ and
$\check{M}$ taken to be identities, to obtain
\begin{equation}
\label{eq:whitt}
\infnorm{\vo - \vxo} \le \dfrac{1}{1-\gamma} 
\left(\maxinf{\ra}{\D\rba} + \dfrac{\gamma \bar{R}_{\dif}}{2(1-\gamma)} 
\maxinf{\Pa}{\D\Ka}
\right),
\end{equation}
where we used the fact that 
$\max_{a,i} \check{r}^{a}_{i} - \min_{a,i} \check{r}^{a}_{i} \le
\bar{R}_{\dif}$.
It remains to bound $\infnorm{\vxo - \Rd\D\Qbo}$.
Since $\rxa = \D\rba$ and $\D\Pba = \D\Ka\D = \Pxa \D$
for all $a \in A$, the stochastic matrix \D\ satisfies
\ctp{sorg2009transfer} definition of a \emph{soft homomorphism} between 
$\check{M}$ and \bM\ (see equations (25)--(28) in their paper).
Applying Theorem~1 by the same authors, we know that
\begin{equation}
\label{eq:sing_tighter}
\infnorm{\Rd(\Qxo - \D\Qbo)} \le 
(1 - \gamma)^{-1} \sup_{i,t} (1 - \max_{j}
d_{ij}) \; \bar{\delta}_{i}^{(t)}, 
\end{equation}
where 
$
\bar{\delta}_{i}^{(t)} = \max_{j: d_{ij} > 0, k}{\bar{q}^{(t)}_{jk}
- \min_{j: d_{ij} > 0, k}{\bar{q}^{(t)}_{jk}}}
$
and $\bar{q}^{(t)}_{jk}$ are elements of $\Qb^{(t)}$, the optimal $t$-step
action-value function of \bM. Since
$\infnorms{\Rd\Qxo - \Rd\D\Qbo} \le \infnorms{\Rd(\Qxo - \D\Qbo)}$
and, for all $t > 0$, $\bar{\delta}_{i}^{(t)} \le (1 - \gamma)^{-1}
(\max_{a,k} \bar{r}^{a}_{k} - \min_{a,k} \bar{r}^{a}_{k} )$,
we can write
\begin{equation}
\label{eq:sing}
\infnorm{\vxo - \Rd\D\Qbo} \le 
\frac{ \bar{R}_{\dif}}{(1-\gamma)^{2}}\mmax{i}{(1- \mmax{j}{d_{ij})}}
= \frac{ \bar{R}_{\dif}}{(1-\gamma)^{2}} \levelD.
\end{equation}
Substituting~(\ref{eq:whitt})
and~(\ref{eq:sing}) back into~(\ref{eq:triangle}), we
obtain~(\ref{eq:bound_sf}).
\end{proof}

We note that our bound can be made tighter if we replace the right-hand side
of~(\ref{eq:sing}) with the right-hand side of~(\ref{eq:sing_tighter}). 
However, such a replacement would result in a less intelligible bound that 
cannot be computed in practice. Needless to say, all 
subsequent developments that depend on Proposition~\ref{teo:bound_sf} 
(and on \boundV\ in particular) are also valid for
the tighter version of the bound. 
In Appendix~\ref{seca:averagers} we derive another bound 
for the distance between $\vo$ and $\Rd\D\Qbo$ which is
valid for any norm.

Our bound depends on two factors: the quality of the MDP's factorization,
given by \maxinf{\Pa}{\D\Ka} and \maxinf{\ra}{\D\rba}, and 
the ``level of stochasticity'' of \D, measured by \levelD.
When the MDP factorization is exact, we recover~(\ref{eq:sing}), which is a
computable version of~\ctp{sorg2009transfer} bound for soft
homomorphisms. On the other hand, when \D\ is deterministic---that is, when all
its nonzero elements are $1$---expression~(\ref{eq:bound_sf}) reduces to
\ctp{whitt78approximations}
classical result regarding state aggregation in dynamic
programming. 
Finally, if we have exact deterministic
factorizations, the right-hand side of~(\ref{eq:bound_sf}) reduces to zero. This
also makes sense, since in this case the stochastic-factorization trick
gives rise to an exact homomorphism~\cp{ravindran2004algebraic}.

Proposition~\ref{teo:bound_sf} elucidates the basic mechanism through which one
can use the stochastic-factorization trick to reduce the number of states
in an MDP (and hence the computational cost of 
finding a policy using dynamic programming). One
possible way to exploit this result is to see the computation of 
\D, \Ka, and \rba\ as an optimization problem in which the objective is to
minimize some function of 
\maxinf{\Pa}{\D\Ka}, \maxinf{\ra}{\D\rba}, and
possibly also \levelD~\cp{barreto2012policy}.
However, in this paper we adopt a different approach: as will be shown, we apply
our trick in the context of reinforcement learning to \emph{avoid} the
construction of \Pa\ and \ra.

\section{Kernel-Based Stochastic Factorization}
\label{sec:kbsf}

In Section~\ref{sec:back} we presented KBRL, an approximation framework for
reinforcement learning whose main drawback is
its high computational complexity. In Section~\ref{sec:sf} we discussed how the
stochastic-factorization trick can in principle be useful to reduce an MDP, as
long as one circumvents the computational burden imposed by the calculation of
the matrices involved in the process.  We now show how to leverage these two
components to produce an algorithm called \emph{kernel-based stochastic
factorization} (KBSF) that overcomes these computational limitations.

KBSF emerges from the application of the
stochastic-factorization trick to KBRL's MDP \cM~\cp{barreto2011reinforcement}.
Similarly to \ct{ormoneit2002kernelbased}, we start by defining a ``mother
kernel'' $\mkb(x) : \R^{+} \mapsto \R^{+}$.
In Appendix~\ref{seca:assumptions} we list our assumptions regarding 
\mkb. Here, it suffices to note that, since our
assumptions and \ctp{ormoneit2002kernelbased} are not mutually exclusive, we can
have $\mk \equiv \mkb$ (by using the Gaussian function in both cases, for
example). Let $\Sb \equiv \{\rs_{1}, \rs_{2},..., \rs_{m}\}$ be a set of 
\emph{representative states}.
Analogously to~(\ref{eq:gaussa})
and~(\ref{eq:kera}), we define the kernel
$
\label{eq:gaussb}
\gaussb(s,\sprime) = 
{\mkb \left({\norm{s - \sprime}}/{ \taub} \right)}
$
and its normalized version 
$
\label{eq:kerb}
\kerb(s, \rs_{i}) = 
{\gaussb(s, \rs_{i})}/{\sum_{j=1}^{m}
\gaussb(s, \rs_{j})} .
$
We will use \kera\ to build matrices \Ka\ and \kerb\ to build matrix \D.

\begin{figure*}
\vspace{5mm}
\centering
\subfloat[KBRL's matrices \label{fig:matrices_kbrl}]{ 
\fbox{
\begin{small}
$
\begin{array}{c}
\Pca = 
\begin{array}{c}
 \hat{s}^{a}_{1} \\
\hat{s}^{a}_{2} \\
\hat{s}^{a}_{3} \\
\hat{s}^{b}_{1} \\
\hat{s}^{b}_{2} \\
\end{array}
\begin{blockarray}{ccccc}
 \hat{s}^{a}_{1} & \hat{s}^{a}_{2} & \hat{s}^{a}_{3} & \hat{s}^{b}_{1} &
\hat{s}^{b}_{2} \\
\begin{block}{[ccccc]}
       \kera(\hat{s}^{a}_{1}, {s}^{a}_{1})  & \kera(\hat{s}^{a}_{1},
{s}^{a}_{2}) & \kera(\hat{s}^{a}_{1}, {s}^{a}_{3}) & 0  & 0\\
       \kera(\hat{s}^{a}_{2}, {s}^{a}_{1})  & \kera(\hat{s}^{a}_{2},
{s}^{a}_{2}) & \kera(\hat{s}^{a}_{2}, {s}^{a}_{3}) & 0  & 0\\
       \kera(\hat{s}^{a}_{3}, {s}^{a}_{1})  & \kera(\hat{s}^{a}_{3},
{s}^{a}_{2}) & \kera(\hat{s}^{a}_{3}, {s}^{a}_{3}) & 0  & 0\\
       \kera(\hat{s}^{b}_{1}, {s}^{a}_{1})  & \kera(\hat{s}^{b}_{1},
{s}^{a}_{2}) & \kera(\hat{s}^{b}_{1}, {s}^{a}_{3}) & 0  & 0\\
       \kera(\hat{s}^{b}_{2}, {s}^{a}_{1})  & \kera(\hat{s}^{b}_{2},
{s}^{a}_{2}) & \kera(\hat{s}^{b}_{2}, {s}^{a}_{3}) & 0  & 0\\
\end{block}
\end{blockarray},
\\ 
\Pcb = 
\begin{array}{c}
 \hat{s}^{a}_{1} \\
\hat{s}^{a}_{2} \\
\hat{s}^{a}_{3} \\
\hat{s}^{b}_{1} \\
\hat{s}^{b}_{2} \\
\end{array}
\begin{blockarray}{ccccc}
 \hat{s}^{a}_{1} & \hat{s}^{a}_{2} & \hat{s}^{a}_{3} & \hat{s}^{b}_{1} &
\hat{s}^{b}_{2} \\
\begin{block}{[ccccc]}
       0  & 0 & 0 & \kera(\hat{s}^{a}_{1}, {s}^{b}_{1})  &
\kera(\hat{s}^{a}_{1}, {s}^{b}_{2})\\
       0  & 0 & 0 & \kera(\hat{s}^{a}_{2}, {s}^{b}_{1})  &
\kera(\hat{s}^{a}_{2}, {s}^{b}_{2})\\
       0  & 0 & 0 & \kera(\hat{s}^{a}_{3}, {s}^{b}_{1})  &
\kera(\hat{s}^{a}_{3}, {s}^{b}_{2})\\
       0  & 0 & 0 & \kera(\hat{s}^{b}_{1}, {s}^{b}_{1})  &
\kera(\hat{s}^{b}_{1}, {s}^{b}_{2})\\
       0  & 0 & 0 & \kera(\hat{s}^{b}_{2}, {s}^{b}_{1})  &
\kera(\hat{s}^{b}_{2}, {s}^{b}_{2})\\
\end{block}
\end{blockarray}
\end{array}
$
\end{small}
}
}

\subfloat[KBSF's sparse matrices \label{fig:matrices_kbsf_sparse}]{ 
\framebox[\textwidth]{
\begin{small}
$
\begin{array}{cl}
\D = 
\begin{array}{c}
 \hat{s}^{a}_{1} \\
\hat{s}^{a}_{2} \\
\hat{s}^{a}_{3} \\
\hat{s}^{b}_{1} \\
\hat{s}^{b}_{2} \\
\end{array}
\begin{blockarray}{cc}
 \rs_{1} & \rs_{2} \\
\begin{block}{[cc]}
\kerb(\hat{s}^{a}_{1}, \rs_{1})  & \kerb(\hat{s}^{a}_{1}, \rs_{2}) \\
\kerb(\hat{s}^{a}_{2}, \rs_{1})  & \kerb(\hat{s}^{a}_{2}, \rs_{2}) \\
\kerb(\hat{s}^{a}_{3}, \rs_{1})  & \kerb(\hat{s}^{a}_{3}, \rs_{2}) \\
\kerb(\hat{s}^{b}_{1}, \rs_{1})  & \kerb(\hat{s}^{b}_{1}, \rs_{2}) \\
\kerb(\hat{s}^{b}_{2}, \rs_{1})  & \kerb(\hat{s}^{b}_{2}, \rs_{2}) \\
\end{block}
\end{blockarray},
&
\begin{array}{l}
\Ka = 
\begin{array}{c}
 \rs_{1} \\
\rs_{2} \\
\end{array}
\begin{blockarray}{ccccc}
 \hat{s}^{a}_{1} & \hat{s}^{a}_{2} & \hat{s}^{a}_{3} & \hat{s}^{b}_{1} &
\hat{s}^{b}_{2} \\
\begin{block}{[ccccc]}
       \kera(\rs_{1}, {s}^{a}_{1})  & \kera(\rs_{1}, {s}^{a}_{2}) &
\kera(\rs_{1}, {s}^{a}_{3}) & 0  & 0\\
       \kera(\rs_{2}, {s}^{a}_{1})  & \kera(\rs_{2}, {s}^{a}_{2}) &
\kera(\rs_{2}, {s}^{a}_{3}) & 0  & 0\\
\end{block}
\end{blockarray},  \\
\matii{K}{b}{} = 
\begin{array}{c}
 \rs_{1} \\
\rs_{2} \\
\end{array}
\begin{blockarray}{ccccc}
 \hat{s}^{a}_{1} & \hat{s}^{a}_{2} & \hat{s}^{a}_{3} & \hat{s}^{b}_{1} &
\hat{s}^{b}_{2} \\
\begin{block}{[ccccc]}
       0 & 0 & 0 & \kera(\rs_{1}, {s}^{b}_{1}) &
\kera(\rs_{1}, {s}^{b}_{2}) \\
       0 & 0 & 0 & \kera(\rs_{2}, {s}^{b}_{1}) &
\kera(\rs_{2}, {s}^{b}_{2}) \\
\end{block}
\end{blockarray}. \\
\end{array}
\end{array}
$
\end{small}
}
}

\subfloat[KBSF's dense matrices \label{fig:matrices_kbsf_dense}]{ 
\framebox[\columnwidth]{
\begin{small}
$
\begin{array}{cc}
\begin{array}{l}
\matii{\dot{D}}{a}{} = 
\begin{array}{c}
 \hat{s}^{a}_{1} \\
\hat{s}^{a}_{2} \\
\hat{s}^{a}_{3} \\
\end{array}
\begin{blockarray}{cc}
 \rs_{1} & \rs_{2} \\
\begin{block}{[cc]}
\kerb(\hat{s}^{a}_{1}, \rs_{1})  & \kerb(\hat{s}^{a}_{1}, \rs_{2}) \\
\kerb(\hat{s}^{a}_{2}, \rs_{1})  & \kerb(\hat{s}^{a}_{2}, \rs_{2}) \\
\kerb(\hat{s}^{a}_{3}, \rs_{1})  & \kerb(\hat{s}^{a}_{3}, \rs_{2}) \\
\end{block}
\end{blockarray}, \\
\matii{\dot{D}}{b}{} = 
\begin{array}{c}
\hat{s}^{b}_{1} \\
\hat{s}^{b}_{2} \\
\end{array}
\begin{blockarray}{cc}
 \rs_{1} & \rs_{2} \\
\begin{block}{[cc]}
\kerb(\hat{s}^{b}_{1}, \rs_{1})  & \kerb(\hat{s}^{b}_{1}, \rs_{2}) \\
\kerb(\hat{s}^{b}_{2}, \rs_{1})  & \kerb(\hat{s}^{b}_{2}, \rs_{2}) \\
\end{block}
\end{blockarray}, \\
\end{array}
&
\begin{array}{l}
\Kda = 
\begin{array}{c}
 \rs_{1} \\
\rs_{2} \\
\end{array}
\begin{blockarray}{ccc}
 \hat{s}^{a}_{1} & \hat{s}^{a}_{2} & \hat{s}^{a}_{3} \\
\begin{block}{[ccc]}
       \kera(\rs_{1}, {s}^{a}_{1})  & \kera(\rs_{1}, {s}^{a}_{2}) &
\kera(\rs_{1}, {s}^{a}_{3})\\
       \kera(\rs_{2}, {s}^{a}_{1})  & \kera(\rs_{2}, {s}^{a}_{2}) &
\kera(\rs_{2}, {s}^{a}_{3}) \\
\end{block}
\end{blockarray}, \\ \\
\matii{\dot{K}}{b}{} = 
\begin{array}{c}
 \rs_{1} \\
\rs_{2} \\
\end{array}
\begin{blockarray}{cc}
 \hat{s}^{b}_{1} & \hat{s}^{b}_{2} \\
\begin{block}{[cc]}
       \kera(\rs_{1}, {s}^{b}_{1}) & \kera(\rs_{1}, {s}^{b}_{2}) \\
       \kera(\rs_{2}, {s}^{b}_{1}) & \kera(\rs_{2}, {s}^{b}_{2}) \\
\end{block}
\end{blockarray}. \\
\end{array}\\
\end{array}
$
\end{small}
}
}
\caption{Matrices built by KBRL and KBSF for the case in which the original MDP
has two actions, $a$ and $b$, and $n_{a} = 3$, $n_{b} = 2$, and $m = 2$.
\label{fig:matrices}}
\end{figure*}

As shown in Figure~\ref{fig:matrices_kbrl}, KBRL's matrices \Pca\ have a very
specific structure, since only
transitions ending in states $\yia \in \Sca$ have 
a nonzero probability of occurrence. 
Suppose that we want to apply the stochastic-factorization trick to KBRL's
MDP. Assuming that the matrices \Ka\ have the same structure
as \Pca, when computing $\Pba = \Ka \D$ we only have to look at the
sub-matrices of \Ka\ and \D\ corresponding to the $n_{a}$ nonzero columns of
\Ka. We call these matrices $\Kda \in \R^{m \times n_{a}}$ and 
$\Dda \in \R^{n_{a} \times m}$. 
The strategy of KBSF is to fill out matrices $\Kda$ and $\Dda$ with elements
\begin{equation}
\label{eq:mat_kbsf}
\begin{array}{ccc}
\dot{k}^{a}_{ij} = \kera(\rs_{i}, \xja) 
& \text{ and } &
\dot{d}^{a}_{ij} =\kerb(\yia,\rs_{j}). \\
\end{array}
\end{equation}
Note that, based on \Dda, one can easily recover \D\ as
$
\begin{array}{c}
\D^{\t} \equiv [
   (\matii{\dot{D}}{1}{})^{\t} 
   (\matii{\dot{D}}{2}{})^{\t}  ...
   (\matii{\dot{D}}{|A|}{})^{\t} ] \in \R^{n \times m}.
\end{array}
$
Similarly, if we let 
$\K \equiv [\matii{\dot{K}}{1}{} 
\matii{\dot{K}}{2}{} ...
\matii{\dot{K}}{{|A|}}{}] \in \R^{m \times n}$,
then $\Ka \in \R^{m \times n}$ is matrix \K\ with all elements replaced by
zeros except for those corresponding to matrix \Kda
(see Figures~\ref{fig:matrices_kbsf_sparse} 
and~\ref{fig:matrices_kbsf_dense} for an illustration). 
It should be thus obvious that $\Pb = \Ka\D = \Kda\Dda$.

In order to conclude the construction of KBSF's MDP, we have to define the
vectors of expected rewards \rba.
As shown in expression~(\ref{eq:mdp_kbrl_R}), the reward functions of KBRL's
MDP, $\hat{R}^{a}(s,\sprime)$, only depend on the ending state $\sprime$.
Recalling
the
interpretation of the rows of \Ka\ as transition probabilities from the
representative states to the original ones, illustrated in
Figure~\ref{fig:mat_trans}, 
it is clear that 
\begin{equation}
\label{eq:rba}
\rba = \Kda \ra = \Ka \rr. 
\end{equation}
Therefore, the formal specification of KBSF's MDP is given by
$\bM \equiv (\bar{S},A,\Kda\Dda,\Kda\ra,\gamma) =
(\bar{S},A,\Ka\Da,\Ka\rr,\gamma) = (\bar{S},A,\Pba,\rba,\gamma)$.

As discussed in Section~\ref{sec:kbrl}, KBRL's approximation scheme can 
be interpreted as the derivation of a finite MDP. In this case, the sample
transitions define both the finite state space \cS\ and the model's 
transition and reward functions.
This means that the state space and dynamics of KBRL's model are inexorably
linked: except maybe for degenerate cases, changing one also changes the other.
By defining a set of representative states, KBSF decouples the MDP's structure
from its particular instantiation. To see why this is so, note that, if we fix
the representative states, different sets of sample transitions will
give rise to different models. Conversely, the same set of transitions
can generate different MDPs, depending on how the representative states are
defined.

A step by step description of KBSF is given in Algorithm~\ref{alg:kbsf}.
As one can see, KBSF is very simple to understand and to implement.
It works as follows: first, the MDP \bM\ is built as described
above. Then, its action-value function \Qbo\ is determined through any 
dynamic programming algorithm. Finally, KBSF returns an approximation
of \vco---the optimal value function of KBRL's MDP---computed as $\vt = \Rd \D
\Qbo$.  Based on \vt, one can compute an approximation of KBRL's action-value
function
$\cQ(s,a)$ by simply replacing $\tV$ for $\cV^{*}$  in~(\ref{eq:kbrl_q}),
that is, 
\begin{equation}
\label{eq:kbsf_q}
\tQ(s,a) = \sum_{i=1}^{n_{a}} \kera(s, \xia) \left[\ria +
\gamma \tV(\yia) \right],
\end{equation}
where $s \in \Sc$ and $a \in A$. Note that $\tV(\yia)$ corresponds to 
one specific entry of vector \vt, whose index is given by 
$\sum_{b=0}^{a-1} n_b + i$, where we assume that $n_0 = 0$.

\begin{algorithm}
   \caption{Batch KBSF}
   \label{alg:kbsf}
\begin{algorithmic}
   \State {\bfseries Input:} 
\begin{tabular}{lr}
$\Sca = \{(\xka,\rka,\yka)| k = 1, 2, ..., n_{a}\}$ for all $a \in A$ 
& \Comment{Sample transitions} \\
$\Sb = \{\rs_{1}, \rs_{2}, ..., \rs_{m}\}$ 
& \Comment{Set of representative states}\\
\end{tabular}
   \State {\bfseries Output:} $\vt \approx \vco$
   \For{each $a \in A$}
   \State Compute matrix \Dda: $\dot{d}^{a}_{ij} = \kerb(\yia, \rs_{j})$
   \State Compute matrix \Kda: $\dot{k}^{a}_{ij} = \kera(\rs_{i}, \xja)$
   \State Compute vector \matii{\bar{r}}{a}{}: $\bar{r}_{i}^{a} = \sum_{j}
\dot{k}^{a}_{ij} \rja$
   \State Compute matrix $\Pba = \Kda \Dda$
   \EndFor
   \State Solve $\bM \equiv(\Sb,A,\Pba,\rba,\gamma)$
\Comment{{\sl i.e.}, compute \Qbo}
   \State Return $\vt = \Rd \D \Qbo$, where {$\D^{\t} = \left[
   (\matii{\dot{D}}{1}{})^{\t} 
   (\matii{\dot{D}}{2}{})^{\t}  ...
   (\matii{\dot{D}}{|A|}{})^{\t} 
   \right]$ }
\end{algorithmic}
\end{algorithm}

As shown in Algorithm~\ref{alg:kbsf}, the key point of KBSF's mechanics is 
the fact that
the matrices $\Pxa = \D\Ka$ are never actually computed, but instead we 
directly solve the MDP \bM\ containing $m$ states only. 
This results in an efficient algorithm that requires only 
$O(n m |A| \dims + \hat{n}m^{2}|A|)$
operations and $O(\hat{n}m)$ bits to build a reduced
version of KBRL's MDP, where $\hat{n} = \max_{a} n_{a}$. After the reduced
model \bM\ has been
constructed,  KBSF's computational cost becomes a function of $m$ only.
In particular, the cost of solving \bM\ through dynamic programming
becomes polynomial in $m$ instead of $n$: while one application of 
$\cT$, the Bellman operator of \cM, is
$O(n\hat{n}|A|)$, the computation of $\bT$ is 
$O(m^{2}|A|)$.
Therefore, KBSF's time
and memory complexities are only linear in $n$.

We note that, in practice, KBSF's computational requirements can be reduced even
further if one enforces the kernels \kera\ and \kerb\ to be sparse. In
particular, given a fixed $\rs_{i}$, instead of computing $\gaussb(\rs_{i},
\xja)$ for $j = 1, 2, ..., n_a$, one can evaluate the kernel on a pre-specified
neighborhood of $\rs_{i}$ only. Assuming that $\gaussb(\rs_{i}, \xja)$ is zero
for all \xja\ outside this region,
one avoids not only computing the kernel but also storing the
resulting values (the same reasoning applies to the computation of
$\gaussa(\yia, \rs_{j})$ for a fixed~\yia).

\subsection{A closer look at KBSF's approximation}
\label{sec:kbsf_app}

As outlined in Section~\ref{sec:back}, KBRL defines the probability of a
transition from state \yib\ to state \yka\ as being $\kera(\yib, \xka)$, where
$a, b \in A$ (see Figure~\ref{fig:matrices_kbrl}). 
Note that the kernel \kera\ is computed with the initial state \xka, and not
\yka\ itself. The intuition behind this is simple: since we know the transition
$\xka \xrightarrow{a} \yka$ has occurred before, the more ``similar'' \yib\ is
to \xka, the more likely the transition $\yib \xrightarrow{a} \yka$
becomes~\cp{ormoneit2002kernelbased}.

From~(\ref{eq:mat_kbsf}), it is clear that 
the computation of matrices \Ka\ performed by KBSF follows the same reasoning
underlying the computation of KBRL's matrices \Pca;
in particular, $\kera(\rs_{j}, \xka)$  gives the probability of 
a transition from $\rs_{j}$ to $\yka$.
However, when we look at matrix \D\ things are slightly different: 
here, the probability of a ``transition'' 
from $\yib$ to representative state $\rs_{j}$ is
given by $\kerb(\yib, \rs_{j})$---a computation that involves $\rs_{j}$ itself.
If we were to strictly adhere to KBRL's logic when computing the transition
probabilities to the representative states $\rs_{j}$, 
the probability of transitioning from \yib\ to $\rs_{j}$ 
upon executing action $a$ should be a function of 
\yib\ and a state $\sprime$ from which we knew a transition 
$\sprime \xrightarrow{a} \rs_{j}$ had occurred. 
In this case we would end up with one matrix
\Da\ for each action $a \in A$. 
Note though that this formulation of the method is not practical, because
the computation of the matrices \Da\ would require 
a transition $(\cdot) \xrightarrow{a} \rs_{j}$ for each $a \in A$ and each
$\rs_{j} \in \Sb$. Clearly, such a requirement is
hard to fulfill even if we have a generative model available to generate sample
transitions.

In this section we provide an interpretation of the approximation computed by
KBSF that supports our definition of matrix \D.
We start by looking at how KBRL constructs the matrices \Pca.
As shown in Figure~\ref{fig:matrices_kbrl}, for each action $a \in A$
the state \yib\ has an associated stochastic vector $\matii{\hat{p}}{a}{j} \in
\R^{1 \times n}$ whose nonzero entries correspond to the kernel $\kera(\yib,
\cdot)$ evaluated at $\xka, k = 1, 2, \ldots, n_{a}$.
Since we are dealing with a
continuous state space, it is possible to compute an analogous vector
for any $s \in \Sc$ and any $a \in A$. Focusing on the nonzero entries
of $\matii{\hat{p}}{a}{j}$, we define the function
\begin{equation}
\label{eq:kfunca}
\begin{array}{cl}
\kfunca: & \Sc \mapsto \R^{1 \times n_a} \\
& \kfunca(s) = \matii{\hat{p}}{a}{} \iff \hat{p}^{a}_{i} =
\kera(s,\xia) \text{ for } i = 1, 2, ..., n_a.
\end{array}
\end{equation}

Clearly, full knowledge of the function \kfunca\ 
allows for an exact computation of KBRL's transition matrix \Pca.
Now suppose we do not know \kfunca\ and we want to compute an approximation 
of this function in the points $\yia \in \Sca$, for all $a \in A$. 
Suppose further that we are only given a ``training set'' composed of $m$ pairs
$(\rs_{j}, \kfunca(\rs_{j}))$. One possible way of 
approaching this problem is to resort to kernel smoothing 
techniques. In this case,
a particularly common choice is the so-called Nadaraya-Watson kernel-weighted
estimator~\cp[Chapter~6]{hastie2002elements}:
\begin{equation}
\label{eq:kfuncd}
 \kfuncb(s) = \frac{\sum_{j=1}^{m} \gaussb(s, \rs_{j}) \kfunca(\rs_{j})}
{\sum_{j=1}^{m} \gaussb(s, \rs_{j})} = 
\sum_{j=1}^{m} \kerb(s, \rs_{j})\kfunca(\rs_{j}).
\end{equation}
Contrasting the expression above with~(\ref{eq:mat_kbsf}), we see that this is
exactly how KBSF computes its approximation $\D\Ka \approx \Pca$,
with \kfuncb\ evaluated at the points $\yib \in \Scb$, $b = 1, 2, ...,
|A|$. In this case, $\kerb(\yib, \rs_{j})$ are the elements of matrix \D, 
and $\kfunca(\rs_{j})$ is the \jth\ row of matrix \Kda.
Thus, in some sense, KBSF uses KBRL's own kernel approximation principle to
compute a stochastic factorization of \cM.

\subsection{Theoretical results}
\label{sec:theory_batch}

Since KBSF comes down to the solution of a finite MDP, 
it always converges to 
the same approximation \vt, whose distance to KBRL's optimal
value function \vco\ is bounded by Proposition~\ref{teo:bound_sf}. 
Once \vt\ is available, the value of any state-action pair can be
determined through~(\ref{eq:kbsf_q}). The following result generalizes
Proposition~\ref{teo:bound_sf} to the entire continuous state space \Sc:

\begin{proposition}
\label{teo:batch_kbsf}
Let $\cQ$ be the value function computed by KBRL through~(\ref{eq:kbrl_q})
and let $\tilde{Q}$ be the value function computed by KBSF
through~(\ref{eq:kbsf_q}).
Then, for any $s \in \Sc$ and any $a \in A$,
$|\cQ(s,a) - \tilde{Q}(s,a)| \le \gamma \boundV$, with \boundV\
defined in~(\ref{eq:bound_sf}).
\end{proposition}
\begin{proof}
\begin{align*}
|\cQ(s,a) - \tilde{Q}(s,a)|
& = 
\left|
\sum_{i=1}^{n_{a}} \kera(s, \xia) \left[\ria +
\gamma \hat{V}^{*}(\yia) \right] -
\sum_{i=1}^{n_{a}} \kera(s, \xia) \left[\ria +
\gamma \tV(\yia) \right] 
\right|
\\
& \le
\gamma \sum_{i=1}^{n_{a}} \kera(s, \xia) \left|\cV^{*}(\yia) -\tV(\yia)
\right| 
 \le
\gamma \sum_{i=1}^{n_{a}} \kera(s, \xia) \boundV
\le \gamma \boundV, \\
\end{align*}
where the second inequality results from the application of
Proposition~\ref{teo:bound_sf} and the third inequality is a consequence of
the fact that $\sum_{i=1}^{n_{a}} \kera(s, \xia)$ defines a convex combination.
\end{proof}

Proposition~\ref{teo:batch_kbsf} makes it clear that the quality of the
approximation computed by KBSF depends crucially on \boundV.
In the remainder of this section we will show that, if the distances between
sampled states and the respective nearest representative states are small
enough, then we can make \boundV\ as small as desired by setting $\bar{\tau}$ to
a sufficiently small value.
To be more precise, let $\neib: \Sc \times \{1,2,...,m\} \mapsto \Sb$
be a function that orders the representative states according to their distance
to a given state $s$, that is, if $\neib(s, i) = \rs_{k}$, then
$\rs_{k}$ is the \ith\ nearest representative state to $s$. 
Define $\ddb: \Sc \times
\{1,2,...,m\} \mapsto \R$ as $\ddb(s,i) = \norm{s - \neib(s, i)}$. 
Assuming that we have $|A|$ fixed sets of sample transitions \Sca, we will show
that, for any $\epsilon >  0$,  
there is a $\delta > 0$ such that, if $\max_{a,i}\ddb(\yia, 1) < \delta$,
then we can set $\taub$ in order to guarantee that $\boundV< \epsilon$.
To show that, we will need the following two lemmas, proved in
Appendix~\ref{seca:theory}:

\begin{lemma}
\label{teo:dist_cont}
For any $\xia \in \Sca$ and any $\epsilon > 0$, there is a
$\delta > 0$ such
that $|\kera(s,\xia) - \kera(\sprime,\xia)| < \epsilon$ if 
$\norm{s - \sprime} <~\delta$.
\end{lemma}

\begin{lemma}
\label{teo:split}
Let $s \in \Sc$, let $m > 1$, and assume there is a $w \in \{1, 2, ..., m-1\}$
such that $\ddb(s, w) < \ddb(s,w+1)$. Define 
\begin{equation*}
W \equiv \{k \;|\; \norm{s\ - \rs_{k}} \le \ddb(s,w) \}
\text{ and }
\bar{W} \equiv \{1, 2, ..., m\} - W.
\end{equation*}
Then, for any
$\alpha > 0$, 
$\label{eq:res_split2}
\sum_{k \in W} \kerb(s, \rs_{k}) < \alpha  \sum_{k \in \bar{W}} 
\kerb(s, \rs_{k})
$
for \taub\ sufficiently small.
\end{lemma}

Lemma~\ref{teo:dist_cont} is basically a
continuity argument: it shows that, for any fixed \xia, $|\kera(s,\xia) -
\kera(\sprime,\xia)| \rightarrow 0$ as $\norm{s - \sprime} \rightarrow 0$. 
Lemma~\ref{teo:split} states that, 
if we order the representative states according to their distance to a fixed 
state $s$, and then partition them in two subsets, we can control the
relative magnitude of the corresponding kernels's sums by adjusting the
parameter \taub\ (we redirect the reader to Appendix~\ref{seca:theory} for
details on how to set \taub). Based on these two lemmas, we present the main
result of this section, also proved in Appendix~\ref{seca:theory}:

\begin{proposition}
\label{teo:bound_rep_states}
For any $\epsilon > 0$, there is a $\delta > 0$ such that, if
$\max_{a,i}\ddb(\yia, 1) < \delta$, then we can 
guarantee that $\boundV<~\epsilon$
by making $\taub$ sufficiently small.
\end{proposition}

Proposition~\ref{teo:bound_rep_states} tells us that, regardless
of the specific reinforcement learning problem at hand, if the distances between
sampled states \yia\ and the respective nearest representative states are small
enough, then we can make KBSF's approximation of
KBRL's value function as accurate as desired by setting \taub\ to a sufficiently
small value (one can see how exactly to set \taub\ in the proof of the
proposition). How small the maximum distance $\max_{a,i}\ddb(\yia, 1)$ should be
depends on the 
particular choice of kernel \gaussa\ and on the sets of
sample transitions \Sca.
Here, we deliberately refrained from making assumptions
on \gaussa\ and \Sca\ in order to present the proposition in its most general
form. 

Note that a fixed number of representative states $m$ 
imposes a minimum possible value for $\max_{a,i}\ddb(\yia, 1)$, and if 
this value is not small enough decreasing \taub\ may actually hurt the
approximation. The optimal value for \taub\ in this case is again
context-dependent. As a positive flip side of this statement, we note that, even
if $\max_{a,i}\ddb(\yia, 1) > \delta$, it might be possible to make
$\boundV<~\epsilon$ by setting \taub\ appropriately.
Therefore, rather than as a practical guide on how to configure KBSF,
Proposition~\ref{teo:bound_rep_states} should be seen as a 
theoretical argument showing that KBSF is a sound algorithm, in the sense that
in the limit it recovers KBRL's solution.

\subsection{Empirical results}
\label{sec:empirical_batch}

We now present a series of computational experiments designed to illustrate
the behavior of KBSF in a variety of challenging domains. We start
with a simple problem, the ``puddle world'', to show that KBSF is
indeed capable of compressing the information contained in KBRL's model. We
then move to more difficult tasks, and compare KBSF with other
state-of-the-art reinforcement-learning algorithms.
We start with two classical control tasks, single and double pole-balancing.
Next we study two medically-related problems based on real data: HIV drug
schedule and epilepsy-suppression domains.

All problems considered in this paper have a continuous state space and a
finite number of actions, and were modeled as discounted tasks. The algorithms's
results correspond to the performance of the greedy decision policy derived from
the final value function computed. In all cases, the decision policies were
evaluated on challenging test states from which the tasks cannot be easily
solved. The details of the experiments are given in
Appendix~\ref{seca:exp_details}.

\subsubsection{Puddle world (proof of concept)}
\label{sec:puddle}

In order to show that KBSF is indeed capable of summarizing the information
contained in KBRL's model, we use the puddle world
task~\cp{sutton96generalization}. The puddle world is a simple 
two-dimensional problem in which the objective is to reach a goal region
avoiding two ``puddles'' along the way. We implemented the task 
exactly as described by~\ct{sutton96generalization}, except that we used a
discount factor of $\gamma=0.99$ and evaluated the decision policies on a set of
pre-defined test states surrounding the 
puddles (see Appendix~\ref{seca:exp_details}).

The experiment was carried out as follows: first,
we collected a set of $n$ sample transitions  $(\xka,\rka,\yka)$ using a
random exploration policy (that is, a policy that selects actions
uniformly at random). 
In the case of KBRL, this set of sample transitions defined the 
model used to approximate the value function.
In order to define KBSF's model, the states $\yka$ were grouped by the
$k$-means algorithm into $m$ clusters and a representative state 
$\rs_{j}$ was placed at the center of each resulting
cluster~\cp{kaufman90finding}.
As for the kernels's widths, we varied both $\tau$ and $\taub$ in the set 
$\{0.01, 0.1, 1\}$ (see Table~\ref{tab:kbsf_params} on
page~\pageref{tab:kbsf_params}). The
results reported represent the
best performance of the algorithms over $50$ runs; that is, for each $n$
and each $m$ we picked the combination of parameters that generated the maximum
average return.
We use the following convention to refer to specific instances of each method:
the first number enclosed in parentheses after an algorithm's name is $n$, the
number of sample transitions used in the
approximation, and the second one is $m$, the size of the model used to
approximate the value function. Note that for KBRL $n$ and $m$ coincide.

In Figure~\ref{fig:puddle_m}
and~\ref{fig:puddle_m_time} we observe the effect of fixing the number of 
transitions $n$ and varying the number of representative states $m$. As
expected, KBSF's results improve as $m \rightarrow n$. More
surprising is the fact that KBSF has essentially the same performance as
KBRL using models one order of magnitude smaller. This
indicates that KBSF is summarizing well the information contained in 
the data. Depending on the values of $n$ and $m$, such a 
compression may represent a significant reduction 
on the consumption of computational resources.
For example, by replacing KBRL($8\ts000$) with KBSF($8\ts000$,
$100$), we obtain a decrease of approximately $99.58\%$ on the number of
operations performed to find a policy, as shown in
Figure~\ref{fig:puddle_m_time} (the cost of constructing KBSF's MDP is
included in all reported run times).

\begin{figure*}
\centering
   \subfloat[Performance as a function of $m$]{ 
   \label{fig:puddle_m}
   \includegraphics[scale=\scll]{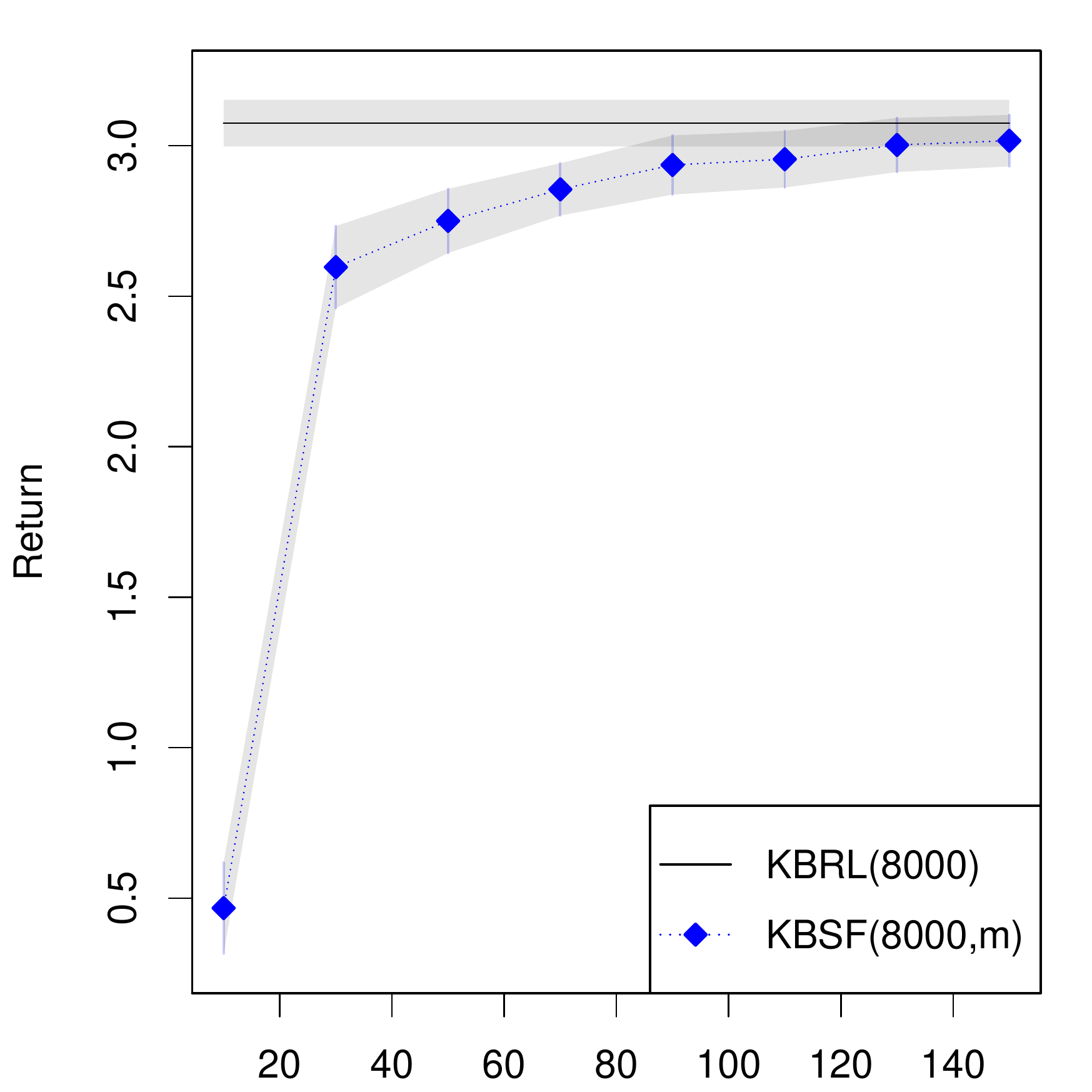} 
   }
    \subfloat[Run time as a function of $m$] { 
   \label{fig:puddle_m_time}
   \includegraphics[scale=\scll]{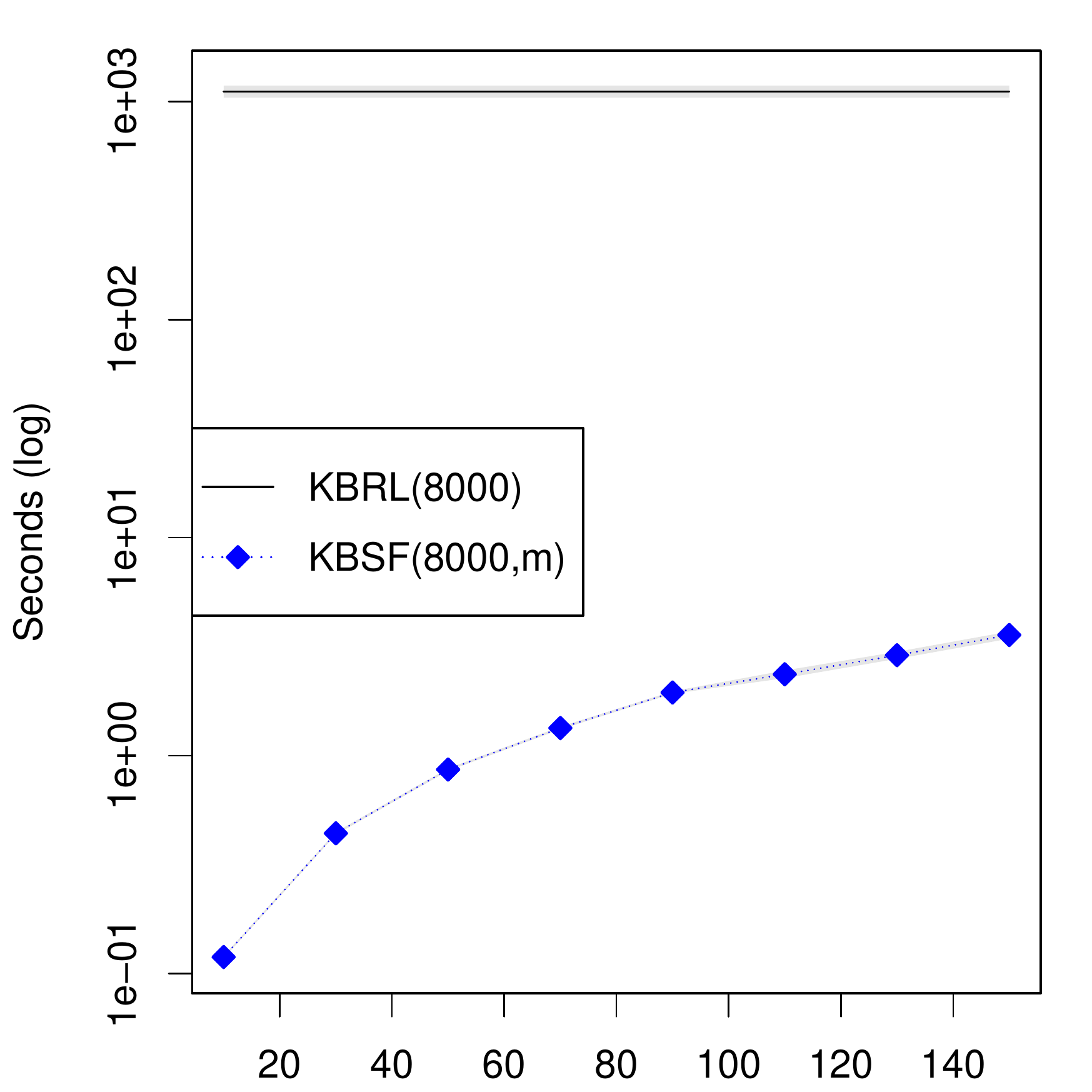}
    }
    
    \subfloat[Performance as a function of $n$]{ 
   \label{fig:puddle_n}
   \includegraphics[scale=\scll]{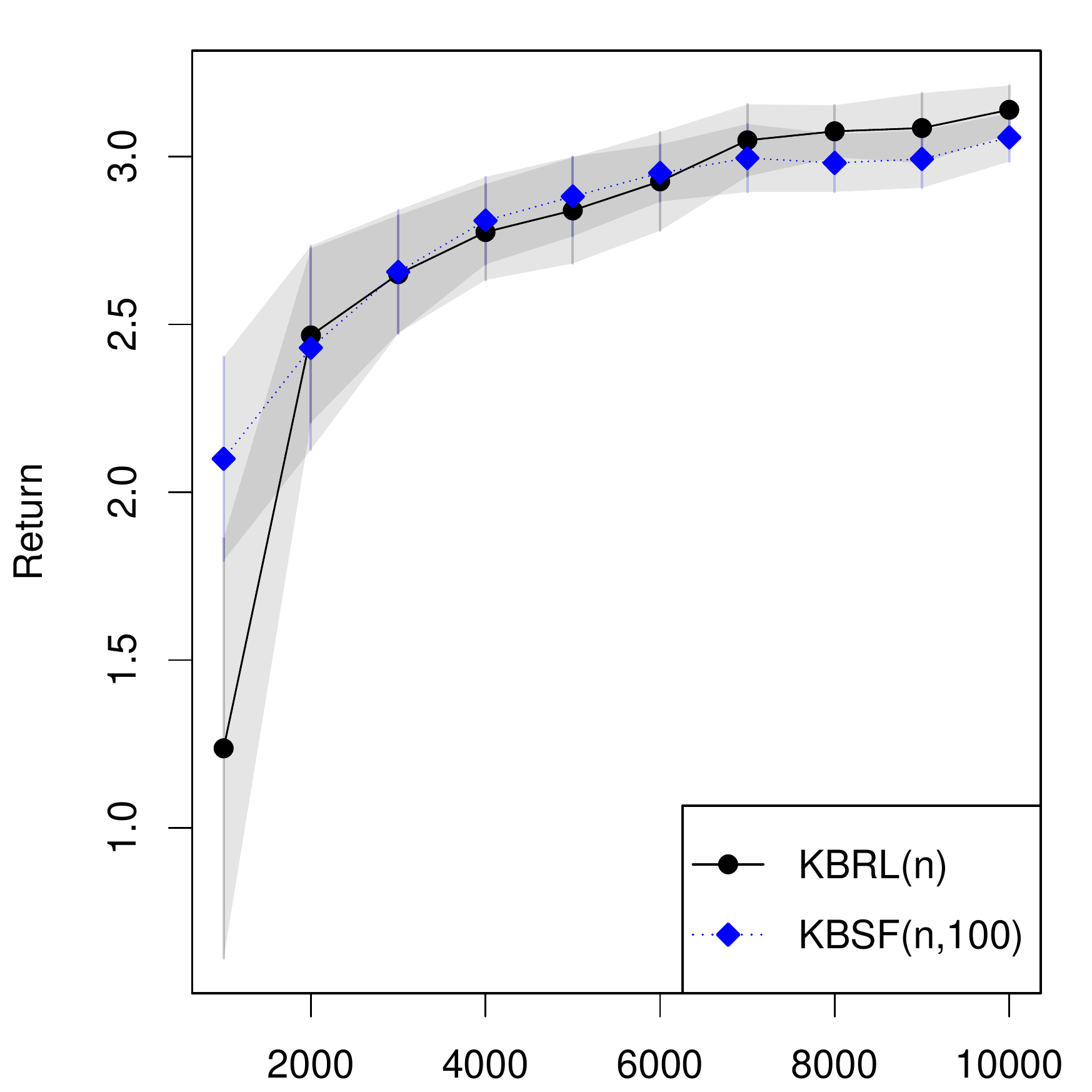} 
   }
    \subfloat[Run time as a function of $n$] { 
   \label{fig:puddle_n_time}
   \includegraphics[scale=\scll]{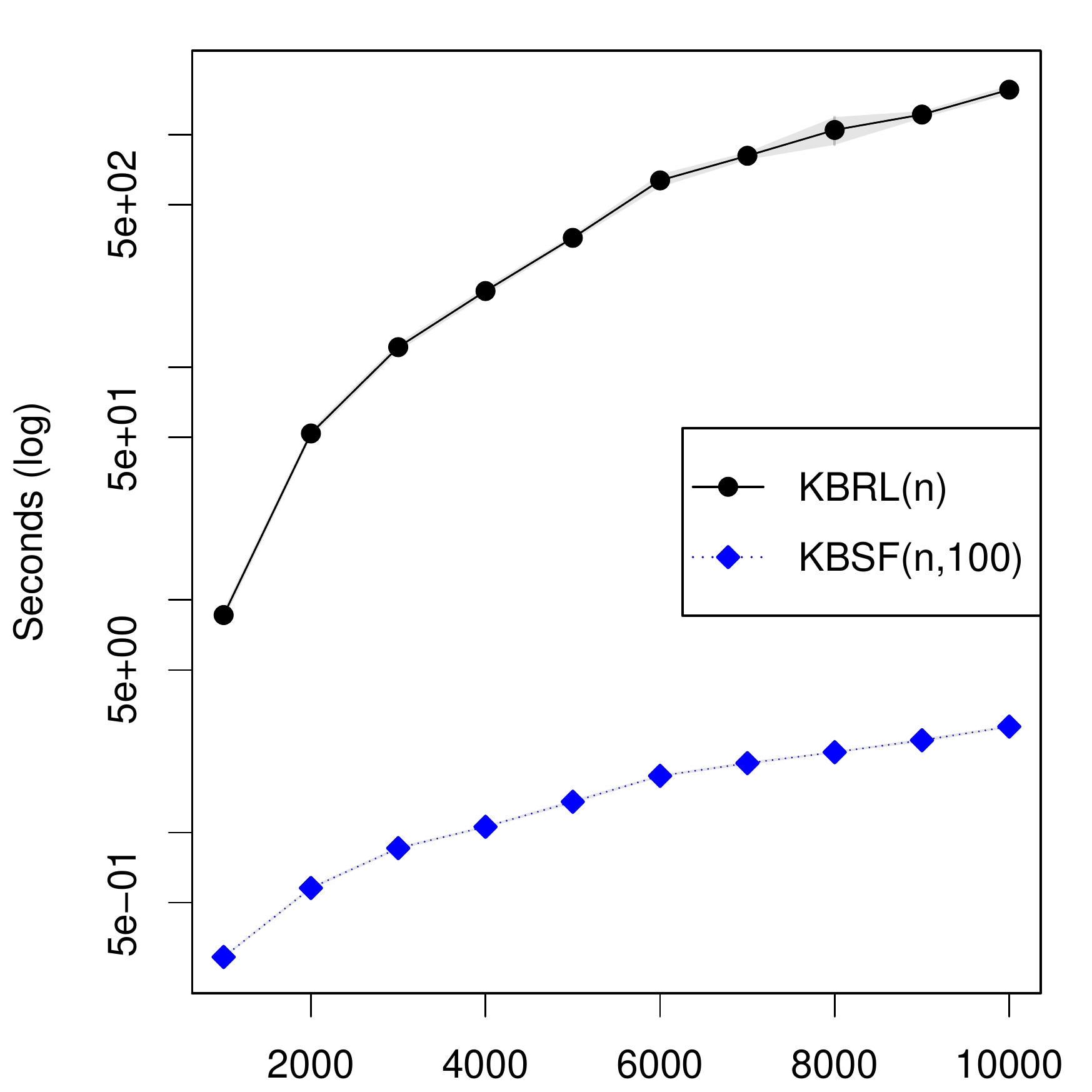}
    }
\caption{Results on the puddle-world task averaged over $50$ runs. The
algorithms were evaluated on a set of test states distributed over a region of
the state space surrounding the ``puddles'' (details in 
Appendix~\ref{seca:exp_details}). The shadowed regions
represent $99\%$ confidence intervals. \label{fig:puddle}}
\end{figure*} 

In Figures~\ref{fig:puddle_n} and~\ref{fig:puddle_n_time} we fix $m$ and vary
$n$. Observe in Figure~\ref{fig:puddle_n} how KBRL and KBSF have similar
performances, and both improve as $n$ increases. However, since
KBSF is using a model of fixed size, its computational cost depends only
linearly on $n$, whereas KBRL's cost grows 
with $n^2\hat{n}$, roughly. This explains the huge
difference in the algorithms's run times shown in
Figure~\ref{fig:puddle_n_time}.

\subsubsection{Single and double pole-balancing (comparison with LSPI) }
\label{sec:pole}

We now evaluate how KBSF compares to other modern reinforcement
learning algorithms on more difficult tasks.
We first contrast our method with \ctp{lagoudakis2003least} least-squares
policy iteration algorithm (LSPI). 
Besides its popularity, LSPI is a natural
candidate for such a comparison 
for three reasons: it also builds an approximator of fixed size out of a batch
of sample transitions, it has good theoretical guarantees, and it has been
successfully applied to several reinforcement learning tasks.

We compare the performance of LSPI and KBSF on the pole balancing task. Pole
balancing has a long history as a benchmark problem
because it represents a rich class of unstable 
systems~\cp{michie68boxes,anderson86learning,barto83neuronlike}. The
objective in this problem is to apply forces to a wheeled cart moving along a
limited track in order to keep one or more poles hinged to the cart from falling
over. 
There are several variations of the task with
different levels of difficulty; among them, balancing two poles 
side by side is particularly hard~\cp{wieland91evolving}.
In this paper we compare LSPI and KBSF on both the single- and two-poles
versions of the problem. We implemented the tasks using a realistic simulator
described by \ct{gomez2003robust}. We refer the reader to
Appendix~\ref{seca:exp_details} for details on the problems's configuration.

The experiments were carried out as described in the previous section,
with sample transitions collected by a random policy and then clustered by the
$k$-means algorithm.
In both versions of the pole-balancing task LSPI used the same data and
approximation architectures as KBSF.
To make the comparison with LSPI as fair as possible, 
we fixed the width of KBSF's kernel \kera\ at $\tau=1$ 
and varied \taub\ in $\{0.01, 0.1, 1\}$ for both algorithms.
Also, policy iteration was used
to find a decision policy for the MDPs constructed by KBSF, and this
algorithm was run for a maximum of $30$ iterations, the same limit used 
for LSPI.

Figure~\ref{fig:pole} shows the results of LSPI and KBSF on the single and
double pole-balancing tasks. We call attention to the fact
that the version of the problems used here is
significantly harder than the more commonly-used variants in which the decision
policies are evaluated on a single state close to the origin. This is probably
the reason why LSPI achieves a success rate of
no more than $60\%$ on the single pole-balancing task, as shown in
Figure~\ref{fig:pole_m}. In contrast, KBSF's decision policies are able to
balance the pole in $90\%$ of the attempts, on average, using as few as $m=30$
representative states.

\begin{figure*}
\centering
   \subfloat[Performance on single pole-balancing]{ 
   \label{fig:pole_m}
   \includegraphics[scale=\scll]{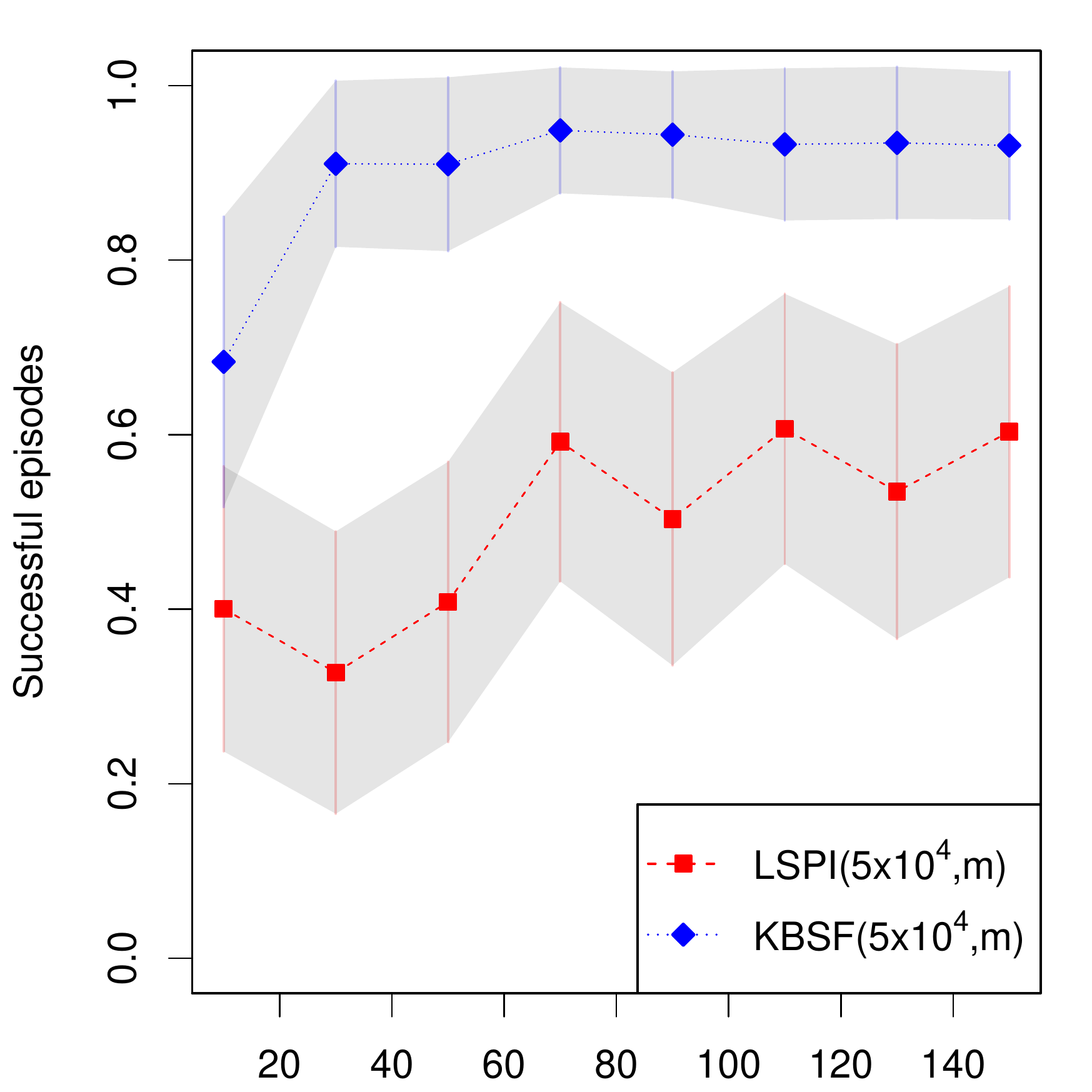} 
   }
    \subfloat[Run time on single pole-balancing] { 
   \label{fig:pole_m_time}
   \includegraphics[scale=\scll]{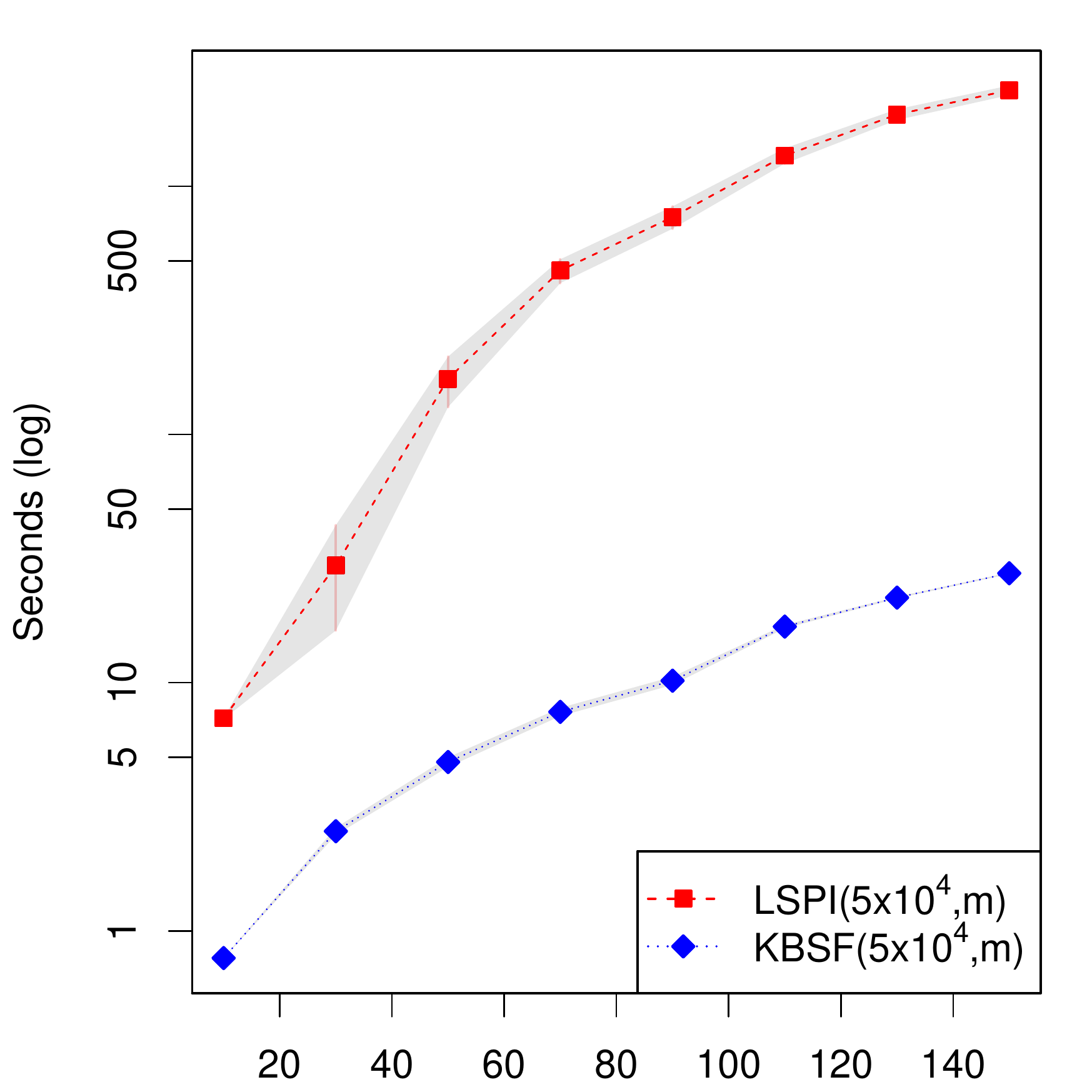}
    }
    
   \subfloat[Performance on double pole-balancing]{ 
   \label{fig:double_pole_m}
   \includegraphics[scale=\scll]{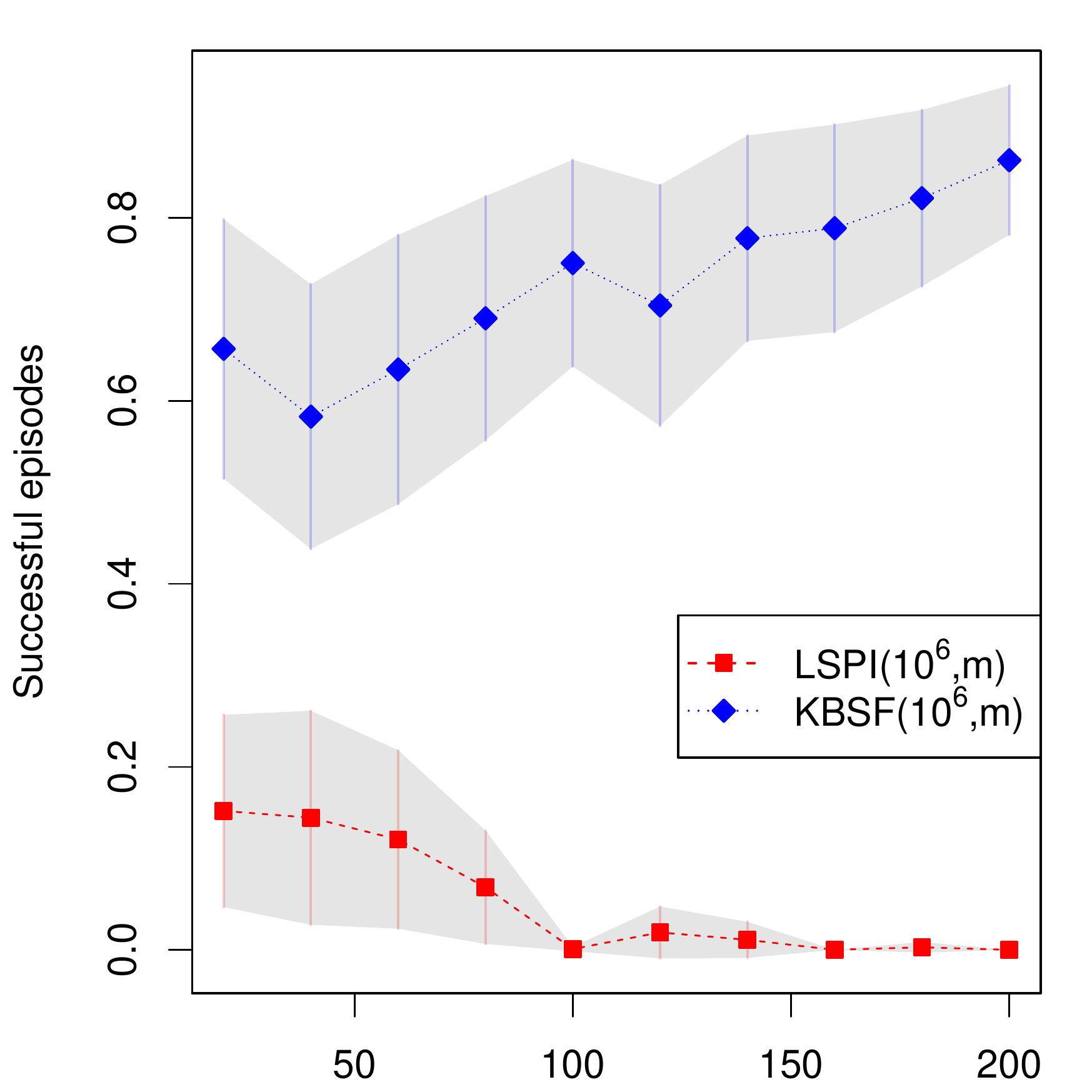} 
   }
    \subfloat[Run time on double pole-balancing] { 
   \label{fig:double_pole_m_time}
   \includegraphics[scale=\scll]{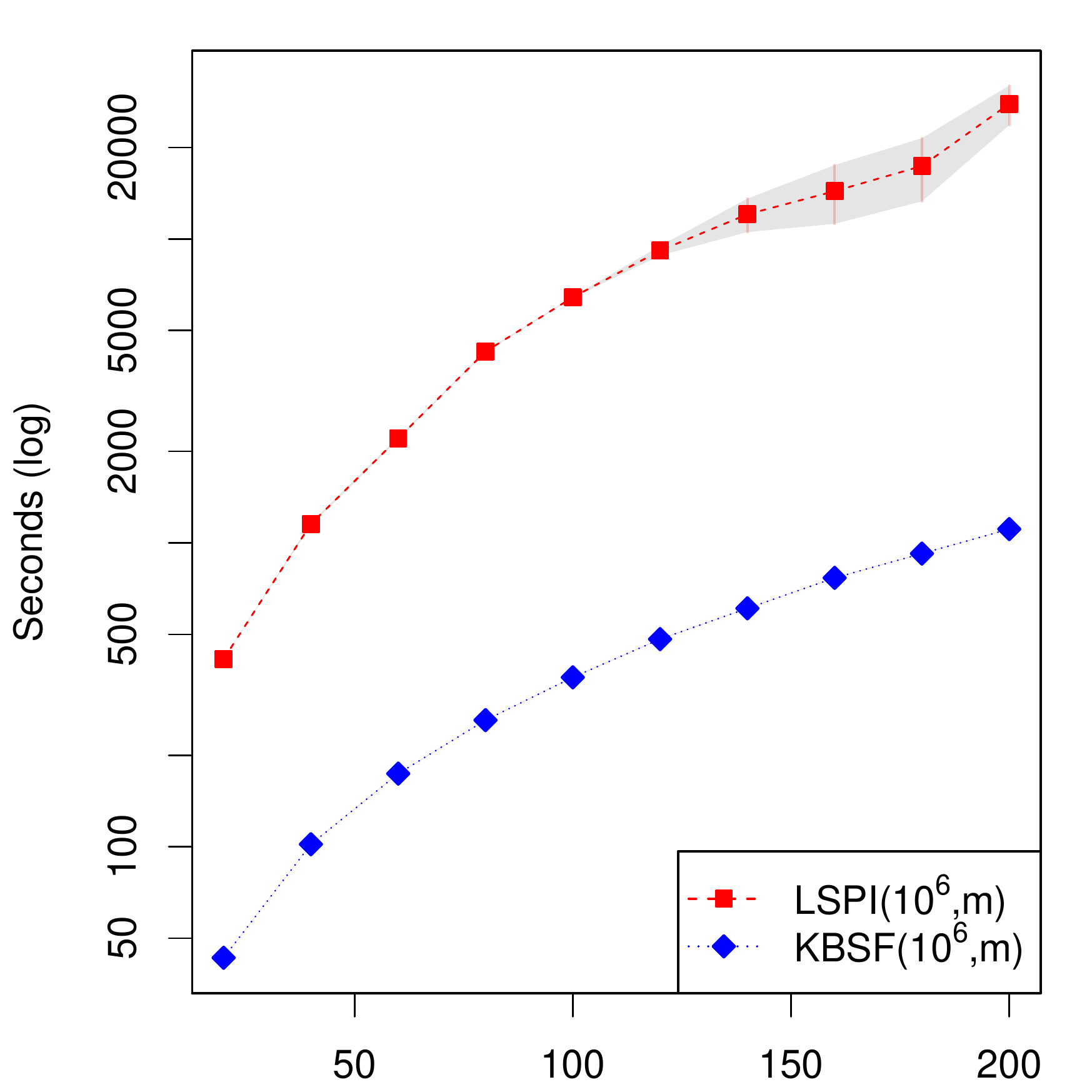}
    }
\caption{Results on the pole-balancing tasks,
as a function of the number of representative states $m$,
averaged over $50$
runs. The values correspond to the fraction of episodes initiated from the test
states in which the pole(s) could be balanced for $3\ts000$ steps (one minute of
simulated time). The test sets were regular grids defined over the hypercube
centered at the origin and covering $50\%$ of the state-space axes in each
dimension (see Appendix~\ref{seca:exp_details}). Shadowed regions represent
$99\%$ confidence intervals.\label{fig:pole}}
\end{figure*}

The results of KBSF on the double pole-balancing task are still more impressive.
As \ct{wieland91evolving} rightly points out, this version of the problem is
considerably more difficult than its single pole variant, and previous attempts
to apply reinforcement-learning techniques to this domain resulted in
disappointing performance~\cp{gomez2006efficient}. As shown in
Figure~\ref{fig:double_pole_m}, KBSF($10^{6}$, $200$) is able to achieve a
success rate of more than $80\%$. To put this number in perspective, recall
that some of the test states are quite challenging, with the two poles inclined
and falling in opposite directions.

The good performance of KBSF comes at a relatively low computational cost. A
conservative estimate reveals that, were KBRL($10^{6}$) run on the same computer
used for these experiments, we would have to wait for more than $6$
\emph{months} to see the results.
KBSF($10^{6}$, $200$) delivers a decision policy in
less than $7$ minutes. KBSF's computational cost also compares well with that of
LSPI, as shown in Figures~\ref{fig:pole_m_time}
and~\ref{fig:double_pole_m_time}. LSPI's policy-evaluation step involves the
update and solution of a linear system of equations, which take $O(nm^{2})$ and
$O(m^{3}|A|^{3})$, respectively. In addition, the policy-update stage requires
the definition of $\pi(\yka)$ for all $n$ states in the set of sample
transitions. In contrast, at each iteration KBSF only performs $O(m^{3})$
operations to evaluate a decision policy and $O(m^{2}|A|)$ operations to update
it. 

\subsubsection{HIV drug schedule (comparison with fitted $Q$-iteration)}
\label{sec:hiv}

We now compare KBSF with the fitted $Q$-iteration
algorithm~\cp{ernst2005tree,antos2007fitted,munos2008finite}. 
Fitted $Q$-iteration is a conceptually simple method that also builds
its approximation based solely on sample transitions.
Here we adopt this algorithm with 
an ensemble of trees generated by \ctp{geurts2006extremely} extra-trees
algorithm.
We will refer to the resulting method as
FQIT.

We chose FQIT for our comparisons because it has shown excellent
performance on both benchmark and real-world reinforcement-learning
tasks~\cp{ernst2005tree,ersnt2006clinical}.
In all experiments reported in this paper we used FQIT with ensembles of $30$
trees. As detailed in
Appendix~\ref{seca:exp_details}, besides the number of trees, FQIT has three
main parameters.
Among them, the minimum number of elements required to split a node
in the construction of the trees,
denoted here by \mne, has a particularly strong effect on both the
algorithm's performance and computational cost.
Thus, in our experiments we fixed FQIT's parameters at reasonable 
values---selected based on preliminary experiments---and
only varied \mne. The respective instances
of the tree-based approach are referred to as 
FQIT(\mne).

We compare FQIT and KBSF on an important medical problem which we will
refer to as the HIV drug schedule
domain~\cp{adams2004dynamic,ersnt2006clinical}.
Typical HIV treatments use drug cocktails containing two types of medication:
reverse transcriptase inhibitors (RTI) and protease inhibitors (PI). 
Despite the success of drug cocktails in maintaining
low viral loads, there are several complications associated with their
long-term use. This has attracted the interest of the scientific community
to the problem of optimizing drug-scheduling strategies. 
One strategy that has
been receiving a lot of attention recently is structured treatment interruption
(STI), in which patients undergo alternate cycles with and without the drugs.
Although many successful STI treatments have been reported in the literature,
as of now there is no consensus regarding the exact protocol that
should be followed~\cp{bajaria2004predicting}.

The scheduling of STI treatments can be seen as a sequential decision problem
in which the actions correspond to the types of cocktail that should
be administered to a patient~\cp{ersnt2006clinical}. 
To simplify the problem's formulation, 
it is assumed that RTI and PI drugs are administered at fixed amounts, 
reducing the actions to the four possible combinations of 
drugs: none, RTI only, PI only, or both. The goal is to minimize the viral load
using as little drugs as
possible. Following \ct{ersnt2006clinical}, 
we performed our experiments using a model 
that describes the interaction of
the immune system with HIV. This model was 
developed by~\ct{adams2004dynamic} and has been identified and validated
based on real clinical data. 
The resulting reinforcement learning task has a $6$-dimensional
continuous state space whose variables describe the overall
patient's condition.

We formulated the problem exactly as proposed 
by~\ct[see Appendix~\ref{seca:exp_details} for details]{ersnt2006clinical}.
The strategy used to generate the data also followed the protocol proposed
by these authors, which we now briefly explain. Starting from a batch
of $6\ts000$
sample transitions generated by a random policy, each algorithm first computed
an initial approximation of the problem's optimal value function. Based on this
approximation, a $0.15$-greedy policy was used to collect a second batch of
$6\ts000$ transitions, which was merged with the 
first.\footnote{As explained
by \ct{sutton98reinforcement}, an $\epsilon$-greedy policy selects the action
with maximum value with probability $1-\epsilon$, and 
with probability $\epsilon$ it picks an action uniformly at random.}
This process was
repeated for $10$ rounds, resulting in a total of $60\ts000$ sample transitions.

We varied FQIT's parameter \mne\ in the set 
$\{50, 100, 200\}$.
For the experiments with KBSF, we fixed $\tau = \taub = 1$
and varied $m$ in $\{2\ts000, 4\ts000, ..., 10\ts000\}$
(in the rounds in which $m \ge n$ we simply used all states \yia\ as 
representative states). 
As discussed in the beginning of this section, it is possible to 
reduce KBSF's computational cost with the use of sparse kernels. 
In our experiments with the HIV drug schedule task, we 
only computed the $\nk = 2$
largest values of $\gaussa(\rs_{i},\cdot)$ and 
the $\nkb = 3$ largest values of $\gaussb(\yia, \cdot)$ (see
Appendix~\ref{seca:algorithms}).
The representative states $\rs_{i}$ were selected at random from the set of
sampled states \yia\ (the reason for this will become clear shortly). 
Since in the current experiments the number
of sample transitions $n$ was fixed,
we will refer to the particular instances of our algorithm 
simply as KBSF($m$).

Figure~\ref{fig:hiv} shows the results obtained by FQIT and KBSF on the
HIV drug schedule task. As shown in Figure~\ref{fig:hiv_return_m}, FQIT's
performance improves when \mne\ is decreased, as expected. 
In contrast, increasing the number of representative states $m$ does not have a
strong impact on the quality of KBSF's solutions (in fact, in some cases the
average return obtained by the resulting policies decreases slightly when $m$
grows).
Overall, the performance of KBSF on the HIV drug schedule task is not
nearly as impressive as on the previous problems. For example, even when using
$m=10\ts000$ representative states, which
corresponds to one sixth of the sampled states,
KBSF is unable to reproduce the performance of FQIT with $\mne = 50$. 

\begin{figure*}
\centering
   \subfloat[Performance]{ 
   \label{fig:hiv_return_m}
   \includegraphics[scale=\scll]{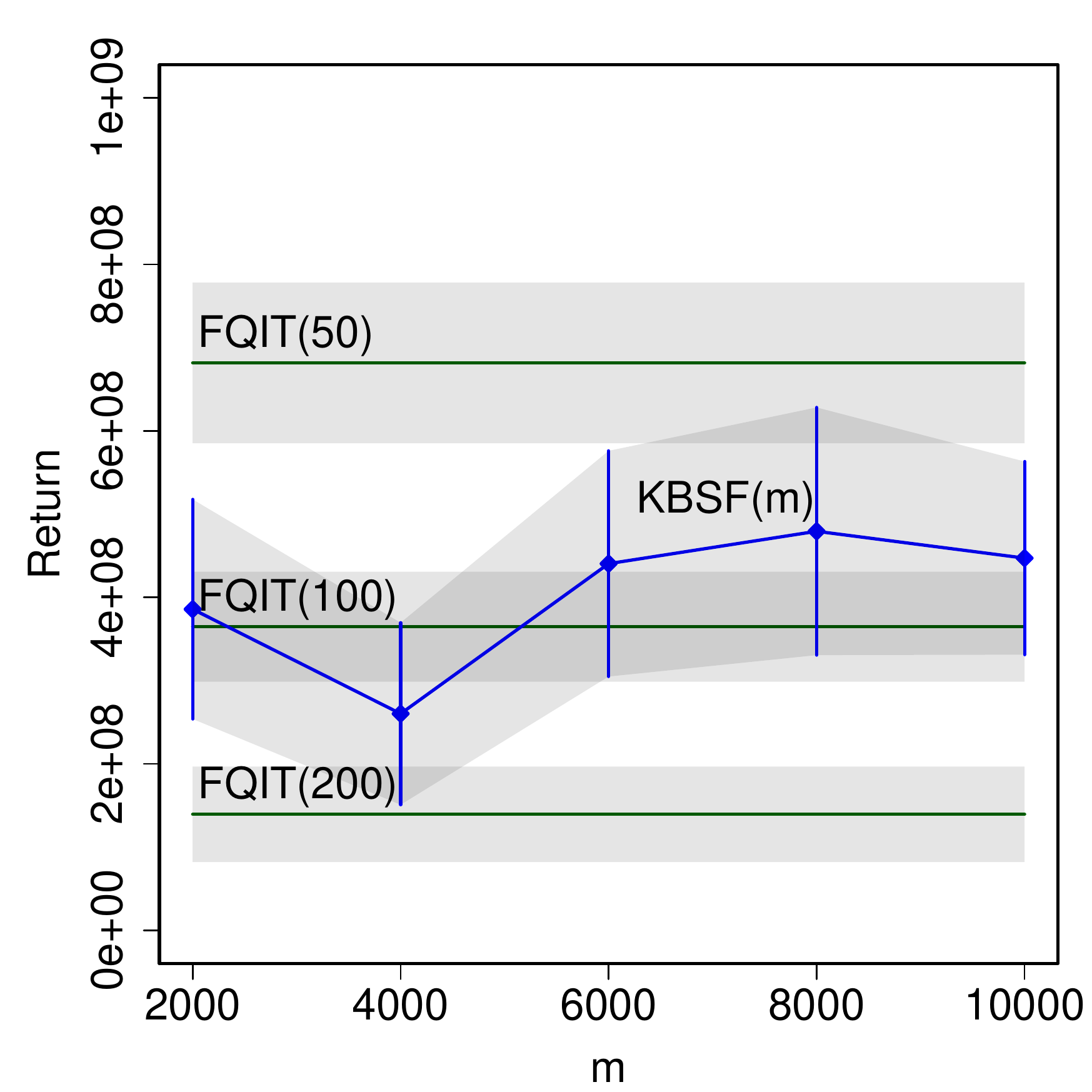} 
   }
    \subfloat[Run times] { 
   \label{fig:hiv_time_m}
   \includegraphics[scale=\scll]{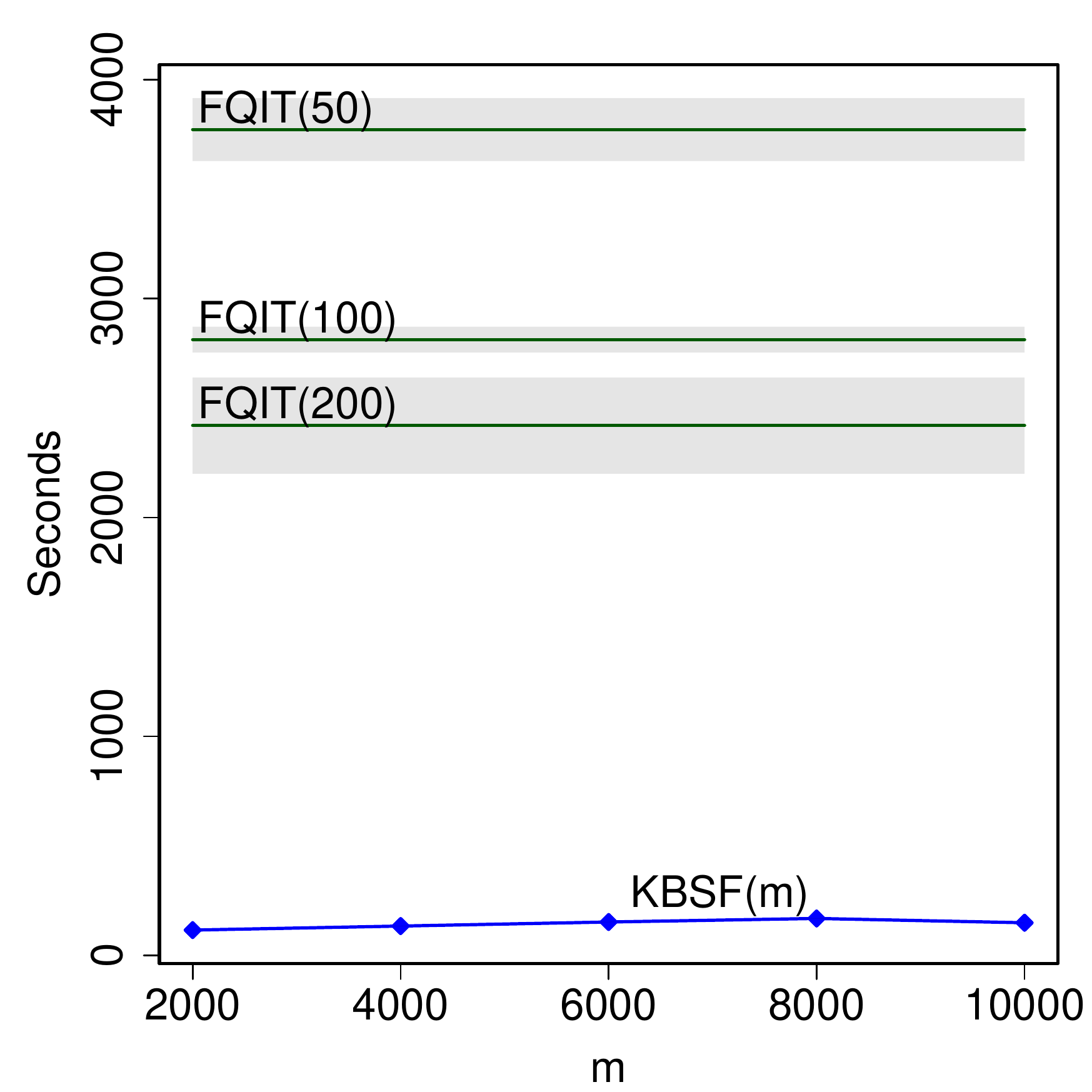}
    }

   \subfloat[Performance]{ 
   \label{fig:hiv_return}
   \includegraphics[scale=\scll]{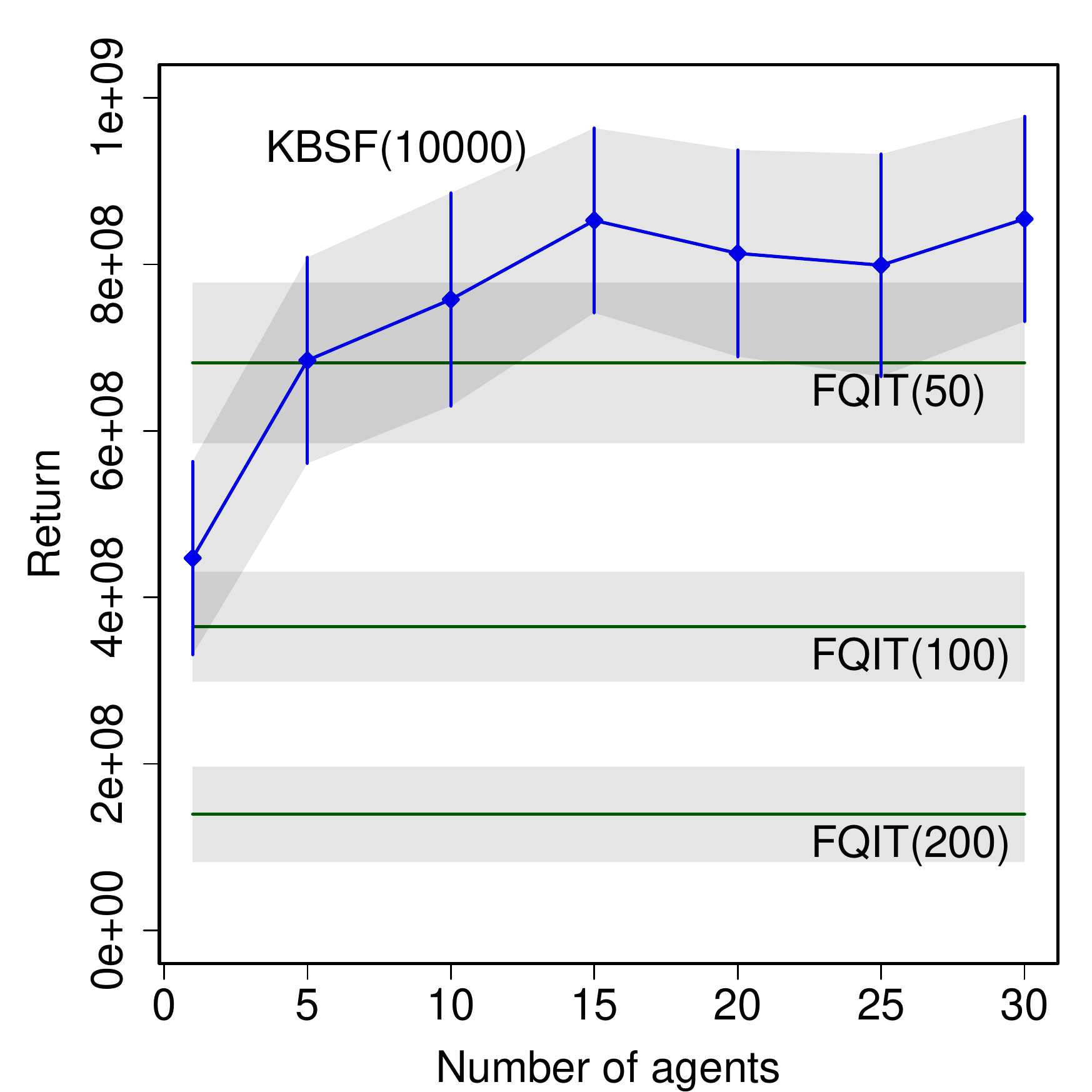} 
   }
    \subfloat[Run times] { 
   \label{fig:hiv_time}
   \includegraphics[scale=\scll]{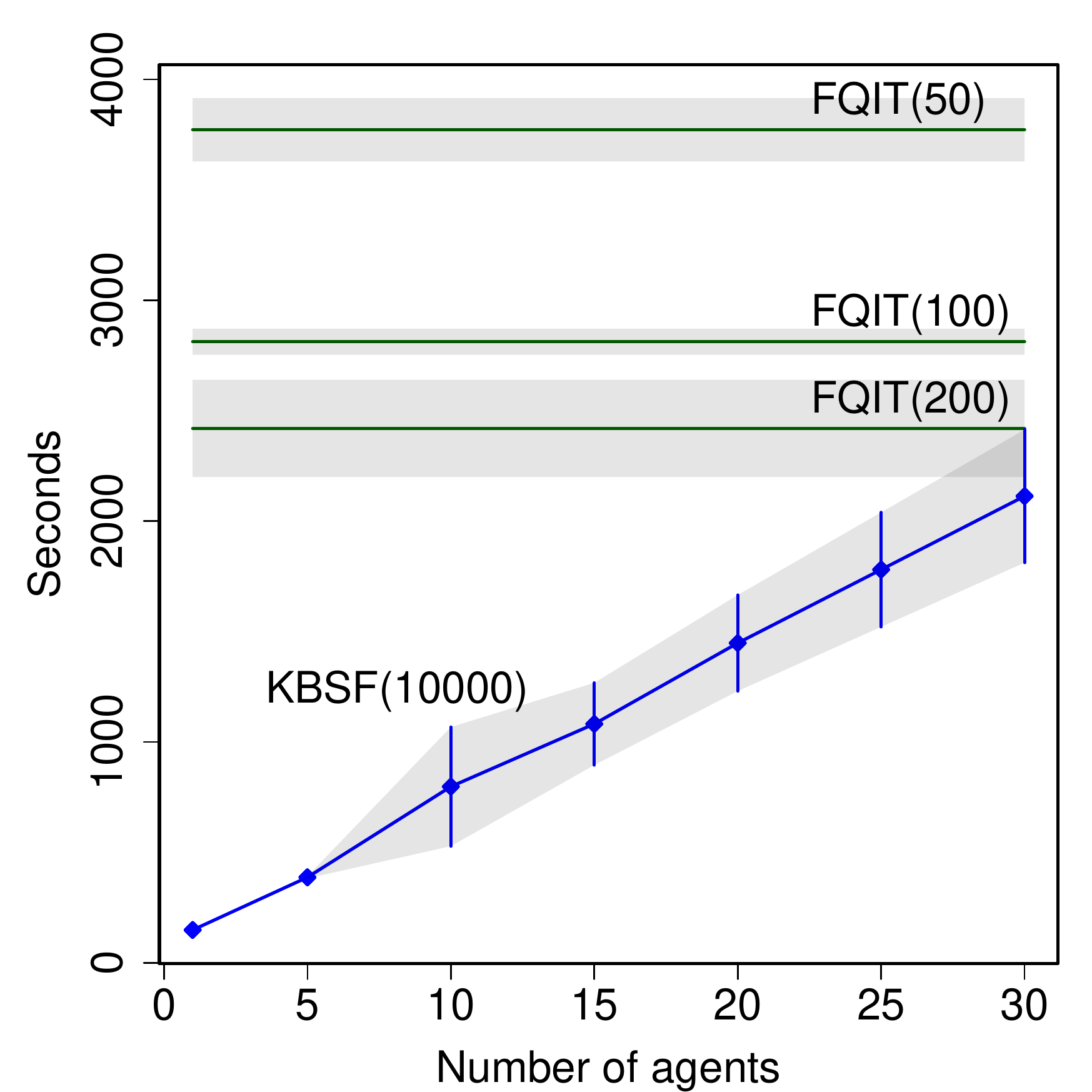}
    }
\caption{Results on the HIV drug schedule task averaged over $50$ runs. The
STI policies were evaluated for $5\ts000$ days starting from a state
representing a patient's unhealthy state (see
Appendix~\ref{seca:exp_details}). The shadowed regions
represent $99\%$ confidence intervals. \label{fig:hiv}}
\end{figure*} 

On the other hand, when we look at Figure~\ref{fig:hiv_time_m}, it is clear
that the difference on the algorithms's performance is counterbalanced by a 
substantial difference on the associated computational costs. As an
illustration, note that 
KBSF($10\ts000$) is $15$ times faster than FQTI($100$) and $20$ times faster 
than FQTI($50$).
This difference on the algorithms's run times is expected, since each iteration
of FQIT involves the construction (or update) of an ensemble of trees, each one
requiring  at least $O(n \log (n/\mne))$ operations, and the improvement of the
current decision policy, which is $O(n|A|)$~\cp{geurts2006extremely}. As
discussed before, KBSF's efficiency comes from the fact that its computational
cost per iteration is independent of the number of sample transitions $n$.

Note that the fact that FQIT uses an ensemble of trees is both a blessing and a
curse. If on the one hand this reduces the variance of the approximation, on
the other hand it also increases the algorithm's computational 
cost~\cp{geurts2006extremely}. Given the
big gap between FQIT's and KBSF's time complexities, one may wonder if
the latter can also benefit from averaging over several models. In order to
verify this hypothesis, we implemented a very simple model-averaging strategy
with KBSF: we trained several agents independently, using 
Algorithm~\ref{alg:kbsf} on the same set of sample
transitions, and then put them together on a single ``committee''. 
In order to increase the variability within the committee of
agents, instead of using $k$-means to determine the representative states 
$\rs_{j}$ we simply selected them uniformly at random from the set of sampled
states \yia\ (note that this has the extra benefit of reducing the method's
overall computational cost).
The actions selected by the committee of agents were determined by
``voting''---that is, we simply picked the action chosen by the majority of
agents, with ties
broken randomly. 

We do not claim that the approach described above is the best 
model-averaging strategy to be used with 
KBSF. However, it seems to be
sufficient to boost the algorithm's performance considerably, as shown in 
Figure~\ref{fig:hiv_return}. Note how KBSF already performs 
comparably to FQTI($50$) when
using only $5$ agents in the committee. When this number is increased 
to $15$, the expected return of KBSF's agents is considerably larger than 
that of the best FQIT's agent, with only a small overlap between 
the $99\%$ confidence intervals associated with the algorithms's results. The
good performance of KBSF is still more impressive when we look at
Figure~\ref{fig:hiv_time}, which shows that even when using a committee of $30$
agents this algorithm is faster than FQIT($200$).

In concluding, we should mention that, overall, our experience with FQIT
confirms \ctp{ernst2005tree} report: it is a
stable, easy-to-configure method that usually delivers good solutions. 
In fact, given the algorithm's ease of use, when the problem at hand can be
solved off-line using a moderate number of sample transitions, FQIT 
may be a very good alternative. 
On the other hand, for on-line problems or off-line problems involving a large
number of sample transitions, FQIT's computational cost can be prohibitive in
practice. In Section~\ref{sec:triple_pole} we will discuss an experiment in
which such a computational demand effectively precludes the use of this
algorithm.

\subsubsection{Epilepsy suppression (comparison with LSPI and fitted
$Q$-iteration)}
\label{sec:epilepsy}

We conclude our empirical evaluation of KBSF by using it to learn a
neuro-stimulation policy for the treatment of epilepsy. 
It has been shown that the electrical stimulation
of specific structures in the neural system at fixed
frequencies can effectively suppress the occurrence of
seizures~\cp{durand2001supression}. Unfortunately, {\sl in vitro}
neuro-stimulation experiments
suggest that fixed-frequency pulses are not equally effective across epileptic
systems. Moreover, the long term use of this treatment may
potentially damage the patients's neural tissues. Therefore, it is desirable to
develop neuro-stimulation policies that replace the fixed-stimulation regime
with an adaptive scheme.

The search for efficient neuro-stimulation strategies can be seen as a
reinforcement learning problem. Here we study it using a generative
model developed by~\ct{bush2009manifold} based on real data collected from
epileptic rat hippocampus slices. This model was shown to reproduce the seizure
pattern of the original dynamical system and 
was later validated through the deployment of a learned treatment policy on a
real brain slice~\cp{bush2009manifold2}. The associated decision problem has a
five-dimensional continuous state space and highly non-linear dynamics. At
each time step the agent must choose whether or not to apply an electrical
pulse. The goal is to suppress seizures as much as possible while minimizing
the total amount of stimulation needed to do so.
 
The experiments were performed as described in Section~\ref{sec:puddle}, with
a single batch of sample transitions collected by a policy that selects 
actions uniformly at random. 
Specifically, the random policy was used to collect $50$ trajectories of
length $10\ts000$, resulting in a total of $500\ts000$ sample transitions.
We use as a baseline for our comparisons the 
already mentioned fixed-frequency stimulation policies usually adopted in {\sl
in vitro} clinical studies~\cp{bush2009manifold2}. 
In particular, we considered policies that apply electrical pulses at
frequencies of $0$ Hz, $0.5$ Hz, $1$ Hz, and $1.5$ Hz. 

We compare KBSF with LSPI and FQIT.
For this task we ran both LSPI and KBSF with sparse kernels, that is, we only
computed the kernels at the $6$-nearest neighbors of a
given state ($\nk=\nkb=6$; see Appendix~\ref{seca:algorithms} for
details). This modification
made it possible to use $m = 50\ts000$ representative states with KBSF. Since
for LSPI the reduction on the computational cost was not very significant, we
fixed $m = 50$ to keep its run time within reasonable bounds.
Again, KBSF and LSPI used the same approximation architectures, with 
representative states defined by the $k$-means algorithm.
We fixed $\tau = 1$ and varied $\taub$ in $\{0.01, 0.1, 1\}$.
FQIT was configured as described in the previous section, with the parameter
\mne\ varying in $\{20, 30, ...,  200\}$. In general, we observed that
the performance of the tree-based method improved with smaller values for \mne,
with an expected increase in the computational cost. Thus, in order to give an
overall characterization of FQIT's performance, we only report
the results obtained with the extreme values of $\mne$.

Figure~\ref{fig:epilepsy} shows the results on the epilepsy-suppression task.
In order to obtain different compromises between the problem's two conflicting
objectives, we varied the relative magnitude of the penalties associated with 
the occurrence of seizures and with the application of an electrical 
pulse~\cp{bush2009manifold,bush2009manifold2}.
Specifically, we fixed the latter at $-1$ and varied the former 
with values in $\{-10, -20, -40\}$. This appears in the plots as subscripts next
to the algorithms's names. 
As shown in Figure~\ref{fig:ep_pareto}, LSPI's policies seem to prioritize
reduction of stimulation  at the expense of higher seizure occurrence, which is
clearly sub-optimal from a clinical point of view. FQIT($200$) also performs
poorly, with solutions representing no advance over the fixed-frequency
stimulation strategies. In contrast, FQTI($20$) and KBSF are both able to
generate decision policies that are superior to the 1~Hz policy, which is the
most efficient stimulation regime known to date in the clinical
literature~\cp{jerger95periodic}. However, as shown in
Figure~\ref{fig:ep_times}, KBSF is able to do it at least $100$ times faster
than the tree-based method.

\begin{figure*}
\centering
   \subfloat[Performance. The length of the rectangles's edges represent $99\%$
confidence intervals.]{ 
   \label{fig:ep_pareto}
   \includegraphics[width=3.5in]{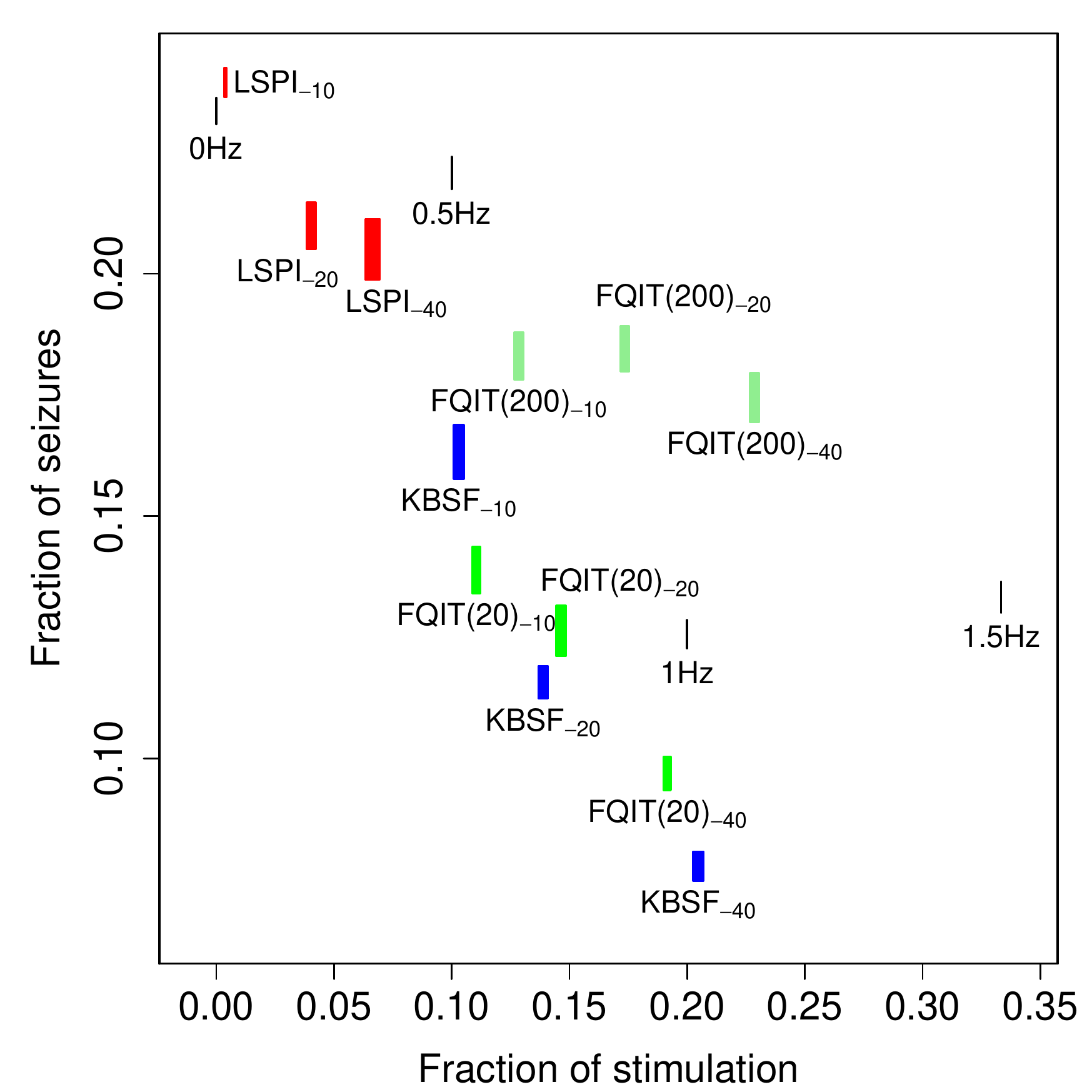}
    } 
   \subfloat[Run times (confidence intervals do not show up in logarithmic
scale) ]{
    \label{fig:ep_times}
    \includegraphics[height=3.5in]{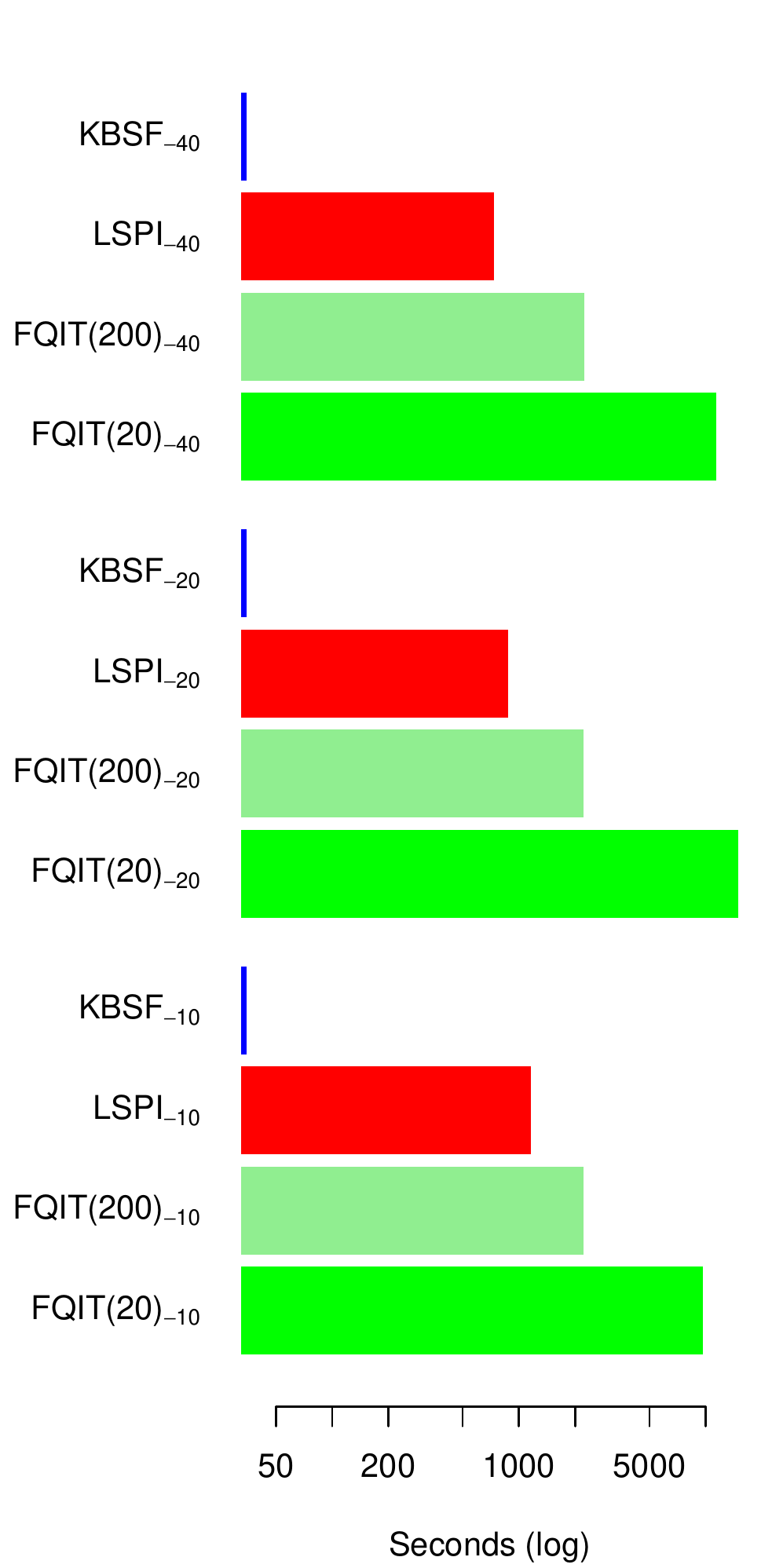} 
   }
\caption{Results on the epilepsy-suppression problem averaged over $50$
runs. The decision policies were evaluated on  episodes
of $10^{5}$ transitions starting from a fixed set of $10$ test states
drawn uniformly at random. \label{fig:epilepsy}}
\end{figure*}

\section{Incremental KBSF}
\label{sec:ikbsf}

As clear in the previous section, one characteristic of KBSF
that sets it apart from other methods is its low demand in terms of
computational resources. Specifically, both time and memory complexities of
our algorithm are linear in the number of sample transitions $n$. In terms of
the number of operations performed by the algorithm, this is the best one can do
without discarding transitions. However, in terms of memory usage, it is
possible to do even better. In this section we show how to build KBSF's
approximation 
incrementally, without ever having access to the entire set of sample
transitions at once.
Besides reducing the memory complexity of the algorithm, this 
modification has the
additional advantage of making KBSF suitable for on-line reinforcement
learning.

In the batch version of KBSF, described in 
Section~\ref{sec:kbsf},
the matrices \Pba\ and vectors \rba\ are
determined using all the transitions in the corresponding sets \Sca.
This has two undesirable
consequences. First, the construction of the MDP \bM\ requires an amount of
memory of $O(\hat{n}m)$. Although this is a
significant improvement over KBRL's memory usage, which is 
lower bounded by $(\min_{a}{n_a})^{2}|A|$, in more challenging domains
even a linear dependence on $\hat{n}$ may
be impractical. Second, in the batch version of KBSF the only 
way to incorporate new data into the model \bM\ is to recompute 
the multiplication $\Pba = \Kda \Dda$ for all actions $a$ for which 
there are new sample transitions available. Even if we ignore the
issue with memory usage, this is clearly inefficient in terms of
computation.
In what follows we 
present an incremental version of KBSF that circumvents these important
limitations~\cp{barreto2012online}.

We assume the same scenario considered in Section~\ref{sec:kbsf}: there is a
set of sample transitions $\Sca = \{(\xka,\rka,\yka)| k = 1, 2,
...,n_{a}\}$ 
associated with each action $a \in A$, where $\xka, \yka \in \Sc$ and
$\rka \in \R$, and a set of representative states 
$\Sb = \{\rs_{1}, \rs_{2}, ..., \rs_{m}\}$, with $\rs_{i} \in \Sc$.
Suppose now that we split the set of sample transitions \Sca\ in two 
subsets $S_{1}$ and $S_{2}$ such that 
$S_{1} \cap S_{2} = \emptyset$ and 
$S_{1} \cup S_{2} = \Sca$ (we drop the ``$a$'' superscript 
in the sets $S_{1}$ and $S_{2}$ to improve clarity).
 Without loss of generality, suppose that the sample
transitions are indexed so that
\begin{equation*}
S_{1} \equiv \{(\xka,\rka,\yka)| k = 1, 2, ..., n_{1}\} 
\mbox{ and }
S_{2} \equiv \{(\xka,\rka,\yka)| k = n_{1} + 1, n_{1}+ 2, ...,
n_{1} + n_{2} = n_{a}\}.
\end{equation*}
Let \Pban\ and \rban\ be matrix \Pba\ and vector \rba\
computed by KBSF using only the $n_{1}$ transitions in $S_{1}$
(if $n_{1}=0$, we define $\Pban = \mat{0} \in \R^{m \times m}$ and 
$\rban = \mat{0} \in \R^{m}$ for all $a \in A$).
We want to compute 
\Pbann\ and \rbann\ from \Pban, \rban, and $S_{2}$, without using the set of
sample transitions~$S_{1}$.

We start with the transition matrices \Pba. We know that 
\begin{equation*}
\begin{array}{cl}
\bar{p}^{_{S_1}}_{ij} & = \sum_{t=1}^{n_{1}} \dot{k}^{a}_{it} \dot{d}^{a}_{tj}
= \sum_{t=1}^ {n_{1}} \dfrac{\gaussa(\rs_{i}, \xta)}{\sum_{l=1}^{n_{1}}
\gaussa(\rs_{i}, \xla)} 
\dfrac{\gaussb(\yta, \rs_{j})}{\sum_{l=1}^{m} \gaussb(\yta, \rs_{l})} \\ \\
& = 
\dfrac{1}{\sum_{l=1}^{n_{1}} \gaussa(\rs_{i}, \xla)}
\sum_{t=1}^ {n_{1}} \dfrac{\gaussa(\rs_{i}, \xta) \gaussb(\yta,
\rs_{j})}{\sum_{l=1}^{m} \gaussb(\yta, \rs_{l})}. 
\end{array}
\end{equation*}
To simplify the notation, define
\begin{equation*}
w^{_{S_1}}_{i} = \sum_{l=1}^{n_{1}} \gaussa(\rs_{i}, \xla), \;
 w^{_{S_2}}_{i} = \sum_{l=n_{1}+1}^{n_{1}+n_{2}} \gaussa(\rs_{i}, \xla),
\mbox{ and }
b^{t}_{ij} = \frac{\gaussa(\rs_{i}, \xta) \gaussb(\yta, \rs_{j})}{\sum_{l=1}^{m}
\gaussb(\yta, \rs_{l})}, 
\end{equation*}
with $t \in \{1, 2, ..., n_{1} + n_{2}\}$. Then, we can write
\begin{equation*}
\begin{array}{cl}
\bar{p}^{_{S_1 \cup S_2}}_{ij} 
= \dfrac{1}{w^{_{S_1}}_{i} + w^{_{S_2}}_{i}} \left(\sum_{t=1}^{n_{1}}
b^{t}_{ij}
+ \sum_{t=n_{1}+1}^{n_{1}+n_{2}} b^{t}_{ij} \right) 
= \dfrac{1}{w^{_{S_1}}_{i} + w^{_{S_2}}_{i}} 
\left(
\bar{p}_{ij}^{_{S_1}} w^{_{S_1}}_{i} + \sum_{t=n_{1}+1}^{n_{1}+n_{2}}
b^{t}_{ij}
\right).
\end{array}
\end{equation*}
Now, defining $b_{ij}^{_{S_2}} = \sum_{t=n_{1}+1}^{n_{1}+n_{2}} b^{t}_{ij}$, we
have the 
simple update rule:
\begin{equation}
\mbox{\fbox{
$\label{eq:upd_p}
\bar{p}^{_{S_1 \cup S_2}}_{ij} = 
\dfrac{1}{w^{_{S_1}}_{i} + w^{_{S_2}}_{i}} \left(b_{ij}^{_{S_2}} +
\bar{p}_{ij}^{_{S_1}} w^{_{S_1}}_{i} \right)
$
}}\;.
\end{equation}

We can apply similar reasoning to derive an update rule for the 
rewards $\bar{r}_{i}^{a}$. We know that
\begin{equation*}
\bar{r}^{_{S_1}}_{i} 
= \frac{1}{\sum_{l=1}^{n_{1}} \gaussa(\rs_{i}, \xla)} \sum_{t=1}^{n_{1}}
\gaussa(\rs_{i}, \xta) r^{a}_{t}
=  \frac{1}{w^{_{S_1}}_{i}} \sum_{t=1}^{n_{1}}
\gaussa(\rs_{i}, \xta) r^{a}_{t}.
\end{equation*}
Let $e^{t}_{i} = \gaussa(\rs_{i}, \xta) r^{a}_{t}$,
with $t \in \{1, 2, ..., n_{1} + n_{2}\}$. Then, 
\begin{equation*}
\begin{array}{cl}
\bar{r}^{_{S_1 \cup S_2}}_{i} 
= \dfrac{1}{w^{_{S_1}}_{i} + w^{_{S_2}}_{i}} 
\left(\sum_{t=1}^{n_{1}} e^{t}_{i} + \sum_{t=n_{1} + 1}^{n_{1} + n_{2}}
e^{t}_{i} \right)  
 = \dfrac{1}{w^{_{S_1}}_{i} + w^{_{S_2}}_{i}} 
\left(w^{_{S_1}}_{i} \bar{r}_{i}^{_{S_1}} + \sum_{t=n_{1} + 1}^{n_{1} +
n_{2}}
e^{t}_{i}
\right) \;. 
\end{array}
\end{equation*}
Defining $e^{_{S_2}}_{i} = \sum_{t=n_{1}+1}^{n_{1}+n_{2}} e^{t}_{i} $,
we have the following update rule:

\begin{equation}
\mbox{\fbox{
$
\label{eq:upd_r}
\bar{r}^{_{S_1 \cup S_2}}_{i} = 
\dfrac{1}{w^{_{S_1}}_{i} + w^{_{S_2}}_{i}} \left(e^{_{S_2}}_{i} +
\bar{r}^{_{S_1}}_{i} w^{_{S_1}}_{i} \right)
$
}} \;.
\end{equation}
Since $b_{ij}^{_{S_2}}$, $e^{_{S_2}}_{i}$, and $w^{_{S_2}}_{i}$ can
be computed based on $S_{2}$ only, we can discard the sample transitions
in $S_{1}$ after computing \Pban\ and \rban. 
To do that, we only have to keep the variables
$w^{_{S_1}}_{i}$. These variables can be stored in $|A|$ vectors 
$\wa \in \R^{m}$, resulting in a modest memory overhead.
Note that we can apply the ideas above recursively, further splitting the sets
$S_{1}$ and $S_{2}$ in subsets of smaller size. Thus, we
have a fully incremental way of computing KBSF's MDP which requires almost no
extra memory. 

Algorithm~\ref{alg:kbsf_upd} shows a step-by-step
description of how to update \bM\ based on a set of sample
transitions. Using this method to update its model, 
KBSF's space complexity drops from $O(\hat{n}m)$ to $O(m^{2})$. 
Since the amount of memory used by KBSF is now independent of
$n$, it can process an arbitrary number of sample
transitions (or, more precisely, the limit on the amount of data it
can process is dictated by time only, not space).

\begin{algorithm}
   \caption{Update KBSF's MDP}
   \label{alg:kbsf_upd}
\begin{algorithmic}
   \State {\bfseries Input:} 
\begin{tabular}{llr}
\Pba, \rba, \wa\  for all $a \in A$
& \Comment{Current model} \\
$\Sca = \{(\xka,\rka,\yka)| k = 1, 2, ..., n_{a}\}$  for all $a \in A$ 
& \hspace{5mm} \Comment{Sample transitions} \\

\end{tabular}
   \State {\bfseries Output:} Updated \bM\ and \wa\
    \vspace{2mm}
   \For{$a \in A$}
   \State {\bf for} $t = 1,...,n_a$ {\bf do} 
$z_{t} \la {\sum_{l=1}^{m} \gaussb(\yta, \rs_{l})}$
   \State $n_a \la |\Sca|$
    \For{$i = 1, 2, ..., m$}
    \State $w' \la \sum_{t=1}^{n_a} \gaussa(\rs_{i}, \xta)$
   \For{$j = 1, 2, ..., m$}
   \State $b \la {\sum_{t=1}^{n_a} {\gaussa(\rs_{i},
\xta) \gaussb(\yta, \rs_{j})} / z_{t} }$
    \State $\bar{p}_{ij} \la 
\dfrac{1}{w^{a}_{i} + w'} \left(b + \bar{p}_{ij} w^{a}_{i}
\right)$
\Comment{Update transition probabilities}
    \EndFor
    \State $e \la \sum_{t=1}^{n_a} \gaussa(\rs_{i}, \xta)
r^{a}_{t}$
    \State $\bar{r}_{i} \la 
\dfrac{1}{w^{a}_{i} + w'} \left(e + \bar{r}_{i} w^{a}_{i} \right)$
\Comment{Update rewards}
    \State $w^{a}_{i} \la w^{a}_{i} + w'$
\Comment{Update normalization factor}
    \EndFor
    \EndFor
\end{algorithmic}
\end{algorithm}

Instead of assuming that $S_{1}$ and $S_{2}$ are a partition
of a fixed data set $\Sca$, we can consider that $S_{2}$ was
generated based on the policy learned by KBSF
using the transitions in $S_{1}$.
Thus, Algorithm~\ref{alg:kbsf_upd} provides
a flexible framework for integrating learning and planning within KBSF.
Specifically, our algorithm 
can cycle between learning a model of the problem based on sample
transitions, using such a model to derive a policy, and resorting to this policy
to collect more data.
Algorithm~\ref{alg:inc_kbsf} shows a possible implementation of this framework. 
In order to distinguish it from its batch counterpart, we will call the
incremental version of our algorithm \ikbsf. \ikbsf\ updates the model \bM\ and
the value function \Qb\ at fixed intervals $t_{m}$ and $t_{v}$, respectively.
When $t_{m} = t_{v} = n$, we recover the batch version of
KBSF; when $t_{m} = t_{v} = 1$, we have a fully on-line method which
stores no sample transitions.

\begin{algorithm}
   \caption{Incremental KBSF (\ikbsf)} 
   \label{alg:inc_kbsf}
\begin{algorithmic}
   \State {\bfseries Input:} 
\begin{tabular}{lr}
$\Sb = \{\rs_{1}, \rs_{2}, ..., \rs_{m}\}$ 
& \hspace{35mm} \Comment{Set of representative states}\\
 $t_{m}$ 
& \Comment{Interval to update model} \\
$t_{v}$ 
& \Comment{Interval to update value function} \\
\end{tabular}
   \State {\bfseries Output:} Approximate value function $\tilde{Q}(s,a)$
    \vspace{2mm}
\State $\Pba \la \mat{0} \in \R^{m \times m}$, $\rba \la \mat{0} \in
\R^{m}$, $\wa \la \mat{0} \in \R^{m}$, for all $a \in A$
\State $\Qb \la$ arbitrary matrix in $\R^{m \times |A|}$
\State $s \la $ initial state
\State $a \la $ random action
\For{$t \la 1, 2, ...$}
\State Execute $a$ in $s$ and observe $r$ and $\hat{s}$ 
\State $\Sca \la \Sca \bigcup \{(s ,r, \hat{s})\}$
\If{($t \mod t_{m} = 0$)} 
\Comment{Update model}
\State Add new representative states to \bM\ using $\Sca$
\Comment{This step is optional}
\State Update \bM\ and \wa\ using Algorithm~\ref{alg:kbsf_upd} and \Sca\
\State $\Sca \la \emptyset$ for all $a \in A$ \Comment{Discard transitions}
\EndIf
\State {\bf if} {($t \mod t_{v} = 0$)} update \Qb\ 
\Comment{Update value function}
\State $s \la \hat{s}$ 
\State Select $a$ based on $\tilde{Q}(s,a) = \sum_{i=1}^{m} \kerb(s, \rs_{i})
\bar{q}_{ia}$
\EndFor  
\end{algorithmic}
\end{algorithm}

Algorithm~\ref{alg:inc_kbsf} also allows for the inclusion of new 
representative states to the model \bM. Using Algorithm~\ref{alg:kbsf_upd} this
is easy to do: given a new representative state $\rs_{m+1}$, it
suffices to set $w^{a}_{m+1} = 0$, $\bar{r}^{a}_{m+1} = 0$, 
and $\bar{p}_{m+1, j} = \bar{p}_{j, m+1} = 0$ for $j = 1, 2, ..., m+1$ and all
$a \in A$. 
Then, in the following applications of update rules~(\ref{eq:upd_p})
and~(\ref{eq:upd_r}), the dynamics of \bM\ will naturally reflect the existence
of state $\rs_{m+1}$.
Note that the inclusion of new representative states does not destroy the
information already in the model. This allows \ikbsf\ to refine its
approximation 
on the fly, as needed. One can think of several ways of detecting the need
for new representative 
states. A simple strategy, based on Proposition~\ref{teo:bound_rep_states}, is
to impose a maximum distance allowed between a sampled state \yia\ and the
nearest representative
state,  $\ddb(\yia, 1)$. Thus, anytime the agent encounters a new state 
\yia\ for which $\ddb(\yia, 1)$ is above a given
threshold, \yia\ is added to the model as $\rs_{m+1}$. In
Section~\ref{sec:empirical_inc} we
report experiments with \ikbsf\ using this approach. Before that, though, we
discuss the theoretical properties of the incremental version of our algorithm.

\subsection{Theoretical results}
\label{sec:theory_inc}

As discussed, \ikbsf\ does not need to store sample transitions to
build its approximation. However, the computation of $\tilde{Q}(s,a)$ 
through~(\ref{eq:kbsf_q}) requires all the tuples $(\xia,\ria,\yia)$ to be
available. In some situations, it may be feasible to keep the transitions
in order to compute $\tilde{Q}(s,a)$. 
However, if we want to use \ikbsf\ to its full extend, we need a way of
computing $\tilde{Q}(s,a)$ without using the sample transitions.
This is why upon reaching state $s$ at time step $t$ \ikbsf\ 
selects the action to be performed based on
\begin{equation}
\label{eq:ikbsf_q}
\tilde{Q}_{t}(s,a) = \sum_{i=1}^{m} \kerb(s, \rs_{i})
\bar{Q}_{t} (\rs_{i}, a),
\end{equation}
where $\bar{Q}_{t} (\rs_{i}, a)$ is the action-value function 
available to \ikbsf\ at the $\tth$ iteration  (see
Algorithm~\ref{alg:inc_kbsf}). Note that we do not assume that \ikbsf\ has
computed the optimal value function of its current model $\bM_{t}$---that is, it
may be the case that $\bar{Q}_{t} (\rs_{i}, a) \ne \bar{Q}_{t}^{*} (\rs_{i},
a)$. 

Unfortunately, when we replace~(\ref{eq:kbsf_q}) with~(\ref{eq:ikbsf_q})
Proposition~\ref{teo:batch_kbsf} no longer applies. In this section we 
address this issue by deriving an upper bound for the difference between
$\tilde{Q}_{t}(s,a)$ and $\cQ_{t}(s, a)$, the action-value function that
would be computed by KBRL using all the transitions processed by \ikbsf\ up to
time step $t$. In order to derive our bound, we assume that \ikbsf\ 
uses a fixed set \Sb---meaning
that no representative states are added to the model \bM---and that it
never stops refining its model, doing so at every iteration
$t$ ({\sl i.e.}, $t_{m} = 1$ in Algorithm~\ref{alg:inc_kbsf}).
We start by showing the following lemma, proved in
Appendix~\ref{seca:theory}:

\begin{lemma}
\label{teo:bound_q}
Let $M \equiv (S, A, \Pa, \ra,\gamma)$ and 
$\tilde{M} \equiv (S, A, \Pta, \rta,\gamma)$ be two finite MDPs. Then,
for any $s \in S$ and any $a \in A$,
\begin{equation*}
|{Q}^{*}(s, a) - \tilde{Q}^{*}(s, a)| \le 
\frac{1}{1-\gamma} \maxinf{\ra}{\rta} + \frac{\gamma (2 -
\gamma)}{2(1-\gamma)^{2}} 
R_{\dif} 
\maxinf{\Pa}{\Pta}, 
\end{equation*}
where 
${R}_{\dif} = \max_{a,i} {r}^{a}_{i} - \min_{a,i}
{r}^{a}_{i}$.
\end{lemma}

Lemma~\ref{teo:bound_q} provides an upper bound for the difference in the
action-value functions of any two MDPs having the same state space $S$, action
space $A$, and discount factor
$\gamma$.\footnote{\ctp{strehl2008analysis}
Lemma~1 is similar to our result. Their bound is more general than ours, as
it applies to any $Q^{\pi}$, but it is also slightly looser.} 
Our strategy will be to use this result to bound the error introduced by the
application of the stochastic-factorization trick in the context of \ikbsf.

When $t_{m} = 1$, at any time step $t$ \ikbsf\ has a model $\bM_{t}$ built
based on the $t$ transitions observed thus far. As shown in the beginning of
this section, $\bM_{t}$ exactly matches the model that would be computed by
batch KBSF using the same data and the same set of representative states. Thus,
we can think of matrices \Pbat\ and vectors \rbat\
available at the \tth\ iteration of \ikbsf\ as the result of the
stochastic-factorization trick applied with matrices \Dt\ and \Kat.
Although \ikbsf\ does not explicitly compute such matrices, they serve as a 
solid theoretical ground to build our result~on. 

\begin{proposition}
\label{teo:inc_kbsf}
Suppose \ikbsf\ is executed with a fixed set of representative states \Sb\
using $t_{m} = 1$. Let \Dt, \Kat\ and \rbat\ be the matrices and the vector 
(implicitly) computed by this algorithm at iteration $t$. 
Then, if $s$ is the state encountered by \ikbsf\ at 
time step $t$,
 
{\footnotesize
\begin{equation*}
|\cQ_{t}(s, a) - \tilde{Q}_{t}(s, a)| \le 
\frac{1}{1 - \gamma} \maxinf{\rcat}{\Dt\rbat} + 
\frac{\bar{R}_{\dif,t}}{(1-\gamma)^{2}}
\left(
\frac{\gamma(2 - \gamma)}{2} \maxinf{\Pcat}{\Dt \Kat}
+
\levelDc
\right) + \epsilon_{\bar{Q}_t}, 
\end{equation*}
}

\noindent
for any $a \in A$,
where $\tilde{Q}_{t}$ is the value function computed by \ikbsf\
at time step $t$ through~(\ref{eq:ikbsf_q}),
$\cQ_{t}$ is the value function 
computed by KBRL through~(\ref{eq:kbrl_q}) based on
the same data,
$\bar{R}_{\dif,t} = \max_{a,i} \bar{r}^{a}_{i,t} - 
\min_{a,i} \bar{r}^{a}_{i,t}$,
$\levelDc = \max_{i}{(1 - \max_{j}{d_{ij,t})}}$,
and 
$\epsilon_{\bar{Q}_{t}}= \max_{i,a} |\bar{Q}^{*}_{t}(\rs_{i},a)
- \bar{Q}_{t}(\rs_{i},a)|$.
\end{proposition}

\begin{proof}
Let $\check{M}_{t} \equiv (\cS_{t}, A, \Pxat, \rxat, \gamma)$,
with $\Pxat = \Dt \Kat$ and 
$\rxat =\Dt \rbat$. From the triangle
inequality, we know that
\begin{equation}
\label{eq:triangle_q}
|\cQ_{t}(s, a) - \tilde{Q}_{t}(s, a)| \le
|\cQ_{t}(s, a) - \check{Q}^{*}_{t}(s, a)|
+
|\check{Q}^{*}_{t}(s, a) - \tilde{Q}^{*}_{t}(s, a)|
+
|\tilde{Q}^{*}_{t}(s, a) - \tilde{Q}_{t}(s, a)|,
\end{equation}
where 
$\cQ_{t}$ and $\tilde{Q}_{t}$ are defined in the proposition's statement,
$\check{Q}_{t}^{*}$ is the optimal action-value function of $\check{M}_{t}$,
and $\tilde{Q}^{*}_{t}(s,a) = \sum_{i=1}^{m} \kerb(s, \rs_{i})
\bar{Q}^{*}_{t} (\rs_{i}, a)$ (the reader will forgive a slight abuse of
notation here, since in general $\tilde{Q}^{*}_{t}$ is not the optimal
value function of any MDP).
Our strategy will be to bound each term on the right-hand side 
of~(\ref{eq:triangle_q}). 
Since $\cM_{t}$ is the model constructed by KBRL using all the data
seen by \ikbsf\ up to time step $t$, state $s$ will correspond to 
one of the states \yib\ in this MDP. Thus, 
from~(\ref{eq:kbrl_q}), we see that $\cQ_{t}(s, a) = \cQ_{t}^{*}(\yib, a)$ for
some $i$ and some $b$.
Therefore, applying Lemma~\ref{teo:bound_q}
to $\cM_{t}$ and $\xM_{t}$, we can write
\begin{equation}
\label{eq:term1}
|\cQ_{t}(s, a) - \check{Q}^{*}_{t}(s, a)|
\le
\frac{1}{1-\gamma} \maxinf{\rcat}{\Dt\rbat} + \frac{\gamma (2 -
\gamma)}{2(1-\gamma)^{2}} 
\bar{R}_{\dif,t} \maxinf{\Pcat}{\Dt\Kat}.
\end{equation}
In order to bound $|\check{Q}^{*}_{t}(s, a) - \tilde{Q}^{*}_{t}(s, a)|$,
we note that, since the information contained in the transition 
to state $s$ has been
incorporated to \ikbsf's model \bM\ 
at time $t$, 
$
\tilde{Q}^{*}_{t}(s,a) = 
\sum_{i=1}^{m} d_{ti,t} \bar{Q}^{*}_{t} (\rs_{i}, a),
$
for any $a \in A$, where $d_{ti,t}$ is the element in the \tth\ row and \ith\
column of \Dt\ (see Figure~\ref{fig:matrices_kbsf_sparse}).
In matrix form, we have $\matii{\tilde{Q}}{*}{t} = \Dt \matii{\bar{Q}}{*}{t}$.
As $\Dt$ is a soft homomorphism
between  $\xM_{t}$ and $\bM_{t}$, we can resort to 
\ctp{sorg2009transfer} Theorem~1,
as done in Proposition~\ref{teo:bound_sf}, to write:
\begin{equation}
\label{eq:term2}
|\check{Q}^{*}_{t}(s, a) - \tilde{Q}^{*}_{t}(s, a)|
\le 
\frac{ \bar{R}_{\dif,t}}{(1-\gamma)^{2}} \levelDc
\end{equation}
(see~(\ref{eq:sing_tighter}) and~(\ref{eq:sing})).
Finally,
\begin{align}
\nonumber
|\tilde{Q}^{*}_{t}(s, a) - \tilde{Q}_{t}(s, a)| 
& = 
\left|\sum_{i=1}^{m}  \kerb(s, \rs_{i}) \bar{Q}^{*}_{t}(\rs_{i}, a)
- \sum_{i=1}^{m}  \kerb(s, \rs_{i}) \bar{Q}_{t}(\rs_{i}, a)\right|
\\
& 
\label{eq:term3}
\le  
\sum_{i=1}^{m}  \kerb(s, \rs_{i}) \left|\bar{Q}^{*}_{t}(\rs_{i}, a)
- \bar{Q}_{t}(\rs_{i}, a) \right|
\le \epsilon_{\bar{Q}_{t}},
\end{align}
where the last step follows from the fact that 
$\sum_{i=1}^{m}  \kerb(s,\rs_{i})$ is a convex combination.
Substituting~(\ref{eq:term1}),~(\ref{eq:term2}), and~(\ref{eq:term3})
in~(\ref{eq:triangle_q}), we obtain the desired bound.
\end{proof}

Proposition~\ref{teo:inc_kbsf} shows that, at any time step $t$, the 
error in the action-value function computed by \ikbsf\ is bounded above by the 
quality and the level of stochasticity of the stochastic 
factorization implicitly computed by the algorithm.
The term $\epsilon_{\bar{Q}_{t}}$ accounts for the possibility that \ikbsf\
has not computed the optimal value function of its model 
at step $t$, either because $t_m \ne t_v$ or because the update of 
\Qb\ in Algorithm~\ref{alg:inc_kbsf} is not done to completion 
(for example, one can apply the Bellman operator $\bT$ a fixed number 
of times, stopping short of convergence).
We note that the restriction $t_{m} = 1$ is not strictly necessary if we
are willing to compare $\tilde{Q}_{t}(s, a)$ with $\cQ_{z}(s, a)$,
where $z = \lfloor (t + t_m) / t \rfloor$ (the next time step scheduled for a 
model update). However, such a result would be somewhat circular, since 
the sample transitions used to build $\cQ_{z}(s, a)$ may depend on
$\tilde{Q}_{t}(s, a)$.

\subsection{Empirical results}
\label{sec:empirical_inc}

We now look at the empirical performance of 
the incremental version of KBSF.
Following the structure of Section~\ref{sec:empirical_batch}, 
we start with the puddle world task to show that \ikbsf\ is
indeed able to match the performance of batch KBSF without storing all 
sample transitions. Next we exploit the scalability of \ikbsf\
to solve two difficult control tasks, triple pole-balancing and 
helicopter hovering. We also compare \ikbsf's performance with
that of other reinforcement learning algorithms.

\subsubsection{Puddle world (proof of concept)}
\label{sec:puddle_inc}

We use the puddle world problem as a proof of
concept~\cp{sutton96generalization}. In this first experiment we show
that \ikbsf\ is able to recover the model that would be
computed by its batch counterpart. In order to do so, we applied
Algorithm~\ref{alg:inc_kbsf} to the puddle-world task
using a random policy to select actions.

Figure~\ref{fig:puddle_n_inc} shows the result of the experiment when we
vary the parameters $t_{m}$ and $t_{v}$.
Note that the case in which $t_{m} = t_{v} = 8\ts000$ corresponds to
the batch version of KBSF, whose results on the puddle world are shown in
Figure~\ref{fig:puddle}. As expected, the performance of KBSF 
policies improves gradually as the algorithm goes through more sample
transitions,
and in general the intensity of the improvement is proportional to the amount of
data processed. More important, the performance of the decision policies after
all sample transitions have been processed is essentially the same
for all values of $t_{m}$ and $t_{v}$, which
confirms that \ikbsf\ can be used as an instrument to circumvent KBSF's 
memory demand. Thus, if one has a batch of sample
transitions that does not fit in the available memory, it is possible to 
split the data in chunks of smaller sizes and still get the same 
value-function approximation
that would be computed if the entire data set were processed at once. 
As shown in Figure~\ref{fig:puddle_n_inc_time}, there is only a small 
computational overhead associated with such a strategy (this results from
unnormalizing and normalizing the elements of \Pba\ and \rba\ several times
through update rules~(\ref{eq:upd_p}) and~(\ref{eq:upd_r})).

\begin{figure*}
\centering
   \subfloat[Performance]{ 
   \label{fig:puddle_n_inc}
   \includegraphics[scale=\scll]{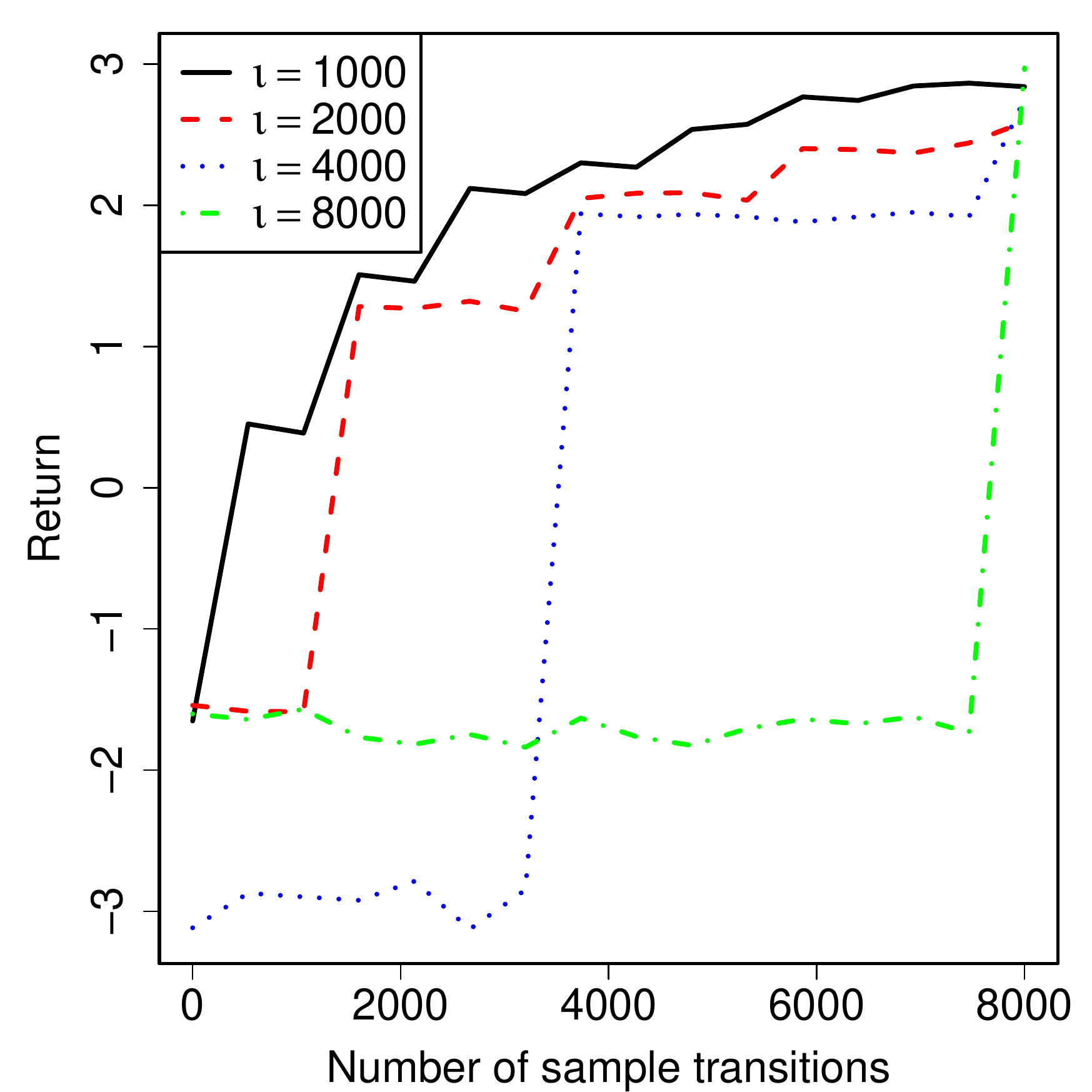} 
   }
    \subfloat[Run times] { 
   \label{fig:puddle_n_inc_time}
   \includegraphics[scale=\scll]{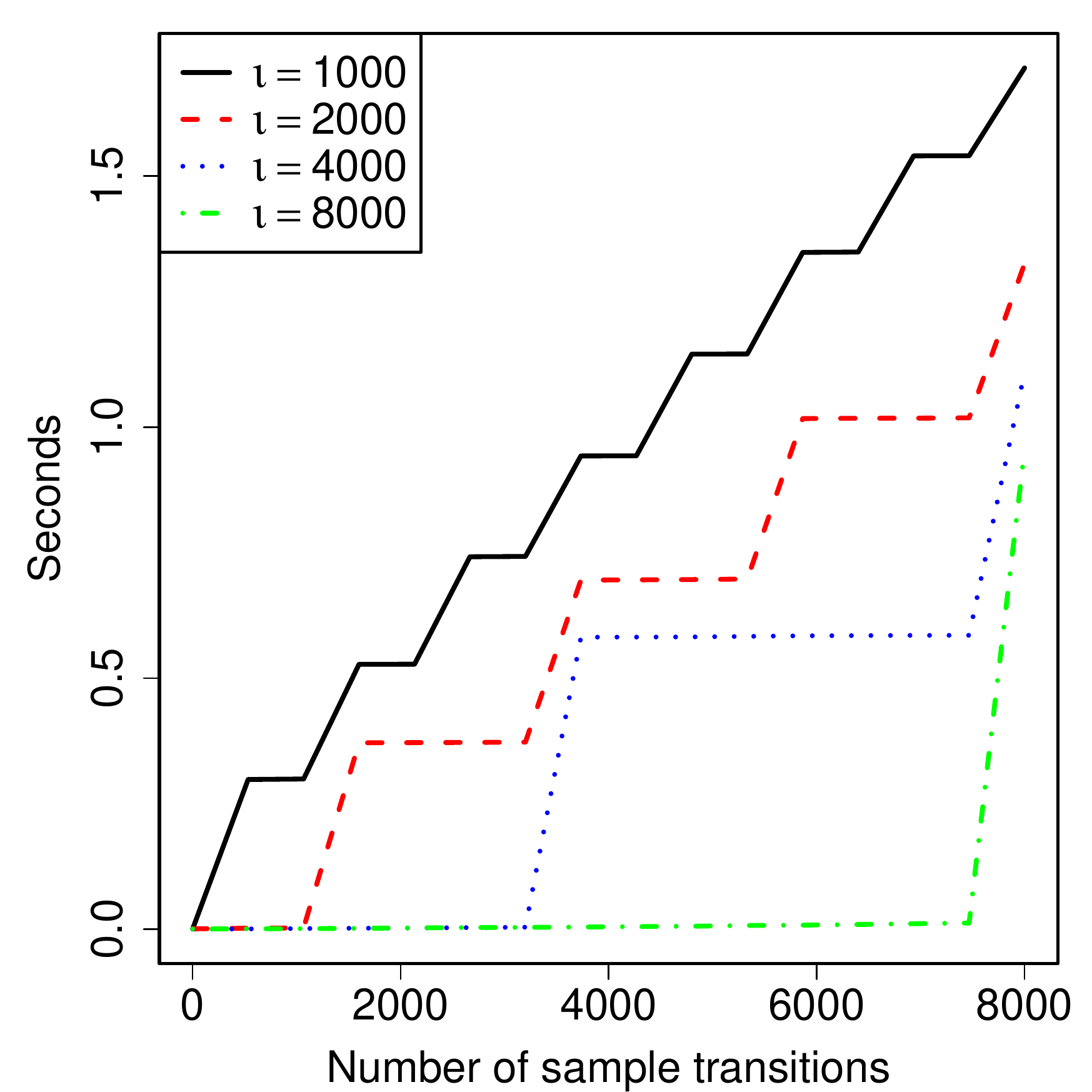}
    }
\caption{Results on the puddle-world task averaged over $50$ runs. 
KBSF used $100$ representative states evenly distributed
over the state space and $t_{m} = t_{v} = \iota$ (see legends).
Sample transitions were collected by a random policy.
The agents were tested on two sets of
states surrounding the ``puddles'' (see Appendix~\ref{seca:exp_details}).
\label{fig:puddle_inc}}
\end{figure*}

\subsubsection{Triple pole-balancing (comparison with fitted $Q$-iteration)}
\label{sec:triple_pole}

As discussed in Section~\ref{sec:pole}, the pole balancing task has been
addressed in several different versions,
and among them simultaneously balancing two poles is particularly 
challenging~\cp{wieland91evolving}. Figures~\ref{fig:double_pole_m}
and~\ref{fig:double_pole_m_time} show that the batch version of KBSF was able to
satisfactorily solve the double pole-balancing task.
In order to show the scalability of the incremental version of our algorithm, 
in this section we raise the bar, adding a third pole to the problem.
We perform our simulations using the parameters usually adopted with the
two-pole problem, with the extra pole having the same length and
mass as the longer pole~\cp[see
Appendix~\ref{seca:exp_details}]{gomez2003robust}. 
This results in a difficult control problem with an $8$-dimensional 
state space \Sc.

In our experiments with KBSF on the two-pole task, we
used $200$ representative states and $10^{6}$ sample transitions collected by a
random policy. Here we start our experiment with
triple pole-balancing using exactly the same configuration, and then we 
let \ikbsf\ refine its 
model \bM\ by incorporating more sample transitions through update
rules~(\ref{eq:upd_p}) and~(\ref{eq:upd_r}).
We also let \ikbsf\ grow its model if necessary. Specifically, 
a new representative state is added to \bM\ on-line every time the agent
encounters a sample state $\yia$ for which $\gaussb(\yia, \rs_{j}) < 0.01$ for
all $j \in 1, 2, ..., m$. This corresponds to setting the maximum allowed
distance from a sampled state to the closest representative state,
$\max_{a,i}\ddb(\yia, 1)$. 

Given the poor performance of LSPI on the double pole-balancing task, shown in
Figures~\ref{fig:double_pole_m} and~\ref{fig:double_pole_m_time}, on the
three-pole version of the problem we only compare KBSF with FQIT.
We used FQIT with the same configuration adopted in 
Sections~\ref{sec:hiv} and~\ref{sec:epilepsy}, with the parameter \mne\ 
varying in the set $\{10\ts000, 1\ts000, 100\}$.
As for KBSF, the widths of the kernels were fixed at $\tau = 100$ and $\taub =
1$ and sparse kernels were used ($\nk=50$ and $\nkb = 10$).

In order to show the benefits provided by the incremental version of our
algorithm, we assumed that both KBSF and FQIT could store at most
$10^{6}$ sample transitions in memory. In the case of \ikbsf, this is not a
problem, since we can always split the data in subsets of smaller size and
process them incrementally. Here, we used Algorithm~\ref{alg:inc_kbsf}
with a $0.3$-greedy policy, $t_{m} = t_{v} = 10^{6}$, and 
$n = 10^{7}$.
In the case of FQIT, we have two options to circumvent the limited 
amount of memory available. The first one is to use a single batch of
$10^{6}$ sample transitions. The other option is to use the initial 
batch of transitions to compute an approximation of the problem's value
function, then use an $0.3$-greedy policy induced by this approximation to
collect a second batch, and so on. Here we show the performance of FQIT using
both strategies.

We first compare the performance of \ikbsf\ with that of FQIT using a single
batch of sample transitions. This is shown in Figure~\ref{fig:tp_ep_batch} 
and~\ref{fig:tp_tim_batch}. For reference, we also show the results of batch
KBSF---that is, we show the performance of the policy that would be computed by
our algorithm if we did not have a way of computing its approximation
incrementally. As shown in Figure~\ref{fig:tp_ep_batch}, both FQIT and batch
KBSF perform poorly in the triple pole-balancing task, with average success
rates below $55\%$. 
These results suggest that the
amount of data used by these algorithms is insufficient to describe the
dynamics of the control task. 
Of course, we could give more sample transitions to FQIT and
batch KBSF. Note however that, since they are batch-learning  methods, there is
an inherent limit on the amount of data that these algorithms can use to
construct their approximation. 
In contrast, the amount of memory required by \ikbsf\ is independent of
the number of sample transitions $n$. This fact together with the fact that
KBSF's computational complexity is only linear in $n$ allow our algorithm
to process a large amount of data in reasonable time. This can be
clearly observed in Figure~\ref{fig:tp_tim_batch}, which shows that \ikbsf\
can build an approximation using $10^{7}$ sample transitions in under $20$
minutes. As a reference for comparison, FQIT($1000$) took an average of
$1$ hour and $18$ minutes to process $10$ times less data.

\begin{figure*}
\centering
   \subfloat[Performance]{ 
   \label{fig:tp_ep_batch}
   \includegraphics[scale=\scll]{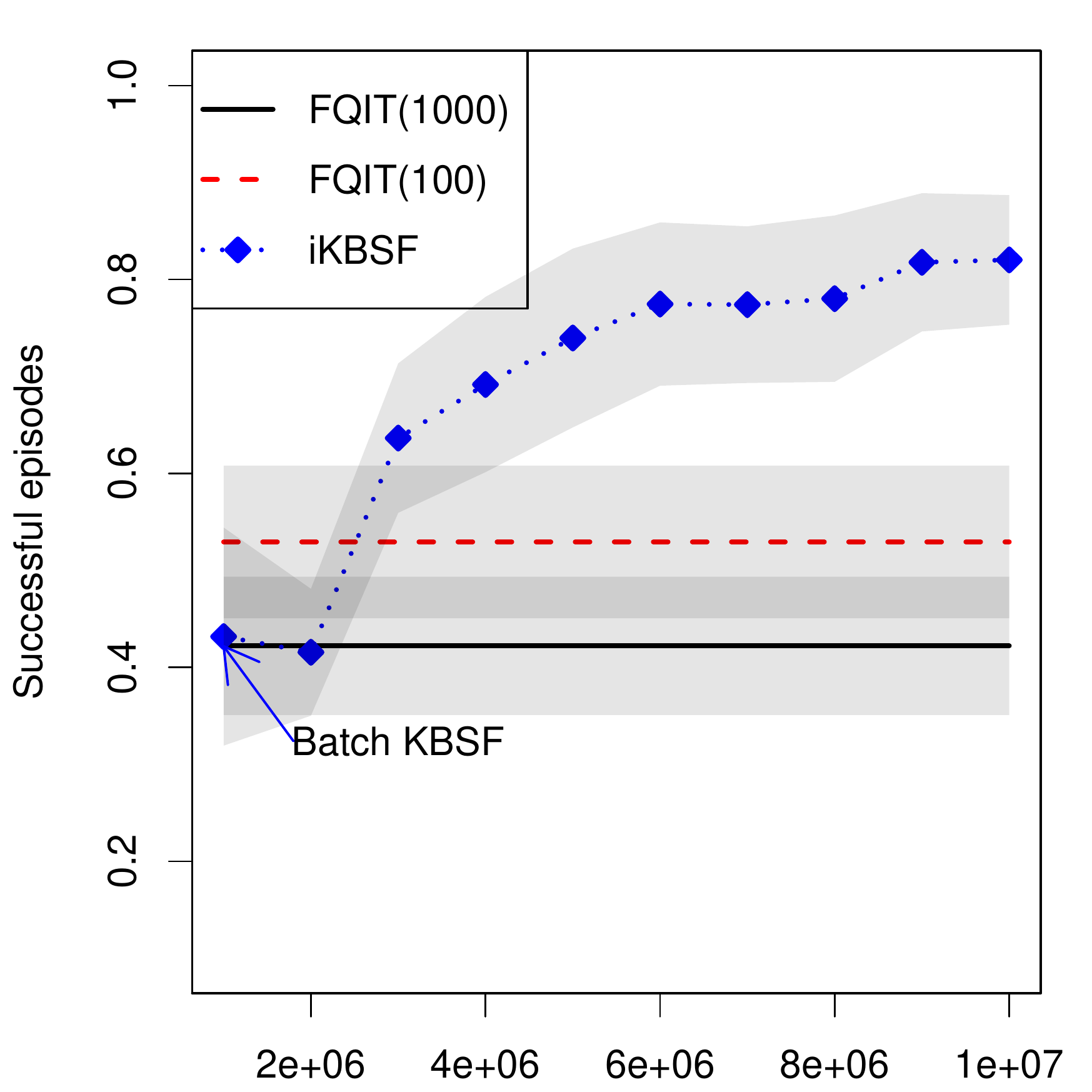} 
   }
    \subfloat[Run times] { 
   \label{fig:tp_tim_batch}
   \includegraphics[scale=\scll]{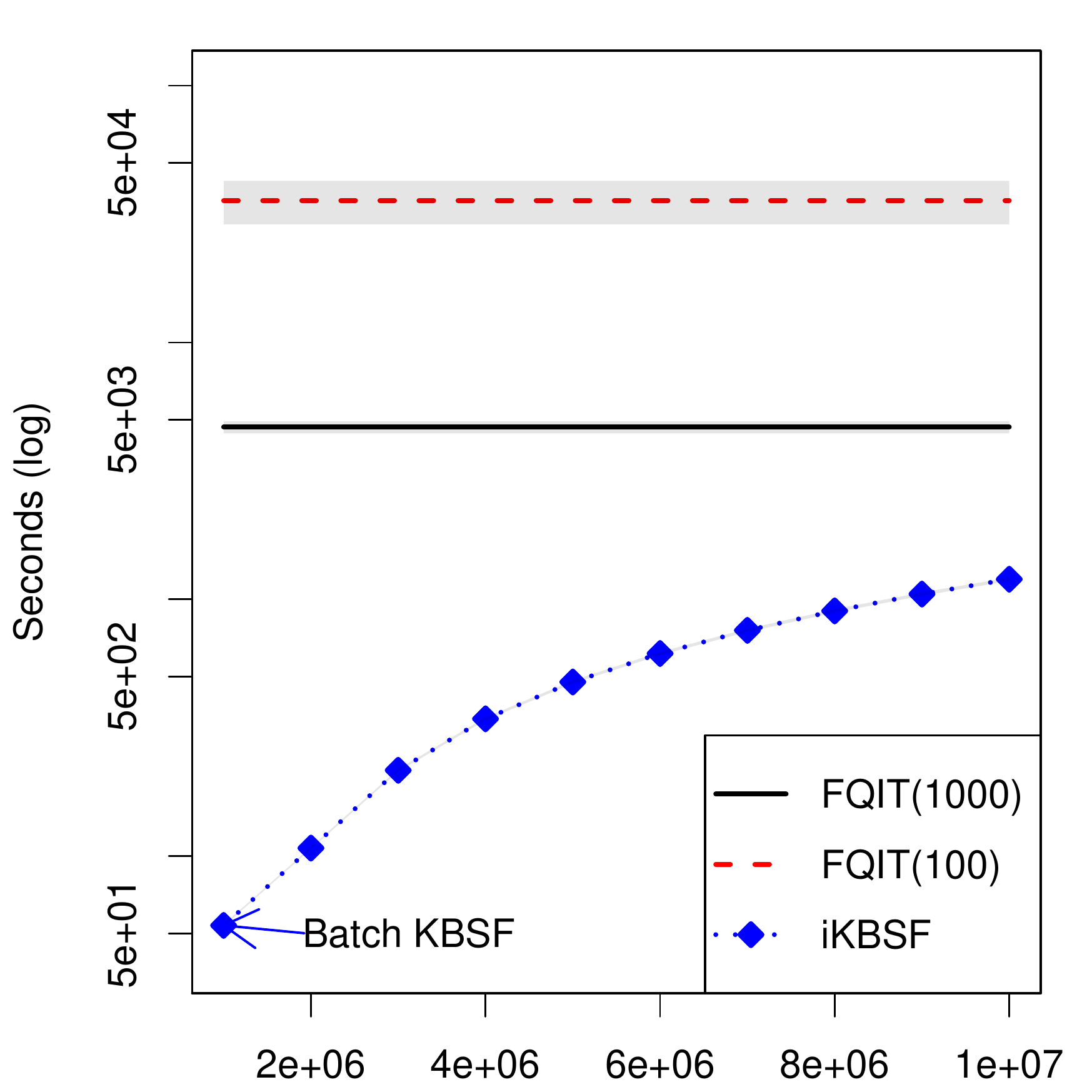}
    }

   \subfloat[Performance]{ 
   \label{fig:tp_ep_inc}
   \includegraphics[scale=\scll]{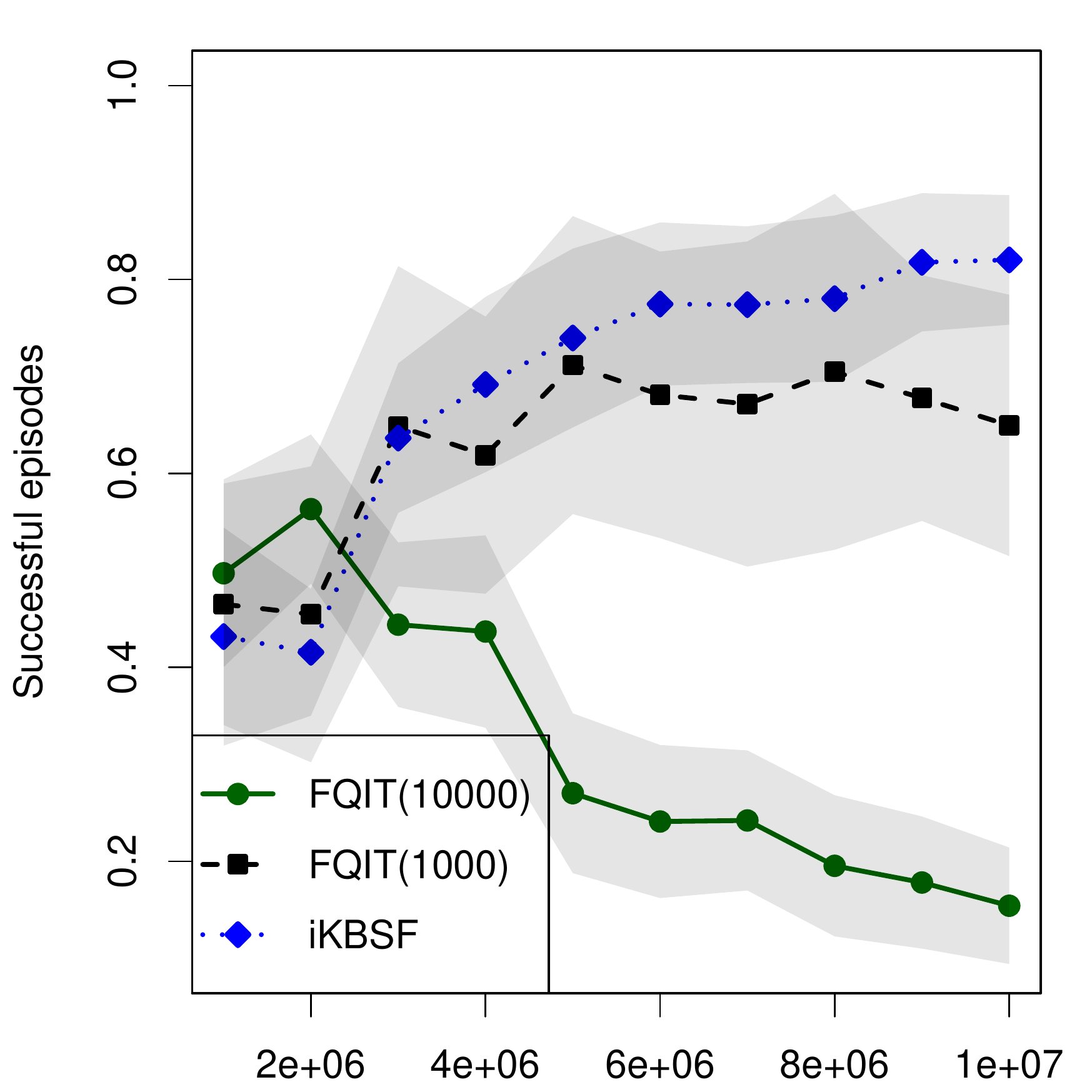} 
   }
    \subfloat[Run times] { 
   \label{fig:tp_tim_inc}
   \includegraphics[scale=\scll]{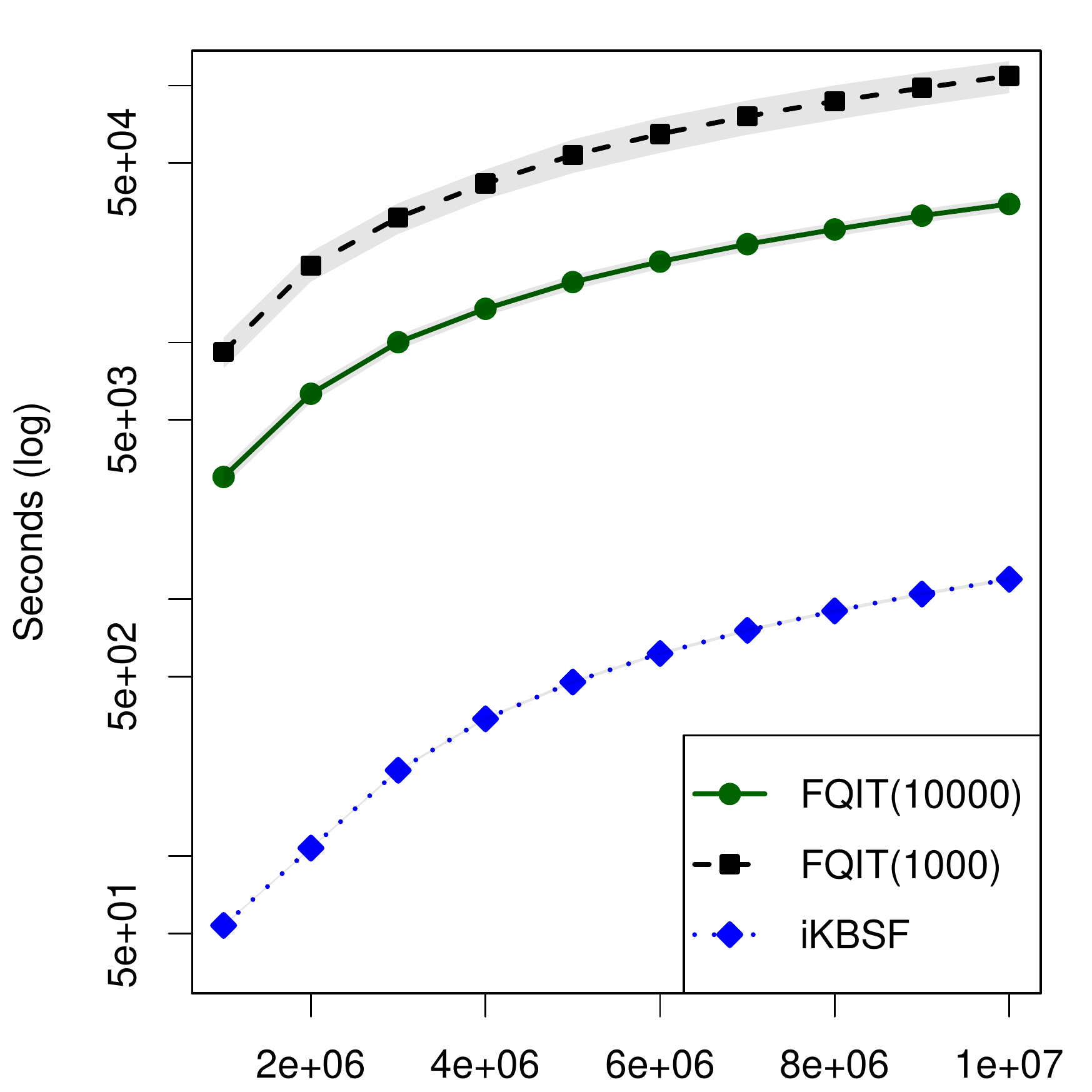}
    }
%
\caption{Results on the triple pole-balancing task,
as a function of the number of sample transitions $n$,
averaged over $50$ runs. 
The values correspond to the fraction of episodes initiated from the
test states in which the $3$ poles could be balanced for $3\ts000$ steps (one
minute of simulated time). The test sets were regular grids 
of $256$ cells defined over the hypercube centered at the origin and covering
$50\%$ of the state-space axes in each dimension (see
Appendix~\ref{seca:exp_details} for details). Shadowed regions represent
$99\%$ confidence intervals.
\label{fig:tp}}
\end{figure*}

As shown in Figure~\ref{fig:tp_ep_batch}, \ikbsf's ability to process
a large number of sample transitions allows our algorithm 
to achieve a success rate of approximately $80\%$.
This is similar to the performance of batch KBSF on the two-pole version of
the problem ({\sl cf.} Figure~\ref{fig:pole}). 
The good performance of \ikbsf\ on the triple pole-balancing task is
especially impressive when we recall that the decision policies 
were evaluated on a set of test states representing all possible
directions of inclination of the
three poles.
In order to achieve the same level of performance
with KBSF, approximately $2$ Gb of memory would be necessary, even using sparse
kernels, whereas \ikbsf\ used less than $0.03$ Gb of memory.

One may argue that the comparison between FQIT and KBSF is not fair, since the
latter used ten times the amount of data used by the former.
Thus, in Figures~\ref{fig:tp_ep_inc} and ~\ref{fig:tp_tim_inc} we show the
results of FQIT using $10$ batches of $10^{6}$ transitions---exactly the same
number of transitions processed by
\ikbsf. Here we cannot compare \ikbsf\ with FQIT($100$) because the
computational cost of the tree-based approach is prohibitively large (it would
take over $4$ days only to train a single agent, not counting the test phase). 
When we look at the other instances of the algorithm, we see two opposite
trends. Surprisingly, the extra sample transitions actually made the performance
of FQIT($10\ts000$) \emph{worse}. On the other hand,  
FQIT($1000$) performs significantly better using more data, though 
still not as well as \ikbsf\ (both in terms of performance and computing time). 

To conclude, observe in Figure~\ref{fig:tp_rs} how the number of representative
states $m$ grows as a function of  the number of sample transitions
processed by KBSF. As expected, in the beginning of the learning process $m$
grows fast, reflecting the fact that some relevant regions of the state space
have not been visited yet. As more and more data come in, the number of
representative states starts to stabilize. 

\begin{figure*}
\centering
\includegraphics[scale=0.35]{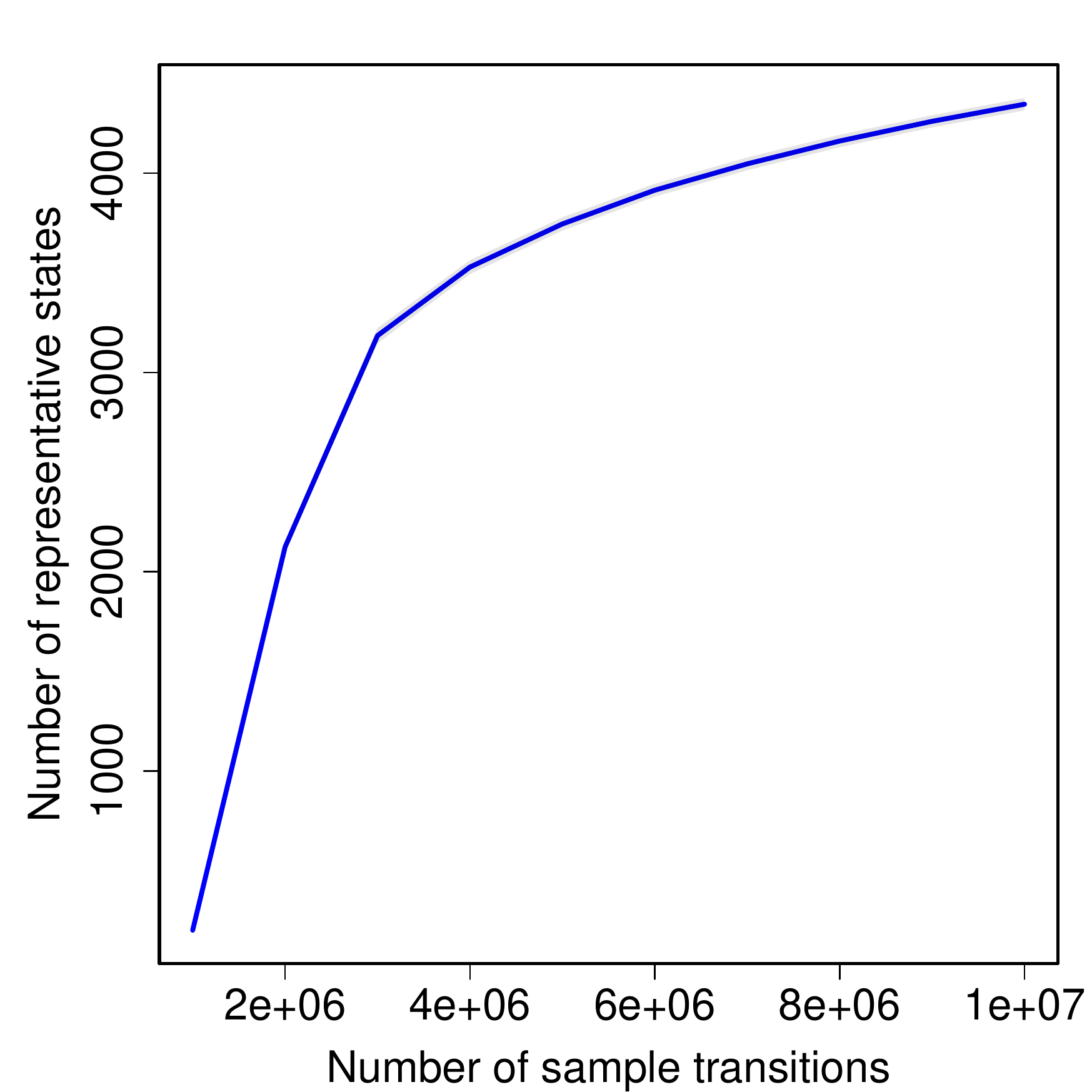}
\caption{Number of representative states used by \ikbsf\ on the triple 
pole-balancing task.
Results were averaged over $50$ runs ($99\%$ confidence intervals are almost
imperceptible in the figure).
\label{fig:tp_rs}}
\end{figure*}

\subsubsection{Helicopter hovering (comparison with SARSA)}
\label{sec:helicopter}

In the previous two sections we showed how \ikbsf\ can be used
to circumvent the inherent memory limitations of batch learning. 
We now show how our algorithm performs in a fully on-line regime.
For that, we focus on a challenging reinforcement learning task in which the
goal is to control an autonomous helicopter.

Helicopters have unique control capabilities, such as low speed
flight and in-place hovering, that make them indispensable instruments in many
contexts. Such flexibility comes at a price, though: it is 
widely recognized that a helicopter is significantly harder to control than a
fixed-wing aircraft~\cp{ng2003autonomous,abbeel2007application}.
Part of this difficulty is due to the complex dynamics of the helicopter,
which is not only non-linear, noisy, and asymmetric, but also 
counterintuitive in some aspects~\cp{ng2003autonomous}.

An additional complication of controlling an autonomous helicopter is
the fact that a wrong action can easily lead to a crash, which is both dangerous
and expensive. Thus, the usual
practice is to first develop a model of the helicopter's dynamics and then 
use the model to design a controller~\cp{ng2003autonomous}. Here
we use the model constructed by~\ct{abbeel2005learning} based on data collected
on actual flights of an XCell Tempest helicopter (see
Appendix~\ref{seca:exp_details}).
The resulting reinforcement learning problem has a $12$-dimensional state space
whose variables represent the aircraft's position, orientation, and
the corresponding velocities and angular velocities along each axis.

In the version of the task considered here the goal is to keep the helicopter
hovering as close as possible to a fixed position. All episodes start 
at the target location, and at each time step the agent receives a negative
reward proportional to the distance from the current state to the desired
position. Because the tail rotor's thrust exerts a sideways force on the
helicopter, the aircraft cannot be held stationary in the zero-cost state even
in
the absence of wind. The episode ends when the helicopter leaves the hover
regime, that is, when any of the state's variables exceeds pre-specified 
thresholds.

The helicopter is controlled via a $4$-dimensional continuous vector whose
variables represent the longitudinal cyclic pitch, the latitudinal cyclic
pitch, the tail rotor collective pitch, and the main rotor collective pitch.
By adjusting the value of these variables the pilot can rotate the helicopter
around its axes and control the thrust generated by the main rotor.
Since KBSF was designed to deal with a finite number of actions, we
discretized the set $A$ using $4$ values per dimension, resulting in
$256$ possible actions. The details of the discretization process are given
below.

Here we compare \ikbsf\ with the SARSA($\lambda$) algorithm
using tile coding for value function approximation
(\cwp{rummery94on-line}, \cwp{sutton96generalization}---see
Appendix~\ref{seca:exp_details}).
We applied SARSA with $\lambda = 0.05$, a learning rate of
$0.001$, and $24$ tilings containing $4^{12}$ tiles each. Except for $\lambda$,
all the parameters were adjusted 
in a set of preliminary experiments
in order to improve the 
performance of the SARSA agent.  
We also defined the action-space discretization 
based on SARSA's performance. In particular, instead of
partitioning each dimension in equally-sized intervals, we spread the break
points unevenly along each axis in order to maximize the agent's return. The
result of this process is described in Appendix~\ref{seca:exp_details}.
The interaction of the SARSA agent with the helicopter hovering task 
was dictated by an $\epsilon$-greedy policy. Initially we set $\epsilon=1$, and
at every $50\ts000$ transitions the value of $\epsilon$
was decreased in $30\%$.

The \ikbsf\ agent collected sample transitions using the same exploration
regime. Based on the first batch of $50\ts000$ transitions, $m=500$
representative
states were determined by the $k$-means algorithm. No representative states were
added to \ikbsf's model after that. Both the value function and the model
were updated at fixed intervals of $t_{v}= t_{m} = 50\ts000$ transitions.
We fixed $\tau = \taub = 1$ and $\nk = \nkb =4$.

Figure~\ref{fig:helicopter} shows the results obtained by SARSA and KBSF on the
helicopter hovering task. Note in Figure~\ref{fig:helicopter_steps} how the
average
episode length increases abruptly at
the points in which the value of $\epsilon$ is decreased. 
This is true for both SARSA and KBSF. Also,
since the number of steps executed per episode increases over
time, the interval in between such abrupt changes decreases in length, as
expected. Finally, observe how the performance of both
agents stabilizes after around  $70\ts000$ episodes, probably because at this
point there is almost no exploration taking place anymore.

\begin{figure*}
\centering
   \subfloat[Performance]{ 
   \label{fig:helicopter_steps}
   \includegraphics[scale=\scll]{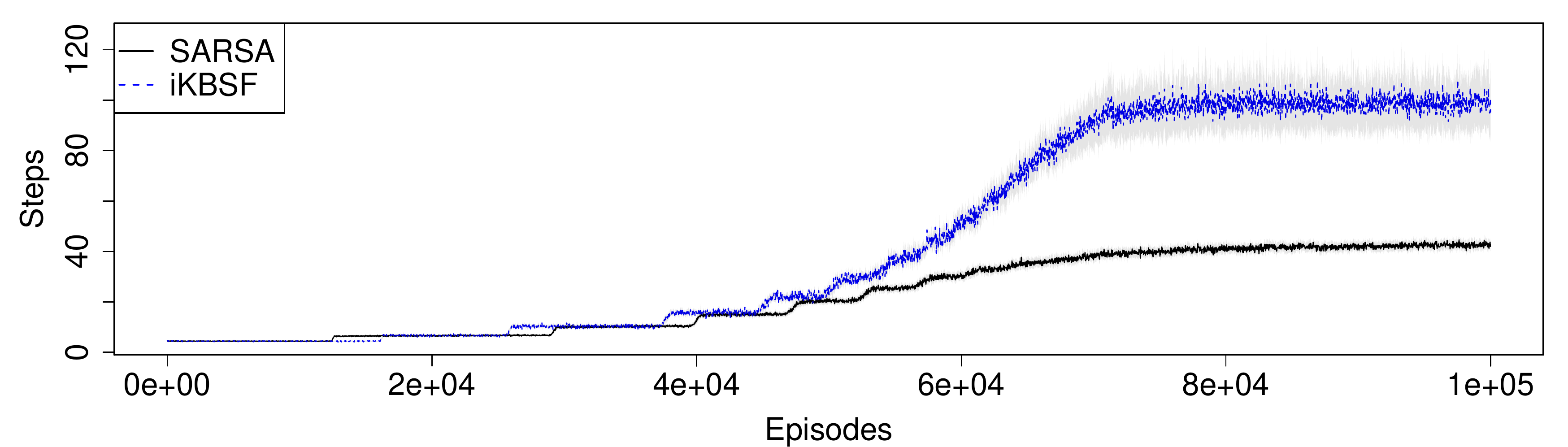} 
   }

    \subfloat[Run time] { 
   \label{fig:helicoper_time}
   \includegraphics[scale=\scll]{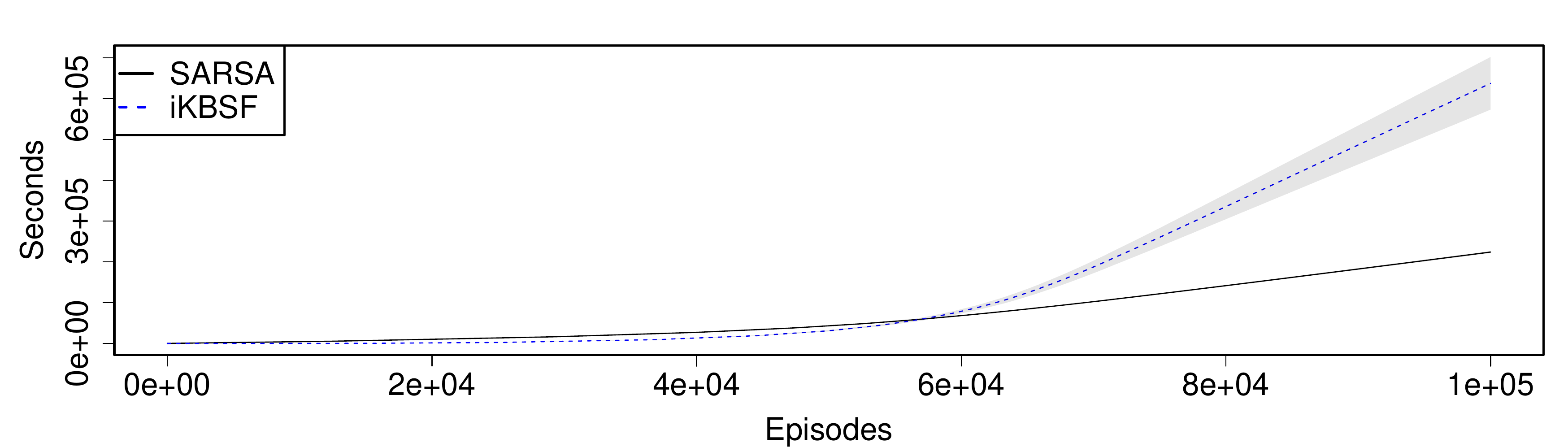}
    }

    \subfloat[Average time per step (time of an episode divided
by the number of steps)]  { 
   \label{fig:helicoper_time_step}
   \includegraphics[scale=\scll]{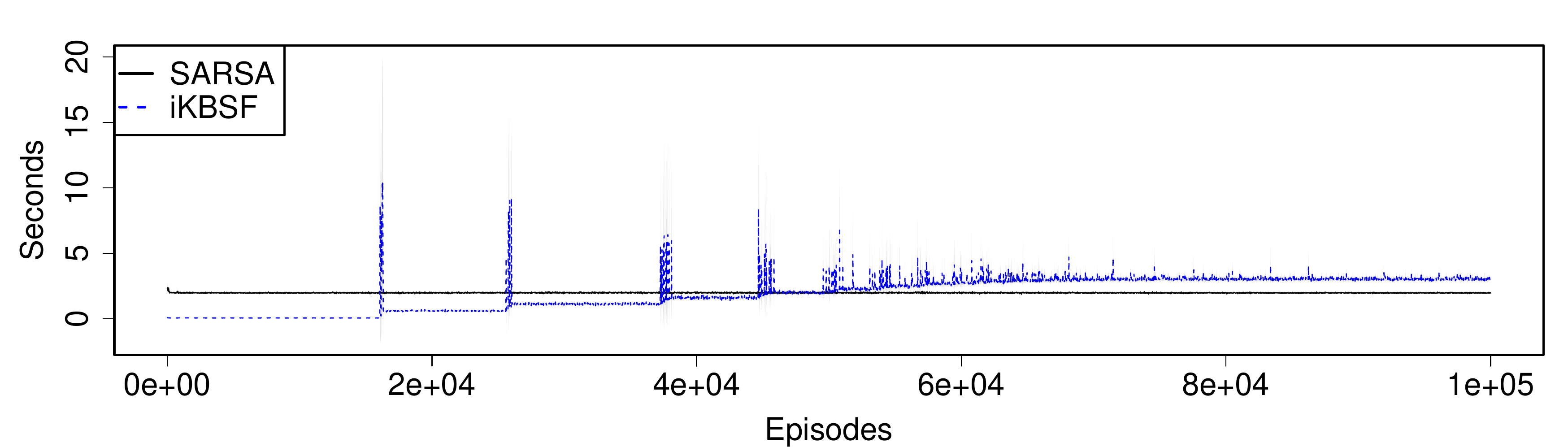}
    }

\caption{Results on the helicopter hovering task averaged over $50$ runs.
The learned controllers were tested from a fixed state (see text for details). 
The shadowed regions represent $99\%$ confidence
intervals. 
\label{fig:helicopter}}
\end{figure*}

When we compare KBSF and SARSA, it is clear
that the former significantly outperforms the latter.
Specifically, after the cut-point of $70\ts000$ episodes, the 
KBSF agent executes approximately $2.25$ times the number of 
steps performed by the SARSA agent before crashing.
Looking at Figures~\ref{fig:helicopter_steps} and~\ref{fig:helicoper_time},
one may argue at first that there is nothing surprising here: being a
model-based algorithm, KBSF is more sample efficient than SARSA, but 
it is also considerably slower~\cp{atkeson97comparison}. 
Notice though that the difference between the run times of SARSA and KBSF
shown in Figure~\ref{fig:helicoper_time} is in part a consequence of the
good performance of the latter: since KBSF is able to control the
helicopter for a larger number of steps, the corresponding episodes will
obviously take longer. A better measure of the algorithms's computational cost
can be seen
in Figure~\ref{fig:helicoper_time_step}, which shows the average time 
taken by each method to perform one transition. 
Observe how KBSF's computing time peaks at the points in which the model and
the value function are updated. In the beginning KBSF's MDP changes
considerably, and as a result the value function updates take longer. As more
and more data come in, the model starts to stabilize, accelerating the
computation of $\Qbo$ (we ``warm start'' policy iteration with the
value function computed in the previous round). At this point, KBSF's
computational cost per step is only slightly higher than SARSA's, even though
the former computes a
model of the environment while the latter directly updates the value function
approximation.

To conclude, we note that our objective in this section was 
exclusively to show that KBSF
can outperform a well-known on-line algorithm with compatible computational
cost. 
Therefore, we focused on the comparison of the algorithms 
rather than on obtaining the best possible performance on the task. 
Also, it is important to mention that more difficult versions of the
helicopter task have been addressed in the literature, usually 
using domain knowledge in the configuration of the algorithms
or to guide the collection of data~\cp{ng2003autonomous,abbeel2007application}. 
Since our focus here was on evaluating the on-line performance of KBSF,
we addressed the problem in its purest form, without
using any prior information to help the algorithms solve the task.

\section{Discussion}
\label{sec:discussion}

During the execution of our experiments we observed several
interesting facts about KBSF which are not immediate from its conceptual
definition. In this section we share some of the lessons learned
with the reader. We start by
discussing the impact of deviating from the
theoretical assumptions over the performance of our algorithm. We then 
present general guidelines on how to configure KBSF to solve reinforcement
learning problems.

\subsection{KBSF's applicability}

The theoretical guarantees regarding KBRL's solution assume that the 
initial states \xia\ in the transitions $(\xia, \ria, \yia)$ are
uniformly sampled from \Sc~\cp[see Assumption~3]{ormoneit2002kernelbased}. This
is somewhat restrictive because it precludes the collection of data through
direct interaction with the environment. 
\ca{ormoneit2002kernelbased} conjectured that sampling the states \xia\ from an
uniform distribution is not strictly necessary, and indeed
later \ct{ormoneit2002kernelbased2} relaxed this assumption for the 
case in which KBRL is applied to an average-reward MDP. In this case, it is only
required that the exploration policy used to collect data chooses all actions
with positive probability.
As described in Sections~\ref{sec:empirical_batch} and~\ref{sec:empirical_inc},
in our computational experiments we collected data through an $\epsilon$-greedy
policy (in many cases with $\epsilon=1$). 
The good performance of KBSF corroborates \ca{ormoneit2002kernelbased}'s
conjecture and suggests that \ca{ormoneit2002kernelbased2}'s
results can be generalized to the discounted reward case, but more theoretical
analysis is needed.

\ct{ormoneit2002kernelbased} also make some assumptions regarding the 
smoothness of the reward function and the transition kernel of the continuous 
MDP (Assumptions~1 and~2). Unfortunately, such assumptions are usually not
verifiable in practice. Empirically, we observed that KBSF 
indeed performs better in
problems with ``smooth dynamics''---loosely speaking, problems in which a small
perturbation in \xia\ results in a small perturbation in \yia,  such as
the pole balancing task. In problems with ``rougher'' dynamics, like
the epilepsy-suppression task, it is still possible to
get good results with KBSF, but in this case it is necessary to use more
representative states and narrower kernels (that is, smaller values for
$\taub$). As a result, in problems of this type KBSF is less
effective in reducing KBRL's computational cost.

\subsection{KBSF's configuration}

The performance of KBSF depends crucially on the definition of the
representative
states $\rs_{j}$. Looking at expression~(\ref{eq:kfuncd}), 
we see that ideally these states
would be such that the rows of the matrices \Ka\ would form a convex hull
containing the rows of the corresponding \Pca. However, it is easy to see
that when $m < n$ such a set of states may not
exist. Even when it does exist, finding this set is not a trivial
problem.

Instead of insisting on finding representative states that allow for an
exact representation of the matrices \Pca, it sounds more realistic to content
oneself with an approximate solution for this problem. 
Proposition~\ref{teo:bound_rep_states} suggests that a reasonable strategy 
to define the representative states is to control the
magnitude of $\max_{a,i}\ddb(\yia, 1)$, the maximum distance  from  a sampled
state \yia\ to the nearest representative state. 
Based on this observation, in our experiments we clustered the states
$\yia$ and used the clusters's centers as our representative states. Despite its
simplicity, this strategy usually results in good performance, as shown in
Sections~\ref{sec:empirical_batch} and~\ref{sec:empirical_inc}. 

Of course, other approaches are possible. The simplest technique is perhaps to
select representative states at random from the set of sampled states \yia.
As shown in Section~\ref{sec:hiv}, this strategy seems to work reasonably well
when adopted together with model averaging. Another alternative is to resort to 
quantization approaches other than $k$-means~\cp{kaufman90finding}.
Among them, a promising method is \ctp{beygelzimer06cover} cover tree, since it
directly tries to minimize $\max_{a,i}\ddb(\yia, 1)$ and can be 
easily updated on-line~(the idea of using 
cover trees for kernel-based reinforcement learning was first proposed by
\cwp{kveton2012kernel}). 
Yet another possibility is to fit a mixture of Gaussians to the sampled states 
\yia~\cp[Chapter~6]{hastie2002elements}. 

The definition of the representative states can also be seen as an opportunity
to incorporate prior knowledge about the domain of interest into the
approximation model. For example, if one knows that some regions of the state
space are more important than others, this information can be used to allocate
more representative states to those regions. Similar reasoning applies to tasks
in which the level of accuracy required from the decision policy varies across
the state space.
Regardless of how exactly the representative states are defined, by using
\ikbsf\ one can always add new ones on-line if necessary
(see Section~\ref{sec:triple_pole}).

Given a well-defined strategy to select representative states, 
the use of KBSF requires the definition of three parameters: the number of
representative states, $m$, and the widths of the kernels used by the algorithm,
$\tau$ and $\taub$. Both theory and practice indicate that KBSF's
performance generally improves when $m$ is increased. Thus, a 
``rule of thumb'' to define the number of representative states is to
simply set $m$ to the
largest value allowed by the available computational resources. This reduces
KBSF's configuration to the definition of the kernels's widths.

The parameters $\tau$ and $\taub$ may have a strong effect on KBSF's 
performance. To illustrate this point, we show in Figure~\ref{fig:puddle_all}
the results of this algorithm on the puddle world task when $\tau$ and
$\taub$ are varied in the set $\{0.01, 0.1, 1\}$ (these were the results used
to generate Figure~\ref{fig:puddle}). 
Of course, the best combination of values for $\tau$ and $\taub$ depends on the
specific problem at hand and on the particular choice of kernels.
Here we give some general advice as to how to set these parameters, based on
both theory in practice. Since $\tau$ is the same parameter used by KBRL, it
should decrease with the number of sample transitions 
$n$ at an ``admissible rate'' (see~\ca{ormoneit2002kernelbased}'s Lemma~2,
\cy{ormoneit2002kernelbased}). Analogously,
Proposition~\ref{teo:bound_rep_states} suggests that $\taub$ should get smaller
as $m \rightarrow n$. 
Empirically, we found out that a simple strategy that usually
facilitates the configuration of KBSF is to rescale the data so
that all the variables have approximately the same magnitude---which corresponds
to using a weighted norm in the computation of the kernels.
Using this strategy we were able to obtain good results with KBSF on all
problems by performing a coarse search in the space of parameters in which we
only varied the order of magnitude of $\tau$ and 
$\taub$ (see Table~\ref{tab:kbsf_params} on
page~\pageref{tab:kbsf_params}).
 
\newcolumntype{S}{>{\centering\arraybackslash} b{.45\linewidth} }
\begin{figure*}
\centering
\begin{tabular}{SS}
\includegraphics[scale=\scll]{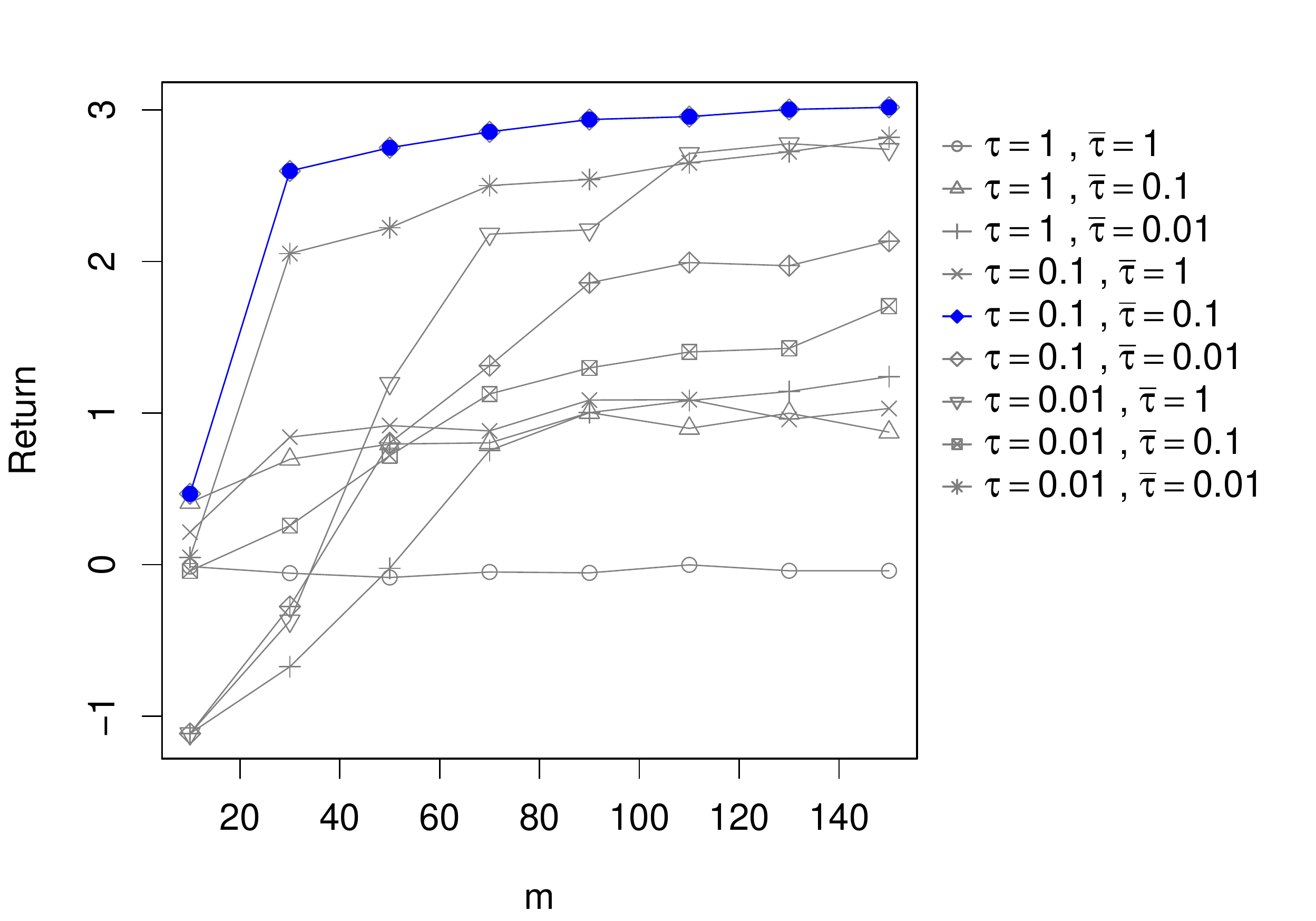} &
\begin{tabular}{cr} 
 $\tau$ & Average return \\ \hline
$1$    &  $1.47 \pm 0.42$ \\
$0.1$  & $3.01 \pm 0.08$ \\
$0.01$ & $3.00 \pm 0.08$ \\ \\ \\ \\ \\  \\
\end{tabular} 
\vspace{-15mm} \\
(a) Performance of KBSF($8\ts000$, $\cdot$)
&
(b) Performance of KBRL($8\ts000$)
\\
\end{tabular}
\caption{The impact of the kernels's widths on the performance of KBSF and
KBRL. Results on the puddle-world task averaged over $50$ runs. The
errors around the mean correspond to the $99\%$
confidence intervals. 
See Figure~\ref{fig:puddle} for details. \label{fig:puddle_all}}
\end{figure*} 

Alternatively, one can fix $\tau$ and $\taub$ and define the
neighborhood used to compute $\gaussa(\rs_{j},\cdot)$ and 
$\gaussb(\yia, \cdot)$. As explained in Appendix~\ref{seca:algorithms}, in 
some of our
experiments we only computed $\gaussa(\rs_{j},\cdot)$ for the $\nk$ closest 
sampled states \xia\ from $\rs_{j}$, and only computed $\gaussb(\yia,
\cdot)$ for the $\nkb$ closest representative states from $\yia$.
When using this approach, a possible way of configuring KBSF is to set $\tau$
and $\taub$ to sufficiently large values (so as to guarantee a minimum level of
overlap between the kernels) and then adjust \nk\ and \nkb. The advantage is
that adjusting \nk\ and \nkb\ may be more intuitive than directly configuring
$\tau$ and $\taub$ ({\sl cf.} Table~\ref{tab:kbsf_params}).

\section{Previous work}
\label{sec:previous}

In our experiments we compared KBSF with KBRL, LSPI, fitted $Q$-iteration, and
SARSA, both in terms of computational cost and in terms of the
quality of the resulting decision policies.
In this section we situate our algorithm in the broader context of
approximate reinforcement learning.
Approximation in reinforcement learning is an important topic that has
generated a huge body of literature. For a broad overview of the subject, we
refer the reader to the books by~\ct{sutton98reinforcement},
\ct{bertsekas96neuro-dynamic}, and \ct{szepesvari2010algorithms}. 
Here we will narrow our attention to \emph{kernel-based} approximation
techniques.

We start by noting that the label ``kernel based'' is used with two different
meanings in the literature. On one side we have kernel smoothing
techniques like KBRL and KBSF, which use local kernels essentially as a device
to implement smooth instance-based approximation~\cp{hastie2002elements}. On the
other side we have methods that use
reproducing kernels to implicitly represent an inner product in a
high-dimensional state space~\cp{scholkopf2002learning}.
Although these two frameworks can give rise to approximators with similar
structures, they rest on different theoretical foundations.
Since reproducing-kernels methods are less directly related to KBSF, we will
only describe them briefly. We will then discuss the kernel smoothing approaches
in more detail.

The basic idea of reproducing-kernel methods is to apply the ``kernel trick''
in the context of reinforcement learning~\cp{scholkopf2002learning}.
Roughly speaking, the approximation problem is rewritten in terms of inner
products only, which are then replaced by a properly-defined kernel. 
This modification corresponds to mapping the problem to a high-dimensional
feature space,
resulting in more expressiveness of the function approximator.
Perhaps the most natural way of applying the kernel trick in the
context of reinforcement learning is to ``kernelize'' some formulation of the
value-function approximation
problem~\cp{xu2005kernel,engel2005reinforcement,farahmand2011regularization}.
Another alternative is to approximate the \emph{dynamics} of an MDP using a
kernel-based regression method~\cp{rasmussen2004gaussian,taylor2009kernelized}.
Following a slightly different line of work,~\ct{bhat2011nonparametric} propose
to kernelize the linear programming
formulation of dynamic programming.
However, this method is not directly applicable to reinforcement learning,
since it is based on the assumption that one has full knowledge of the MDP. A
weaker assumption is to suppose that only the reward function is known and focus
on the approximation of the transition function. This is the approach taken 
by~\ct{grunewalder2012modelling}, who propose to 
embed the conditional distributions defining the transitions of an MDP into a
Hilbert space induced by a reproducing kernel.

We now turn our attention to kernel-smoothing techniques, which are more
closely related to KBRL and KBSF.  \ct{kroemer2011nonparametric} propose
to apply kernel density estimation to the problem of policy evaluation. They
call their method \emph{non-parametric dynamic programming} (NPDP). 
If we use KBRL to compute the value function of a fixed policy, we see many
similarities with NPDP, but also some important differences.
Like KBRL, NPDP is statistically consistent. Unlike KBRL, which 
assumes a finite action space $A$ and directly
approximates the conditional density functions $P^{a}(\sprime| s)$,
NPDP assumes that $A$ is continuous and models the joint density 
$P(s,a,\sprime)$. \ct{kroemer2011nonparametric} showed that the
value function of NPDP has a Nadaraya-Watson kernel
regression form. Not surprisingly, this is also the form of KBRL's solution if
we fix the policy being evaluated ({\sl cf.} equation~(\ref{eq:kbrl_q})).
In both cases, the coefficients of the kernel-based approximation are derived
from the value function of the approximate MDP.
The key difference is the way the transition matrices are computed in each
algorithm.
As shown in~(\ref{eq:mdp_kbrl_P}), the transition probabilities of KBRL's model
are given by the kernel values themselves. In contrast, the computation of each
element of NDPD's transition matrix requires an integration over 
the continuous state space \Sc. 
In practice, this is done by numerical integration techniques that may be 
very computationally demanding (see for
example the experiments performed by~\cwp{grunewalder2012modelling}).

We directly compared NPDP with KBRL because both algorithms build a model whose
number of states is dictated by the number of sample transitions $n$, and 
neither method explicitly attempts to keep $n$ small.
Since in this case each application of the Bellman operator is $O(n^{2})$,
these methods are not suitable for problems in which a large number
of transitions are required, nor are they applicable to on-line
reinforcement learning.\footnote{We note that, incidentally, all the
reproducing-kernel methods discussed in this section also have a computational
complexity super-linear in $n$.}
There are however kernel-smoothing methods that try to avoid this computational
issue by either keeping $n$ small or by executing a number of
operations that grows only linearly with $n$.
These algorithms are directly comparable with KBSF.

One of the first attempts to adapt KBRL to the on-line scenario was that 
of~\ct{jong2006kernelbased}. Instead of collecting a batch of sample
transitions before the learning process starts, the authors propose 
to grow such a set incrementally, based on an exploration policy derived from
KBRL's current model.
To avoid running a dynamic-programming algorithm to completion in between 
two transitions, which may not be computationally
feasible,~\ct{jong2006kernelbased}
resort to~\ctp{moore93prioritized} ``prioritized sweeping'' method to
propagate the changes in the value function every time the model is modified.
The idea of exploiting the interpretation of KBRL as the derivation of a finite
MDP in order to use tabular exploration methods is insightful.
However, it is not clear whether smart exploration is sufficient to overcome the
computational difficulties arising from the fact that
the size of the underlying model is inexorably linked to the number of sample
transitions. For example, even using sparse kernels in their experiments,
\ct{jong2006kernelbased} had to fix an upper limit for the size of KBRL's model.
In this case, once the number of sample transitions has reached the upper
limit, all subsequent data must be ignored. 

Following the same line of work,~\ct{jong2009compositional} later 
proposed to guide KBRL's exploration of the state space
using~\ctp{brafman2003rmax} R-MAX algorithm.
In this new paper the authors address the issue with KBRL's scalability
more aggressively.
First, they show how to combine their approach
with~\ctp{dietterich2000hierarchical} MAX-Q algorithm,
allowing the decomposition of KBRL's MDP into a hierarchy of simpler models.
While this can potentially reduce the computational burden of finding a policy,
such a strategy transfer to the user the responsibility of identifying a useful
decomposition of the task. A more practical approach is to combine 
KBRL with some stable form of value-function approximation.
For that, \ct{jong2009compositional} suggest the use of~\ctp{gordon95stable2}
averagers. As shown in Appendix~\ref{seca:averagers}, 
this setting corresponds to a particular case of KBSF in which
representative states are selected among the set of sampled states \yia.
It should be noted that, even when using temporal abstraction and
function approximation, \ctp{jong2009compositional} approach requires
recomputing KBRL's transition probabilities at each new sample, which can be
infeasible in reasonably large problems.

\ct{kveton2012kernel} propose a more practical algorithm to reduce KBRL's
computational cost. Their method closely resembles the batch version of KBSF. As
with our algorithm, \ctp{kveton2012kernel} method defines a set of
representative states $\rs_{i}$ that give rise to a reduced MDP. The main
difference in the construction of the models is that, instead of computing a
similarity measure between each sampled state $\yia$ and all representative
states $\rs_{j}$, their algorithm 
associates each \yia\ with a single $\rs_{j}$---which comes down to computing
a hard aggregation of the state space \cS. Such an aggregation corresponds to
having a matrix \D\ with a single nonzero element per
row. In fact, it is possible to rewrite \ctp{kveton2012kernel} algorithm using
KBSF's formalism. In this case, the elements of \Dda\ and \Kda\ would be defined
as:
\begin{equation}
\label{eq:kveton}
\begin{array}{ccc}
\dot{k}^{a}_{ij} = \kera(\rs_{i}, \neib(\xja,1)), 
& \text{ and } & 
\dot{d}^{a}_{ij} = \kerk(\neib(\yia,1),\rs_{j}) \\
\end{array}
\end{equation}
where $\kerk$ is the normalized kernel induced by an infinitely
``narrow'' kernel $\gaussk(s, \sprime)$ whose value 
is greater than zero if and only if $s = \sprime$ (recall from
Section~\ref{sec:theory_batch} that $\neib(s,1)$ gives the closest
representative state from $s$).
It is easy to see that we can make matrix \D\ computed by KBSF as close as
desired to a hard aggregation by setting \taub\ to a sufficiently small value
(see Lemma~\ref{teo:split}). 
More practically, we can simply plug~(\ref{eq:kveton}) in place
of~(\ref{eq:mat_kbsf}) in Algorithm~\ref{alg:kbsf} to exactly recover
\ca{kveton2012kernel}'s method. 
Note though that, by replacing 
$\kera(\rs_{i}, \xja)$ with $\kera(\rs_{i}, \neib(\xja,1))$ in the computation
of \Kda, we would be actually deviating from KBRL's framework.
To see why this is so, note that if the representative states $\rs_{i}$ are
sampled from the set of states \yia, the rows of matrix \Ka\ computed by
KBSF would coincide with a subset of the rows of the corresponding KBRL's matrix
\Pca\ ({\sl cf.}~(\ref{eq:kfunca})).  However, this property is lost if
one uses~(\ref{eq:kveton}) instead 
of~(\ref{eq:mat_kbsf}).\footnote{This observation does not imply that 
\ca{kveton2012kernel}'s algorithm is not a principled method. }

\section{Conclusion}
\label{sec:conclusion}

This paper presented KBSF, a reinforcement learning algorithm that results from
the application of the stochastic-factorization trick to KBRL.
KBSF summarizes the information contained in KBRL's MDP in a model of fixed
size. By doing so, our algorithm decouples the structure of the model from its
configuration. This makes it possible to build an approximation which accounts
for both the difficulty of the problem \emph{and} the computational
resources available.

One of the main strengths of KBSF is its simplicity. 
As shown in the paper, its uncomplicated mechanics can be unfolded into two
update rules that
allow for a fully incremental version of the algorithm.
This makes the amount of memory used by KBSF independent of the number of
sample transitions.
Therefore, with a few lines of code one has a 
reinforcement-learning algorithm that can be
applied to large-scale problems, in both off-line and on-line regimes.

KBSF is also a sound method from a theoretical point of view. As discussed,
the distance between the value function computed by this algorithm and the one
computed by KBRL is bounded by two factors: the quality 
and the level of stochasticity of the underlying stochastic
factorization. We showed that both factors can be made arbitrarily small, which 
implies that, in theory, we can make KBSF's solution as close to KBRL's solution
as desired.

But theoretical guarantees do not always translate into practical methods,
either because they are built upon unrealistic assumptions or because they do
not account for procedural difficulties that arise in practice. 
To ensure that this is not the case with our algorithm, we presented an
extensive empirical study in which KBSF was successfully applied to 
different problems, some of them quite challenging. We also presented general
guidelines on how to configure KBSF to solve a reinforcement learning
problem.

For all the reasons listed above, we believe that KBSF has the potential of 
becoming a valuable resource in the solution of reinforcement learning
problems. This is
not to say that the subject has been exhausted. 
There are several possibilities for future research, some of which we now
briefly discuss.

From an algorithmic perspective, perhaps the most pressing demand is for more
principled methods to select the
representative states. Incidentally, this also opens up the possibility of an
automated procedure to set the kernel's widths \taub\ based solely on data.
Taking the idea a bit further, one can think of having one distinct $\taub_{i}$
associated with each kernel $\kerb(\cdot, \rs_{i})$.
Another important advance 
would be to endow \ikbsf\ with more elaborate exploration strategies, maybe
following the line of research initiated
by~\ct{jong2006kernelbased,jong2009compositional}.

Regarding the integration of KBSF to its broader context,
a subject that deserves further investigation is the possibility of 
building an approximation based on multiple models. Model averaging
is not inherently linked to KBSF, and in principle it can be used with virtually
any reinforcement learning algorithm. However, KBSF's low
computational cost makes it particularly amenable to this technique. Since
our algorithm is orders of magnitude faster than any method whose
complexity per iteration is a function of the number of sample transitions, we
can afford to compute several approximations and still have a solution in
comparable time (see Section~\ref{sec:hiv}). Understanding to
what extend this can improve the quality of
the resulting decision policy is a matter of interest.

In this paper we emphasized the role of KBSF as a technique to reduce
KBRL's computational cost. However, it is equally important to ask whether 
our algorithm provides benefits from a statistical point of view.
~\ct{ormoneit2002kernelbased} showed that, in general, the number of sample
transitions needed by KBRL to achieve a certain approximation accuracy 
grows exponentially with the dimension of the state
space. As with other methods, the only way to avoid
such an exponential dependency is to explore some sort of regularity in the
problem's structure---paraphrasing the authors, one can only
``break'' the curse of dimensionality by incorporating prior knowledge into the
approximation~\cp{ormoneit2002kernelbased}. We think that KBSF may be cast as a
strategy to do so. In particular, the definition of the representative states
can be interpreted as a practical mechanism to incorporate knowledge into the
approximation. Whether or not this will have an impact on the algorithm's
sample complexity is an
interesting question for future investigation.

We conclude by noting that KBSF represents one particular way in which the
stochastic-factorization trick can be exploited in the context of reinforcement
learning. In principle, any algorithm that builds a model based on sample
transitions can resort to the same trick to leverage the use of the data. The
basic idea remains the same: instead of estimating the transition probabilities
between every pair of states, one focuses on a small set of representative
states whose values are propagated throughout the state space based on some
notion of similarity. We believe that this general framework can potentially be
materialized into a multitude of useful reinforcement learning
algorithms.

\appendix

\section{Theoretical Results}
\label{seca:theory}

\subsection{Assumptions}
\label{seca:assumptions}

We assume that KBSF's kernel $\mkb(x) : \R^{+} \mapsto \R^{+}$ has the
following properties: 
\begin{enumerate}[(i)]
 \item \label{it:dec} $\mkb(x) \ge \mkb(y)$ if $x < y$,
\item \label{it:ed} $\exists \;  A_{\mkb} > 0, \lambda_{\mkb} \ge 1, B
\ge 0 \text{ such that } A_{\mkb} \exp(-x) \le \mkb(x) \le \lambda_{\mkb}
A_{\mkb} \exp(-x) 
\text{ if } x \ge B.$
\newcounter{enumi_saved}
\setcounter{enumi_saved}{\value{enumi}}
\end{enumerate}
Given \mkb, we will denote by $B_{\mkb}$ the smallest $B$ that
satisfies~(\ref{it:ed}).
Assumption~(\ref{it:ed}) implies that the function \mkb\ 
is positive and will eventually decay exponentially.
Note that we assume that \mkb\ is greater than zero everywhere 
in order to guarantee
that \kerb\ is well defined for any value of \taub. 
It should be straightforward to generalize our results for the case 
in which \mkb\ has finite support by ensuring
that, given sets of sample transitions \Sca\ and a set of representative
states \Sb, \taub\ is such that, for any $\yia \in \Sca$, with $a \in A$,
there is a $\rs_{j} \in \Sb$ for which $\gaussb(\yia, \rs_{j}) > 0$
(note that this assumption is naturally satisfied by the
``sparse kernels'' used in some of the experiments).

\subsection{Proofs}
\label{seca:proofs}

\reptheo{Lemma~\ref{teo:dist_cont}}
{
For any $\xia \in \Sca$ and any $\epsilon > 0$, there is a $\delta > 0$ such
that $|\kera(s,\xia) - \kera(\sprime,\xia)| < \epsilon$ if 
$\norm{s - \sprime} <~\delta$.
}
\begin{proof}
Define the function
{\small
\begin{equation*}
\dista^{a,i}_{\tau, s}(\sprime) 
= \left| \dfrac{\gaussa(s, \xia)}{\sum_{j=1}^{n_a} \gaussa(s, \xja)}
- \dfrac{\gaussa(\sprime, \xia)}{\sum_{j=1}^{n_a} \gaussa(\sprime, \xja)}
\right|
=
\left| \dfrac{\mk\left({\norm{s - \xia}}/{\tau}\right)}
{\sum_{j=1}^{n_a} \mk\left({\norm{s - \xja}}/{\tau}\right)}
- \dfrac{\mk\left({\norm{\sprime - \xia}}/{\tau}\right)}
{\sum_{j=1}^{n_a} \mk\left({\norm{\sprime - \xja}}/{\tau}\right)} \right|.
\end{equation*}
}

\noindent
Since \mk\ is continuous, it is obvious that $\dista^{a,i}_{\tau, s}(\sprime)$
is also continuous in $\sprime$. The property
follows from the fact that 
$\lim_{\sprime \rightarrow s} \dista^{a,i}_{\tau, s}(\sprime) = 0$. 
\end{proof}

\long\def\symbolfootnote[#1]#2{\begingroup%
\def\thefootnote{\fnsymbol{footnote}}\footnote[#1]{#2}\endgroup} 
\reptheo{Lemma~\ref{teo:split}\symbolfootnote[1]{We restate
the lemma here showing explicitly how to define \taub. This detail was omitted
in the main body of the text to improve clarity.}}
{
Let $s \in \Sc$, let $m > 1$, and assume there is a $w \in \{1, 2, ..., m-1\}$
such that $\ddb(s, w) < \ddb(s,w+1)$. Define 
$
W \equiv \{k \;|\; \norm{s\ - \rs_{k}} \le \ddb(s,w) \}
\text{ and }
\bar{W} \equiv \{1, 2, ..., m\} - W.
$
Then, for any
$\alpha > 0$, 
we can guarantee that 
\begin{equation}
\label{eq:res_split}
\sum_{k \in \bar{W}} \kerb(s, \rs_{k}) < \alpha  \sum_{k \in W} 
\kerb(s, \rs_{k})
\end{equation}
by making
$ 
\taub < \varphi(s, w, m, \alpha),
$
where 
\begin{equation}
\label{eq:set_vt}
\varphi(s, w, m, \alpha) = \min(\varphi_{1}(s,w),\varphi_{2}(s,w,
m,\alpha))
\end{equation}
and

{\footnotesize
\begin{align*}
\varphi_{1}(s,w) = \left\{
\begin{array}{l}
  \dfrac{\ddb(s, w)}{B_{\mkb}}, 
\text{ if } B_{\mkb} > 0, \\
\infty, \text{ otherwise, } 
\end{array}
\right.
&
\varphi_{2}(s, w, m, \alpha) = \left\{
\begin{array}{l}
  \dfrac{\ddb(s, w) - \ddb(s, w+1)}{\ln({\alpha w} /
(m-w) \lambda_{\mkb})}, 
\text{ if } \dfrac{\alpha w}{(m-w) \lambda_{\mkb}} < 1, \\
\infty, \text{ otherwise. } 
\end{array}
\right.
\end{align*}
}
}
\begin{proof}
Expression~(\ref{eq:res_split}) can be rewritten as
 \begin{equation*}
\label{eq:ic}
\dfrac{\sum_{k \in \bar{W}} \gaussb(s, \rs_{k})}
{\sum_{i=1}^{m} \gaussb(s, \rs_{i})}
<
 \alpha
\dfrac{\sum_{k \in W} \gaussb(s, \rs_{k})}
{\sum_{i=1}^{m} \gaussb(s, \rs_{i})}
\iff
\sum_{k \in \bar{W}}\gaussb(s, \rs_{k})
<
 \alpha 
\sum_{k \in W}\gaussb(s, \rs_{k}),
\end{equation*}
which is equivalent to 
\begin{equation}
 \label{eq:ic2}
\sum_{k \in \bar{W}} \mkp{\dfrac{\norm{s - \rs_{k}}}{\taub}}
<
 \alpha
\sum_{k \in W}\mkp{\dfrac{\norm{s - \rs_{k}}}{\taub}}.
\end{equation}
Based on Assumption~(\ref{it:dec}), we know that
a sufficient condition for~(\ref{eq:ic2}) to hold is
\begin{equation}
\label{eq:ic3}
\mkp{\dfrac{\ddb(s, w+1)}{\taub}}
< 
\dfrac{\alpha w}{m-w} \mkp{\dfrac{\ddb(s, w)}{\taub}}.
\end{equation}
Let $\beta = {\alpha w}/{(m-w)}$. If $\beta > 1$, 
then~(\ref{eq:ic3}) is always true, regardless of the value of \taub. 
We now show that, when $\beta \le 1$, it is always possible to set \taub\ in
order to guarantee that~(\ref{eq:ic3}) holds.
Let $z = \ddb(s, w)$ and let $\delta = \ddb(s, w+1) - z$.
From Assumption~(\ref{it:ed}), we know that, if $B_{\mkb} = 0$ or
$\taub < z / B_{\mkb}$,
\begin{equation*}
\begin{array}{cl}
\dfrac{\mkb((z + \delta)/\taub)}{\mkb(z / \taub)}
& \le 
\dfrac{\lambda_{\mkb} A_{\mkb} \exp(-(z + \delta)/\taub)}{A_{\mkb} \exp(-z /
\taub)} 
= \dfrac{\lambda_{\mkb} \exp(-(z + \delta)/\taub)}{\exp(-z /\taub)}. \\
\end{array}
\end{equation*}
Thus, in order for the result to follow, it suffices to show that
\begin{equation}
\label{eq:a}
\dfrac{\exp(-(z + \delta)/\taub)}{\exp(-z / \taub)} 
< \dfrac{\beta}{\lambda_{\mkb}}.
\end{equation}
We know that, since $\delta > 0$, if $\beta / \lambda_{\mkb} = 1$
inequality~(\ref{eq:a}) is true. 
Otherwise, 
\begin{equation*}
\label{eq:prop_exp}
\begin{array}{cc}
\dfrac{\exp(-(z + \delta)/\taub)}{\exp(-z / \taub)} <
\dfrac{\beta}{\lambda_{\mkb}}
\iff 
\ln\left(\dfrac{\exp(-(z + \delta)/\taub)}{\exp(-z / \taub)}\right) <
\ln\left(\dfrac{\beta}{\lambda_{\mkb}}\right) 
\\
\iff 
- \dfrac{\delta}{\taub} < \ln\left(\dfrac{\beta}{\lambda_{\mkb}}\right)
\iff 
\taub < - \dfrac{\delta}{\ln(\beta / \lambda_{\mkb})}.
\end{array}
\end{equation*}
Thus, by taking 
$\taub < -\delta / \ln(\beta / \lambda_{\mkb})$ if $B_{\mkb} = 0$, 
or $\taub < \min(-\delta / \ln(\beta / \lambda_{\mkb}), z /
B_{\mkb})$ otherwise, the result follows.
\end{proof}

\noindent
{\bf Note:} 
We briefly provide some intuition on the functions $\varphi_{1}$ and
$\varphi_{2}$. 
Since we know from Assumption~(\ref{it:dec}) that $\mkb$ is non-increasing,
we can control the magnitude of 
$\sum_{k \in \bar{W}} \kerb(s,\rs_{k})$ $/\sum_{k \in W} \kerb(s,\rs_{k})$ by
controlling 
\begin{equation}
\label{eq:frac_control}
\frac{\kerb(s,\neib(s,w+1))}{\kerb(s,\neib(s,w))}
 = \frac{\mkb(\ddb(s,w+ 1)/\taub)}{\mkb(\ddb(s,w)/\taub)}.
\end{equation}
Function 
$\varphi_{1}$ imposes an upper bound on $\taub$ in order to ensure that
$\ddb(s, w) / \taub \ge B_{\mkb}$.
This implies that $\mkb(\ddb(s,w) / \taub)$ will be in the 
``exponential region'' of $\mkb$, which makes it possible to
control the magnitude of~(\ref{eq:frac_control}) by adjusting \taub.
In particular, because of Assumption~(\ref{it:ed}), we know that
$\mkb(\ddb(s,w+1) / \taub) / \mkb(\ddb(s,w) / \taub) \rightarrow 0$ as $\taub
\rightarrow 0$. 
Function $\varphi_{2}$ exploits this fact, decreasing 
the maximum allowed value for \taub\ according to two factors. The first one is 
the difference of magnitude of $\ddb(s,w+1)$ and $\ddb(s,w)$. This is easy to
understand. Suppose we 
want to make~(\ref{eq:frac_control}) smaller than
a given threshold. If $\neib(s,w+1)$ is much farther from $s$ than
$\neib(s,w)$, the value of $\mkb(\ddb(s,w+1)/\taub)$ will be considerably
smaller than the value of $\mkb(\ddb(s,w)/\taub)$ even if \taub\ is large. 
On the other hand, if the difference of magnitude of $\ddb(s,w+1)$ and
$\ddb(s,w)$ is small, 
we have to decrease \taub\ to ensure that~(\ref{eq:frac_control}) is
sufficiently small. Therefore, the upper bound for \taub\ set by $\varphi_{2}$
decreases with $|\ddb(s,w+1) - \ddb(s,w)|$.
The second factor that influences this upper bound is $w / (m - w)$, the
relative sizes of the sets $W$ and $\bar{W}$. Again, this
is not hard to understand: as we reduce the size of $W$, 
we also decrease the number of terms in 
the sum $\sum_{k \in W}\kerb(s,\rs_{k})$, and thus we must 
decrease the ratio~(\ref{eq:frac_control}) to make sure that $\sum_{k \in
\bar{W}} \kerb(s,\rs_{k}) /\sum_{k \in W} \kerb(s,\rs_{k})$ is sufficiently
small. Thus,
the upper bound on \taub\ defined by $\varphi_{2}$ grows with 
$w / (m - w)$.

\vspace{5mm}

\reptheo{Proposition~\ref{teo:bound_rep_states}}{
For any $\epsilon > 0$, there is a $\delta > 0$ such that, if
$\max_{a,i}\ddb(\yia, 1) < \delta$, then we can 
guarantee that $\boundV<~\epsilon$
by making $\taub$ sufficiently small.
}
\begin{proof}
From~(\ref{eq:rca}) and~(\ref{eq:rba}), we know that
\begin{equation}
\label{eq:p_only}
\infnorms{\rca - \D\rba} = \infnorms{\Pca \rr - \D\Ka \rr} = 
\infnorms{(\Pca - \D\Ka) \rr} \le \infnorms{\Pca - \D\Ka}\infnorms{\rr}.
\end{equation}
\hspace{-3mm}
Thus, plugging~(\ref{eq:p_only}) back into~(\ref{eq:bound_sf}), it is
clear that there is a $\nu > 0$ such that 
$\boundV< \epsilon$ if 
\begin{equation}
\label{eq:first_goal}
\max_{a} \infnorms{\Pca - \D\Ka} < \nu
\end{equation}
and
\begin{equation}
\label{eq:second_goal}
\max_{i}{(1 - \max_{j}{d_{ij})}} < \nu.
\end{equation}
We start by showing that there is a $\delta > 0$ and a
$\theta > 0$ such that expression~(\ref{eq:first_goal}) is true
if $\max_{a,i}\ddb(\yia, 1) < \delta$ and $\taub < \theta$. 
Let $\Pxa = \D\Ka$ and 
let $\matii{\hat{p}}{a}{i} \in \R^{1 \times n}$ and 
$\matii{\check{p}}{a}{i} \in \R^{1 \times n}$
be the \ith\ rows of \Pca\ and \Pxa, respectively.
Then,
\begin{equation}
\label{eq:dif_pij}
\begin{array}{cl}
\infnorm{\matii{\hat{p}}{a}{i} - \matii{\check{p}}{a}{i}} 
& = 
\sum_{j=1}^{n_a} 
|\hat{p}_{ij}^{a} - \sum_{k=1}^{m} \dot{d}^{a}_{ik}\dot{k}^{a}_{kj}|
\\
& = 
\sum_{j=1}^{n_a} 
|\kera(\yia,\xja) - \sum_{k=1}^{m} \kerb(\yia,\rs_{k}) \kera(\rs_{k},\xja)|
\\
& = 
\sum_{j=1}^{n_a} 
|\sum_{k=1}^{m} \kerb(\yia,\rs_{k}) \kera(\yia,\xja) - \sum_{k=1}^{m}
\kerb(\yia,\rs_{k}) \kera(\rs_{k},\xja)|
\\
& = 
\sum_{j=1}^{n_a} 
|\sum_{k=1}^{m} \kerb(\yia,\rs_{k}) [\kera(\yia,\xja) -
\kera(\rs_{k},\xja)]|
\\
& \le 
\sum_{j=1}^{n_a} 
\sum_{k=1}^{m} \kerb(\yia,\rs_{k}) \left|\kera(\yia,\xja) -
\kera(\rs_{k},\xja)\right|.
\\
\end{array}
\end{equation}
Our strategy will be to show that, for any $a$, $i$, and $j$, there is a
$\delta^{a,i,j} > 0$ and a $\theta^{a, i, j} > 0$ such that 
\begin{equation}
\label{eq:goal1}
\sum_{k=1}^{m} \kerb(\yia, \rs_{k}) | \kera(\yia,\xja) -\kera(\rs_{k},\xja) |
< \frac{\nu}{n_{a}}
\end{equation}
if $\ddb(\yia,1) < \delta^{a,i,j}$ and $\taub < \theta^{a,i,j}$.
To simplify the notation, we will use the superscript `$\aij$' meaning
`$a,i,j$'. Define $\difajik \equiv | \kera(\yia,\xja) -\kera(\rs_{k},\xja) |$.
From Lemma~\ref{teo:dist_cont} we know that there is a $\delta^{\aij} > 0$
such that $\difajik < \nu / n_{a}$ 
if $\norm{\yia - \rs_{k}} < \delta^{\aij}$.
Let $W^{\aij} \equiv \{k \;|\; \norm{\yia\ - \rs_{k}} < \delta^{\aij} \}$
and 
$\bar{W}^{\aij} \equiv \{1, 2, ..., m\} - W^{\aij}$.  
Since we are assuming that $\ddb(\yia, 1) < \delta^{\aij}$, 
we know that $W^{\aij} \ne \emptyset$.
In this case, we can write:
\begin{equation*}
\sum_{k=1}^{m}
\kerb(\yia, \rs_{k})
\difajik \\
= 
\sum_{k \in W^{\aij}}
\kerb(\yia, \rs_{k})
\difajik
+
\sum_{k \in \bar{W}^{\aij}}
\kerb(\yia, \rs_{k})
\difajik. 
\end{equation*}
Let 
\begin{equation*}
\begin{array}{ccc}
\minv = 
\left\{\begin{array}{l}
\mmin{k \in W^{\aij}} \{\difajik | \difajik > 0\} \text{
if }
\mmax{k  \in W^{\aij}} \difajik > 0, \\
0 \text{ otherwise} 
       \end{array}\right.
&
\text{and}
&
\maxv = 
\left\{\begin{array}{l}
\mmax{k \in \bar{W}^{\aij}} \difajik \text { if } |W^{\aij}| < m, \\
0 \text{ otherwise} .
       \end{array}\right.
\end{array}
\end{equation*}
If $\maxv = 0$, inequality~(\ref{eq:goal1}) is necessarily true,
since 
$
\sum_{k \in W^{\aij}} 
\kerb(\yia, \rs_{k}) \difajik \le \mmax{k \in W^{\aij}}
\difajik < {\nu}/{n_{a}}.
$
We now turn to the case in which $\maxv > 0$.
Suppose first that $\minv = 0$. In this case, we have to show that
there is a $\taub$ that yields
\begin{equation}
\label{eq:inwb}
\begin{array}{cl}
\sum_{k \in \bar{W}^{\aij}}
\kerb(\yia, \rs_{k})
\difajik
< \dfrac{\nu}{n_{a}}.\\ 
\end{array}
\end{equation}
A sufficient condition for~(\ref{eq:inwb}) to be true is
\begin{align}
\label{eq:inwb2}
\sum_{k \in \bar{W}^{\aij}}
\kerb(\yia, \rs_{k}) < \dfrac{\nu}{n_{a} \maxv}
\iff
\dfrac{1}{\sum_{j=1}^{m} \gaussb(\yia, \rs_{j})}
\sum_{k \in \bar{W}^{\aij}}
\gaussb(\yia, \rs_{k}) < \frac{\nu}{ n_{a} \maxv}.
\end{align}
Obviously, if $\maxv \le \nu / n_{a}$ inequality~(\ref{eq:inwb2}) is
always true, regardless of the value of \taub. Otherwise, we can
rewrite~(\ref{eq:inwb2}) as
\begin{align*}
\label{eq:inwb3}
\sum_{k \in \bar{W}^{\aij}}
\gaussb(\yia, \rs_{k}) < \dfrac{\nu}{n_{a}\maxv}
\left(
{\sum_{j\in W^{\aij}} \gaussb(\yia, \rs_{j}) + \sum_{k \in \bar{W}^{\aij}}
\gaussb(\yia,\rs_{k})}\right),
\end{align*}
and, after a few algebraic manipulations, we obtain
{\small
\begin{equation}
\sum_{k \in \bar{W}^{\aij}}
\gaussb(\yia, \rs_{k}) 
< \dfrac{\nu}{n_{a} \maxv  - \nu}
\sum_{k\in W^{\aij}} \gaussb(\yia, \rs_{k}),
\iff
\label{eq:second_exit}
\sum_{k \in \bar{W}^{\aij}}
\kerb(\yia, \rs_{k}) 
< \dfrac{\nu}{n_{a} \maxv - \nu}
\sum_{k\in W^{\aij}} \kerb(\yia, \rs_{k}).
\end{equation}
}
We can guarantee that~(\ref{eq:second_exit}) is true by applying
Lemma~\ref{teo:split}. Before doing so, though, lets analyze the case 
in which $\minv > 0$.
Define
\begin{equation}
\label{eq:beta}
\beta^{\aij} = \dfrac{\nu} 
{n_{a} \sum_{k \in W^{\aij}}
\kerb(\yia, \rs_{k})\difajik} - 1 
\end{equation}
(note that $\beta^{\aij} > 0$ because 
$\sum_{k \in W^{\aij}} \kerb(\yia, \rs_{k})\difajik < v/n_{a}$).
In order for~(\ref{eq:goal1}) to hold, we must show that there is a $\taub$ that
guarantees that
\begin{equation}
\label{eq:split}
\begin{array}{cl}
\sum_{k \in \bar{W}^{\aij}}
\kerb(\yia, \rs_{k})
\difajik - 
 \beta^{\aij} 
\sum_{k \in W^{\aij}}
\kerb(\yia, \rs_{k})
\difajik
< 0.\\ 
\end{array}
\end{equation}
A sufficient condition for~(\ref{eq:split}) to hold is
\begin{equation}
\label{eq:first_exit}
\begin{array}{cl}
\sum_{k \in \bar{W}^{\aij}}
\kerb(\yia, \rs_{k})
< 
 \dfrac{\beta^{\aij} \minv}{\maxv} \sum_{k \in W^{\aij}}
\kerb(\yia, \rs_{k}).
\end{array}
\end{equation}
Observe that expressions~(\ref{eq:second_exit}) and~(\ref{eq:first_exit}) only
differ in the coefficient multiplying the right-hand side of the inequalities.
Let $\alpha^{\aij} < \min(\nu/(\maxv n_{a} -\nu), {\beta^{\aij}
\minv}/{\maxv})$. Then, if we make 
$\theta^{\aij} = \varphi(\yia,|W|,m,\alpha^{\aij})$, with
$\varphi$ defined in~(\ref{eq:set_vt}),
we can apply Lemma~\ref{teo:split} to guarantee that~(\ref{eq:goal1}) holds.
Finally, if we let $\delta = \min_{\aij} \delta^{\aij}  = \min_{a,i,j}
\delta^{a,i,j}$ and 
$\theta = \min_{\aij} \theta^{\aij} = \min_{a,i,j} \theta^{a,i,j}$, we can
guarantee that~(\ref{eq:goal1}) is true for all $a$, $i$, and $j$, which
implies that~(\ref{eq:first_goal}) is also true (see~(\ref{eq:dif_pij})).

It remains to show that there is a $\omega > 0$
such that~(\ref{eq:second_goal}) is true if $\bar{\tau} < \omega$. 
Recalling that, for any $i$ and any $a$,
\begin{equation*}
\max_{j}\dot{d}^{a}_{ij} = 
\dfrac{\gaussb(\yia, \neib(\yia, 1))}{\sum_{k=1}^{m}\gaussb(\yia, \rs_{k})},
\end{equation*}
we want to show that
\begin{equation*}
\label{eq:goal2}
\gaussb(\yia, \neib(\yia, 1)) > (1 - \nu) 
\left[\gaussb(\yia, \neib(\yia, 1)) + \sum_{k=2}^{m}\gaussb(\yia, \neib(\yia,
k))\right],
\end{equation*}
which is equivalent to
\begin{equation}
\label{eq:cond_dij}
(1 - \nu)
\sum_{k=2}^{m}\gaussb(\yia, \neib(\yia, k))
< \nu \gaussb(\yia, \neib(\yia, 1)). 
\end{equation}
If $\nu \ge 1$, inequality~(\ref{eq:cond_dij}) is true regardless of 
the particular choice of \taub. Otherwise, we can
rewrite~(\ref{eq:cond_dij}) as
{\small
\begin{equation}
\label{eq:third_exit}
\sum_{k=2}^{m}\gaussb(\yia, \neib(\yia, k))
< \dfrac{\nu}{1 - \nu} \gaussb(\yia, \neib(\yia, 1))
\iff 
\sum_{k=2}^{m}\kerb(\yia, \neib(\yia, k))
< \dfrac{\nu}{1 - \nu} \kerb(\yia, \neib(\yia, 1)). 
\end{equation}
}
Let $\alpha = {\nu}/{(1 - \nu)}$. Then, if
we make $\omega^{a,i} = \varphi(\yia,1,m,\alpha)$, with
$\varphi$ defined in~(\ref{eq:set_vt}), we can resort
to Lemma~\ref{teo:split} to guarantee that~(\ref{eq:third_exit}) holds.
As before, if we let $\omega = \min_{a,i} \omega^{a,i}$, we can guarantee 
that~(\ref{eq:second_goal}) is true. Finally, by making 
$\taub = \min(\theta, \omega)$, the result follows.
\end{proof}

\long\def\symbolfootnote[#1]#2{\begingroup%
\def\thefootnote{\fnsymbol{footnote}}\footnote[#1]{#2}\endgroup} 

\reptheo{Lemma~\ref{teo:bound_q}}
{
Let $M \equiv (S, A, \Pa, \ra,\gamma)$ and 
$\tilde{M} \equiv (S, A, \Pta, \rta,\gamma)$ be two finite MDPs. Then,
for any $s \in S$ and any $a \in A$,
\begin{equation}
\label{eq:bound_q}
|{Q}^{*}(s, a) - \tilde{Q}^{*}(s, a)| \le 
\frac{1}{1-\gamma} \maxinf{\ra}{\rta} + \frac{\gamma (2 -
\gamma)}{2(1-\gamma)^{2}} 
R_{\dif} 
\maxinf{\Pa}{\Pta}, 
\end{equation}
where 
${R}_{\dif} = \max_{a,i} {r}^{a}_{i} - \min_{a,i}
{r}^{a}_{i}$.
}
\begin{proof}
Let $\matii{q}{a}{*}, \matii{\tilde{q}}{a}{*} \in \R^{|S|}$ be the
\ath\ columns of \Qo\ and \Qto, respectively.
Then, 
\begin{align}
\nonumber
\infnorm{\matii{{q}}{a}{*} - \matii{\tilde{q}}{a}{*}} 
& \nonumber = \infnorm{\ra + \gamma \Pa \vo - \rta - \gamma \Pta \vto}  \\
& \nonumber \le \infnorm{\ra - \rta} + \gamma \infnorm{\Pa \vo - \Pta \vto} \\
& \nonumber = \infnorm{\ra - \rta} + \gamma 
\infnorm{\Pa \vo - \Pta \vo + \Pta \vo - \Pta \vto} \\
& \nonumber \le \infnorm{\ra - \rta} + \gamma \infnorm{\vo (\Pa - \Pta) }
+ \gamma \infnorm{\Pta (\vo - \vto)}  \\
&  \label{eq:dif_q}
\le \infnorm{\ra - \rta} + \gamma \infnorm{\vo (\Pa - \Pta) }
+ \gamma \infnorm{\vo - \vto}, 
\end{align}
where in the last step we used the fact that $\Pta$ is stochastic, and 
thus $\infnorms{\Pta \v} \le \infnorms{\v}$ for any \v.
We now provide a bound for \infnorms{\vo (\Pa - \Pta)}.
Let $\A = \Pa - \Pta$.
Then, for any $i$, 
$
\sum_{j} a_{ij} 
 = \sum_{j} ( {p}^{a}_{ij} - \tilde{p}^{a}_{ij})
= \sum_{j}  {p}^{a}_{ij} - \sum_{j} \tilde{p}^{a}_{ij} = 0,
$
that is, the elements in each row of \A\ sum to zero.  Let $a^{+}_{i}$ be the
sum of positive elements in the \ith\ row of \A\ and 
let $a^{+}_{\max} = \max_{i} a^{+}_{i}$. 
It should be clear that $\infnorm{\A} = 2 a^{+}_{\max}$.
Then, for any $i$,
\begin{align}
\nonumber
|\sum_{j} a_{ij}  {v}^{*}_{j} | 
& \le \sum_{(j : a_{ij} > 0)} a_{ij}  {v}^{*}_{\max} 
+ \sum_{(j: a_{ij} < 0)} a_{ij}  {v}^{*}_{\min} 
 = a^{+}_{i}  {v}^{*}_{\max} - 
a^{+}_{i}  {v}^{*}_{\min} 
  \le a^{+}_{\max} ( {v}^{*}_{\max} -  {v}^{*}_{\min}) \\
& \label{eq:bound_av}
\le \frac{a^{+}_{\max}}{1-\gamma} ( {r}^{a}_{\max} -  {r}^{a}_{\min})
\le \frac{a^{+}_{\max} R_{\dif}}{1-\gamma} 
=  \frac{  {R}_{\dif}}{2(1-\gamma)} \infnorms{\Pa - \Pta},
\end{align}
where we used the convention ${v}_{\max} = \max_{i}{v}_{i}$ 
(analogously for $v_{\min}$).
As done in~(\ref{eq:whitt}), we can resort to~\ctp{whitt78approximations}
Theorem~3.1 and Corollary~(b) of his Theorem~6.1 to obtain a bound 
for \infnorm{\vo - \vto}. 
Substituting such a bound and expression~(\ref{eq:bound_av}) 
in~(\ref{eq:dif_q}), we obtain

{\footnotesize
\begin{align*}
\infnorm{\matii{q}{a}{*} - \matii{\tilde{q}}{a}{*}} 
\le \infnorm{\ra - \rta} 
+ \dfrac{\gamma  R_{\dif}}{2(1-\gamma)} \infnorms{\Pa - \Pta} 
+ \dfrac{\gamma }{1-\gamma} \left(\maxinf{\ra}{\rta} + \dfrac{ R_{\dif}}
{2(1-\gamma)} \maxinf{\Pa}{\Pta} \right) 
\\ 
\le \maxinf{\ra}{\rta} + \dfrac{ R_{\dif}} {2(1-\gamma)} \maxinf{\Pa}{\Pta} 
+  \dfrac{\gamma }{1-\gamma} \left(\maxinf{\ra}{\rta} 
+ \dfrac{\gamma  R_{\dif}}{2(1-\gamma)} \maxinf{\Pa}{\Pta}
\right).
\end{align*}
}
\end{proof}

\noindent
{\bf Note:} From the proof of Lemma~\ref{teo:bound_q} we see that
{\footnotesize
\begin{equation*}
|{Q}^{*}(s_i, a) - \tilde{Q}^{*}(s_i, a)|
\le |r^{a}_{i} -  \tilde{r}^{a}_{i}|
+ \dfrac{\gamma  R_{\dif}}{2(1-\gamma)} \infnorms{\Pa - \Pta} 
+ \dfrac{\gamma }{1-\gamma} \left(\maxinf{\ra}{\rta} + \dfrac{ R_{\dif}}
{2(1-\gamma)} \maxinf{\Pa}{\Pta} \right), 
\end{equation*}
}

\noindent
which is tighter than~(\ref{eq:bound_q}). Here we favor the more
intelligible version of the bound, but of course Proposition~\ref{teo:inc_kbsf}
could also have been derived based on the expression above.

\subsection{Alternative error bound}
\label{seca:averagers}

In Section~\ref{sec:sf} we derived an upper bound for the 
approximation error introduced by the application of the
stochastic-factorization trick. In this section we introduce another bound that
has different properties. First, the bound is less applicable,
because it depends on quantities that are usually unavailable 
in a practical situation (the fixed points of two contraction mappings).
On the bright side, unlike the bound presented in
Proposition~\ref{teo:bound_sf}, the new bound is valid for any norm. Also,
it draws an interesting connection with an
important class of approximators known as
\emph{averagers}~\cp{gordon95stable2}.

We start by deriving a theoretical result that only applies to stochastic
factorizations of order $n$. We then generalize this result to the case in
which the factorizations are of order $m <  n$.

\begin{lemma}
\label{teo:pre_bound}
Let $M \equiv(S,A,\Pa,\ra,\gamma)$ be a finite MDP with $|S| = n$ and 
$0 \le \gamma < 1$. Let $\E\La = \Pa$ be $|A|$ stochastic factorizations of
order $n$ and let $\rdda$ be vectors in $\R^{n}$ such that $\E\rdda = \ra$ for
all $a \in A$. Define the MDPs $\hM \equiv(S,A,\La,\rdda,\gamma)$ and 
$\ddM \equiv(S,A,\Pdda,\rdda,\gamma)$, with $\Pdda = \La \E$. Then,
\begin{small}
\begin{equation}
\label{eq:pre_bound}
\norm{\vo - T \E \vddo} \le \boundVV \equiv
\dfrac{2\gamma}{1-\gamma} \norm{\vo -\u} +
\dfrac{\gamma(1+\gamma)}{1-\gamma} \norm{\vo - \vho},
\end{equation}
\end{small}
where $\norm{\cdot}$ is a norm in $R^{n}$ and $\u $  is a vector in $\R^{n}$
such that $\E\u = \u$.
\end{lemma}
\begin{proof}
The Bellman operators of $M$, \hM, and
\ddM\ are given by $T = \Rd\Ex$, $\hT = \Rd \hEx$, and $\ddT = \Rd \ddEx$.
Note that
$
\qa 
 = \ra + \gamma \Pa \v 
 = \E\rdda + \gamma \E\La \v 
 = \E (\rdda + \gamma \La \v),
$
where \qa\ is the \ath\ column of \Q.
Thus, $\Ex = \E\hEx$. Since \E\ is stochastic, we can think of it as 
one of \ctp{gordon95stable2} averagers given by
$A(\v) = \E\v$, and then resort to Theorem~4.1 by the same author to conclude
that $\ddT = \E \hT$. Therefore,\footnote{Interestingly,
the effect of swapping
matrices \E\ and \La\ is to also swap the operators \Rd\ and \E.}
\begin{eqnarray}
\label{eq:no}
T\v = \Rd\E\hEx \v & \mbox{ and } & \ddT \v = \E\Rd\hEx\v.
\end{eqnarray}
Using~(\ref{eq:no}), it is easy to obtain the desired upper bound
by resorting to the triangle inequality, the definition of a contraction map,
and \ctp{denardo67contraction} Theorem~1:
{\small
\begin{align*} 
\norm{\vo - T \E \vddo} 
& \le \ \gamma \norm{\vo - \E \vddo} 
 \le \gamma (\norm{\vo - \u} + \norm{\u - \E \vddo}) 
 \le \gamma (\norm{\vo - \u} + \norm{\u - \vddo}) \\ 
 & \le \gamma \left(\norm{\vo - \u} + \dfrac{1}{1-\gamma} \norm{\u - \E\Rd\hEx\u
}\right) 
 \le \gamma \left(\norm{\vo - \u} + \dfrac{1}{1-\gamma} \norm{\u - \Rd\hEx\u
}\right)\\
 & \le \gamma \left[\norm{\vo - \u} + \dfrac{1}{1-\gamma} \left( \norm{\u -
\vho} + \norm{\vho - \Rd\hEx\u} \right) \right] \\
 & \le \gamma \left[\norm{\vo - \u} + \dfrac{1}{1-\gamma} \left( \norm{\u -
\vho} + \gamma \norm{\vho - \u} \right) \right]
  = \gamma \left[\norm{\vo - \u} + \dfrac{1+\gamma}{1-\gamma} \norm{\u -
\vho} \right]\\
&  \le \gamma \left[\norm{\vo - \u} + \dfrac{1+\gamma}{1-\gamma} \left(
\norm{\u - \vo} + \norm{\vo - \vho} \right)\right] \\
& = \gamma \norm{\vo - \u} +  \dfrac{\gamma(1+\gamma)}{1-\gamma} \norm{\vo -
\u} + \dfrac{\gamma(1+\gamma)}{1-\gamma} \norm{\vo - \vho} \\
& = \dfrac{\gamma - \gamma^{2} + \gamma + \gamma^{2}}{1-\gamma} \norm{\vo -
\u} + \dfrac{\gamma(1+\gamma)}{1-\gamma} \norm{\vo - \vho}.  \\
\end{align*} 
}
\end{proof}

The derived upper bound depends on two fixed points:
\u, a fixed point of \E, and \vho, the unique fixed point of $\hT  = \Rd\hEx$. 
Since the latter is defined by \rba\ and \La, the bound is essentially a
function of the factorization terms, as expected. Notice that the bound is valid
for any norm and any fixed point of \E\ (we may think of \u\ as
the closest vector to \vo\ in $\R^{n}$ which satisfies this property). 
Notice also that the first term on the right-hand side of~(\ref{eq:pre_bound})
is exactly the error bound derived in
\ctp{gordon95stable2} Theorem~6.2. When $\La = \Pa$ and
$\ra = \rdda$ for all $a \in A$, the operators $T$
and $\hT$ coincide, and hence the second term of~(\ref{eq:pre_bound}) vanishes.
This makes sense, since in this case $\ddT = \E T$, that is, the 
stochastic-factorization trick reduces
to the averager $A(\v) = \E\v$. 

As mentioned above, one of the assumptions of Lemma~\ref{teo:pre_bound} is that
the factorizations $\E\La = \Pa$ are of order $n$. This is unfortunate, since 
the whole motivation behind the stochastic-factorization trick 
is to create an MDP with $m < n$ states. One way to obtain such a
reduction is to suppose that matrix \E\ has
$n-m$ columns with zeros only. Define $\cE
\subset \{1,2,...,n\}$ as the set of columns of \E\ with at least one nonzero
element and let $\H$ be a matrix in $\R^{m \times n}$ such that $h_{ij} = 1$ if
$j$ is the $i^{\mathrm{th}}$ smallest element in \cE\ and $h_{ij} = 0$
otherwise. The following proposition shows that, based on the
action-value
function of \bM, it is possible to find an approximate solution for the original
MDP whose
distance to the optimal one is also bounded by~(\ref{eq:pre_bound}).

\begin{proposition}
\label{teo:bound_averagers}
Suppose the assumptions of Lemma~\ref{teo:pre_bound} hold. Let 
$\D = \E\H^{\t}$, $\Ka = \H\La$, and $\rba = \H \rdda$, with \H\ defined as
described above. Define the MDP $\bM \equiv(\bar{S},A,\Pba,\rba,\gamma)$, with
$|\bar{S}|=m$
and $\Pba = \Ka \D$. Then,
$\norm{\vo - \Rd \D \Qbo} \le \boundVV$,
with $\boundVV$ defined in~(\ref{eq:pre_bound}).
\end{proposition}
\begin{proof}
Let $\qbao \in \R^{m}$ be the \ath\ column of \Qbo. Then,
\begin{align*}
\D\qbao 
& = \D \left( \rba + \gamma \Pba \vbo \right) 
 = \D \rba + \gamma \D \Ka \D \vbo 
= \E\H^{\t} \H \rdda + \gamma \E\H^{\t} \H\La \E\H^{\t} \vbo \\
& = \E \rdda + \gamma \E \Pdda \H^{\t} \vbo 
  = \E \rdda + \gamma \E \Pdda \vddo 
 = \E \left(\rdda + \gamma \Pdda \vddo \right) 
 = \E \qddao, \\
\end{align*}
where the equality $\E\H^{\t} \H  = \E$ follows from the definition of \H\
and $\Pdda \H^{\t} \vbo = \Pdda \vddo$ is a consequence of the fact that $s_{i}$
is transient if $i \notin \cE$. Therefore,
$
\label{eq:DQ}
\D\Qbo = \E\Qddo.
$
Also, since
$\E\qddao 
= \E \rdda + \gamma \E \La \E \vddo 
= \ra  + \gamma \Pa \E\vddo$,
we know that $\E\Qddo = \Ex \E \vddo$. 
Putting these results together, we obtain 
$\norm{\vo - \Rd \D \Qbo} 
= \norm{\vo - \Rd \Ex \E \vddo} 
= \norm{\vo - T \E \vddo}$, and Lemma~\ref{teo:pre_bound} applies.
\end{proof}

The derived bound can be generalized 
to the case of approximate stochastic factorizations through the triangle
inequality, as done in~(\ref{eq:triangle}). However, if one resorts to
\ctp{whitt78approximations} results to bound the distance between
\vo\ and \vxo---where \vxo\ is the optimal value function of 
$\xM \equiv(S,A,\D\Ka,\D\rba,\gamma)$---the compounded bound will no
longer be valid for all norms, since~(\ref{eq:whitt})
only holds for the infinity norm.

\section{Details of the experiments}
\label{seca:exp_details}

This appendix describes the details of the experiments omitted in the
paper.

\subsection{Tasks}

{\bf Puddle World}: The puddle-world task was implemented as described
by~\ct{sutton96generalization}, but here the task was modeled as a
discounted problem with $\gamma=0.99$. All the transitions were associated with
a zero reward, except those leading to the goal, which resulted in a reward 
of $+5$, and those ending inside one of the puddles, which lead to a penalty
of $-10$ times the distance to the puddle's nearest edge. 
If the agent did not reach the goal after $300$ steps the episode was
interrupted and considered as a failure.
The algorithms were evaluated on two sets of states distributed over
disjoint regions of the state space surrounding the puddles. The first set
was a $3 \times 3$ grid defined over $[0.1,0.3] \times [0.3, 0.5]$ and the
second one was composed of four states: $\{0.1, 0.3\} \times \{0.9, 1.0\}$.

{\bf Pole Balancing}: We implemented the simulator of the three versions
of the pole-balancing task using the equations of motion and parameters given in
the appendix of \ctp{gomez2003robust} PhD thesis. 
For the integration we used the $4^{th}$ order Runge-Kutta method with a
time step of $0.01$ seconds and actions chosen every $2$ time steps. 
We considered the version of the task in which the angle between the
pole and the vertical plane must be kept within 
$[-36^{o},36^{o}]$.
The problem was modeled as a discounted task with $\gamma = 0.99$.
In this formulation, an episode is interrupted and the agent gets a reward of
$-1$ if the pole falls past a 36-degree angle or the
cart reaches the boundaries of the track, located at $2.4\mathrm{m}$
from its center. At all
other steps the agent receives a reward of $0$.
In all versions of the problem an episode was considered a success if the
pole(s) could be balanced for $3\ts000$ steps (one minute of simulated time).
The test set was comprised of $81$ states equally
spaced in the region defined by 
$\pm[1.2\mathrm{m}, 1.2/5 \mathrm{m}, 18^{o}, 75^{o}/s]$, for the single pole
case, and by 
$\pm[1.2\mathrm{m}, 1.2/5 \mathrm{m}, 18^{o}, 75^{o}/s, 18^{o}, 150^{o}/s]$
for the two pole version of the problem. These values 
correspond to a hypercube centered at the origin and covering $50\%$ of the
state-space axes in each dimension (since the velocity of the cart and the
angular velocity of the poles are theoretically not bounded, we defined the
limits of these variables based on samples generated in simple preliminary
experiments).
For the triple pole-balancing task we
performed our simulations using the parameters usually adopted with
the two pole version of the problem, but we added a third pole with
the same length and mass as the longer pole~\cp{gomez2003robust}.
In this case the decision policies were evaluated on a test set 
containing $256$ states equally distributed in the region
$\pm[1.2\mathrm{m}, 1.2/5 \mathrm{m}, 18^{o}, 75^{o}/s, 18^{o}, 
150^{o}/s, 18^{o}, 75^{o}/s]$.

{\bf HIV drug schedule}: The HIV drug schedule task was implemented
using the system of ordinary differential equations (ODEs) given
by~\ct{adams2004dynamic}.
Integration was carried out by the Euler method using a step size of
$0.001$ with actions selected at each $5\ts000$ steps (corresponding to $5$ 
days of simulated time). As suggested by \ct{ersnt2006clinical}, the
problem was modeled as a discounted task with 
$\gamma=0.98$. All other parameters of the task, as well as the protocol used
for the numerical simulations, also followed the suggestions of the same
authors. 
In particular, we assumed the existence of $30$ patients who were monitored for
$1\ts000$ days. 
During the monitoring period, the content of the drug cocktail administered to
each patient could be changed at fixed intervals of $5$ days. 
Thus, in a sample transition $(\xia,\ria,\yia)$: 
$\xia$ is the initial patient condition, $a$ is one of the four
types of cocktails to be administered for the next $5$ days, \yia\ is the
patient condition $5$ days later, and \ria\ is a reward computed based on the
amount of drug in the selected cocktail $a$ and on the difference between the
patient's condition from \xia\ to \yia~\cp{ersnt2006clinical}.
The results reported in Section~\ref{sec:hiv} correspond to the performance of
the greedy policy induced by the value function computed by the algorithms 
using all available sample transitions. 
The decision policies (in this case STI treatments) were evaluated for 
$5\ts000$ days starting from an ``unhealthy'' state corresponding to a basin of
attraction of the ODEs describing the 
problem's dynamics (see the papers by~\ca{adams2004dynamic}
and~\ca{ersnt2006clinical}).

{\bf Epilepsy suppression}: We used a generative model developed
by~\ct{bush2009manifold} to perform our experiments with the epilepsy
suppression
task. The model was generated based on labeled field
potential recordings of five rat brain slices electrically
stimulated at frequencies of $0.0$ Hz, $0.5$ Hz, $1.0$ Hz, and $2.0$ Hz. The
data was used to construct a manifold embedding which in turn gave rise to the
problem's state space. 
The objective
is to minimize the occurrence of seizures using as little stimulation as
possible, therefore there is a negative reward associated with both events (see
Section~\ref{sec:epilepsy}). 
\ca{bush2009manifold}'s generative model is public available as an environment
for the RL-Glue package~\cp{tanner2009rl-glue}.
In our experiments the problem was modeled as a discounted task with 
$\gamma = 0.99$. The decision policies were evaluated on  episodes
of $10^{5}$ transitions starting from a fixed set of $10$ test states
drawn uniformly at random from the problem's state space.

{\bf Helicopter hovering}: In the experiments with the helicopter hovering task
we used the simulator developed by~\ct{abbeel2005learning}, which is available
as an environment for the RL-Glue package~\cp{tanner2009rl-glue}.
The simulator was built based on data collected from two separate flights of a
XCell Tempest helicopter. The data was used to adjust the parameters of an
``acceleration prediction model'', which is more accurate than the linear model 
normally adopted by industry. The objective in the problem is to keep the
helicopter hovering as close as possible to a specific location. Therefore, at
each time step the agent gets a negative reward proportional to the distance
from the target position. 
Since the problem's original action space is $A \equiv [-1,1]^4$, we discretized
each dimension using $4$ break points distributed unevenly over $[-1,1]$.
We tried several possible discretizations and picked the one which resulted in
the best performance of the SARSA agent (see Section~\ref{sec:helicopter}).
After this process, the problem's action space was redefined as 
$A \equiv \{-0.25,-0.05, +0.05, +0.25\}^4$.
The problem was modeled as a discounted task with $\gamma = 0.99$.
The decision policies were evaluated in episodes starting from the target
position and ending when the helicopter crashed.

\subsection{Algorithms}
\label{seca:algorithms}

In all experiments, we used
\begin{equation}
\label{eq:ker_used}
\mk(z) \equiv \mkb(z) \equiv \exp(-z)
\end{equation}
to define the kernels used by KBRL, LSPI, and KBSF.
In the experiments involving a large number of sample transitions we used sparse
kernels, that is, we only computed the \nk\ largest values of
$\gaussa(\rs_{i},\cdot)$ and the \nkb\ largest values of 
$\gaussb(\yia, \cdot)$.
In order to implement this feature, we used a KD-tree to find the $\nk$ (\nkb)
nearest neighbors of $\rs_{i}$ ($\yia$) and only computed \gaussa\ (\gaussb) in
these states~\cp{bentley75multidimensional}. The value of  \gaussa\ and \gaussb\
outside this neighborhood was truncated to zero (we used specialized data
structures to avoid storing those). 

We now list a few details regarding the algorithms's implementations which
were not described in the paper:

\begin{itemize}
 \item 
{\bf KBRL} and {\bf KBSF}: We used modified policy iteration to
compute \Qco~\cp{puterman78modified}. The value function of a fixed
policy $\pi$ was approximated through value iteration using the stop criterion
described by~\ct[Proposition 6.6.5]{puterman94markov} with $\varepsilon =
10^{-6}$. Table~\ref{tab:kbsf_params} shows the parameters's values used by KBSF
across the experiments.

\item
{\bf LSPI}: As explained above, LSPI used the kernel derived
from~(\ref{eq:ker_used}) as its basis function. 
Following~\ct{lagoudakis2003least}, we adopted one block of basis functions for
each action $a \in A$. Singular value decomposition was used to avoid eventual
numerical instabilities in the system of linear equations constructed at each
iteration of LSPI~\cp{golub93matrix}. 

\item
{\bf Fitted $Q$-iteration and extra trees}: FQIT has four main parameters: the
number of iterations, the number of trees composing the ensemble, the number of
candidate cut-points
evaluated during the
generation of the trees, and the minimum number of elements required to split a
node, denoted here \mne. In general, increasing the first three improves
performance, while \mne\ has an inverse relation with the quality of the 
final value function approximation. Our experiments indicate that 
the following configuration of FQIT usually results in good performance
on the tasks considered in this paper:
$50$ iterations (with the structure of the trees fixed after the $10\th$ one),
an ensemble of $30$ trees, and $\dims$ candidate cut points. 
The parameter \mne\ has a particularly strong effect on FQIT's performance and
computational cost, and its correct value seems to be more problem-dependent.
Therefore, in all of our experiments we fixed the parameters of FQIT as
described above and only varied \mne.

\item
{\bf SARSA}: We adopted the implementation of SARSA($\lambda$) available in the
RL-Glue package~\cp{tanner2009rl-glue}. The algorithm uses gradient descent
temporal-difference learning to configure a tile coding function
approximator.

\end{itemize}

\begin{center}
\begin{table}
{\footnotesize
\begin{tabular}{lccccccc}
\hline
 {\bf Problem} & {\bf Section} & $\rs_{i}$ & $m$ & $\tau$ & $\taub$ & $\nk$ &
$\nkb$ \\
\hline 
Puddle  & \ref{sec:puddle} 
& $k$-means & $\{10,30,...,150\}$ & $\{0.01, 0.1, 0.1\}$ & $\{0.01, 0.1, 0.1\}$
& $\infty$ & $\infty$ \\ 
Puddle & \ref{sec:puddle_inc} 
& evenly & $100$ & $\{0.01, 0.1, 0.1\}$ & $\{0.01, 0.1, 0.1\}$
& $\infty$ & $\infty$ \\ 
Single Pole &  \ref{sec:pole}
& $k$-means & $\{10,30,...,150\}$ & $1$ & $\{0.01, 0.1, 0.1\}$
& $\infty$ & $\infty$ \\ 
Two Poles & \ref{sec:pole}
& $k$-means & $\{20,40,...,200\}$ & $1$ & $\{0.01, 0.1, 0.1\}$
& $\infty$ & $\infty$ \\ 
Triple Pole & \ref{sec:triple_pole}
& on-line & on-line & $100^{*}$ & $1^{*}$
& $50^{*}$ & $10^{*}$ \\ 
HIV & \ref{sec:hiv} 
& random & $\{2\ts000,4\ts000,...,10\ts000\}$ & $1$ & $1$
& $2^{*}$ & $3^{*}$ \\ 
Epilepsy & \ref{sec:epilepsy} 
& $k$-means & $50\ts000^{*}$ & $1$ & $\{0.01, 0.1, 0.1\}$ & $6^{*}$ & $6^{*}$
\\ 
Helicopter & \ref{sec:helicopter} 
& $k$-means & $500^{*}$ & $1$ & $1$ & $4^{*}$ & $4^{*}$ \\ 
\hline
\end{tabular}
}
\caption{Parameters used by KBSF on the computational experiments. The values
marked with an asterisk~($^{*}$) were determined 
by trial and error on preliminary
tests. The remaining parameters were kept fixed from the start or
were defined based on a very coarse search.
\label{tab:kbsf_params}}
\end{table}
\end{center}

\long\def\symbolfootnote[#1]#2{\begingroup%
\def\thefootnote{\fnsymbol{footnote}}\footnote[#1]{#2}\endgroup}

\section*{Acknowledgments}
Most of the work described in this technical report 
was done while Andr\'e Barreto was a 
postdoctoral fellow in the School of Computer Science
at McGill University.
The authors would like to thank Yuri Grinberg and Amir-massoud Farahmand 
for valid discussions regarding KBSF and related subjects. We also thank 
Keith Bush for making the epilepsy simulator available, and Alicia Bendz 
and Ryan Primeau for helping in some of the computational experiments.
Funding for this research  was provided by the National
Institutes of Health (grant R21 DA019800) and the NSERC Discovery Grant
program.


\begin{thebibliography}{68}
\providecommand{\natexlab}[1]{#1}
\providecommand{\url}[1]{\texttt{#1}}
\expandafter\ifx\csname urlstyle\endcsname\relax
  \providecommand{\doi}[1]{doi: #1}\else
  \providecommand{\doi}{doi: \begingroup \urlstyle{rm}\Url}\fi

\bibitem[Abbeel et~al.(2005)Abbeel, Ganapathi, and Ng]{abbeel2005learning}
P. Abbeel, V. Ganapathi, and A. Ng.
\newblock Learning vehicular dynamics, with application to modeling
  helicopters.
\newblock In \emph{Adv. in Neural Information Processing Systems ({NIPS})},
2005.

\bibitem[Abbeel et~al.(2007)Abbeel, Coates, Quigley, and
  Ng]{abbeel2007application}
P. Abbeel, A. Coates, M. Quigley, and A. Ng.
\newblock An application of reinforcement learning to aerobatic helicopter
  flight.
\newblock In \emph{Adv. in Neural Information Processing Systems
({NIPS})},~2007.

\bibitem[Adams et~al.(2004)Adams, Banks, Kwon, and Tran]{adams2004dynamic}
B. Adams, H. Banks, H.~Kwon, and H. Tran.
\newblock Dynamic multidrug therapies for {HIV}: optimal and {STI} control
  approaches.
\newblock \emph{Mathematical Biosciences and Engineering}, 1\penalty0
  (2):\penalty0 223--41, 2004.

\bibitem[Anderson(1986)]{anderson86learning}
C. Anderson.
\newblock \emph{Learning and Problem Solving with Multilayer Connectionist
  Systems}.
\newblock PhD thesis, Computer and Information Science, University of
  Massachusetts, 1986.

\bibitem[Antos et~al.(2007)Antos, Munos, and Szepesv{\'a}ri]{antos2007fitted}
A.~Antos, R.~Munos, and {Cs}. Szepesv{\'a}ri.
\newblock Fitted {Q}-iteration in continuous action-space {MDP}s.
\newblock In \emph{Advances in Neural Information Processing Systems ({NIPS})},
2007.

\bibitem[Atkeson and Santamaria(1997)]{atkeson97comparison}
C. Atkeson and J.~Santamaria.
\newblock A comparison of direct and model-based reinforcement learning.
\newblock In \emph{Proc. of the {IEEE} International Conference on
  Robotics and Automation},~1997.

\bibitem[Bajaria et~al.(2004)Bajaria, Webb, and
  Kirschner]{bajaria2004predicting}
S. Bajaria, G. Webb, and D. Kirschner.
\newblock Predicting differential responses to structured treatment
  interruptions during {HAART}.
\newblock \emph{Bulletin of Mathematical Biology}, 66\penalty0 (5):\penalty0
  1093 -- 1118, 2004.

\bibitem[Barreto and Fragoso(2011)]{barreto2011computing}
A. Barreto and M. Fragoso.
\newblock {Computing the Stationary
  Distribution of a Finite {M}arkov Chain Through Stochastic Factorization}.
\newblock \emph{{SIAM} Journal on Matrix Analysis and Applications},
  32:\penalty0 1513--1523, 2011.

\bibitem[Barreto et~al.(2013)Barreto, Pineau, and
Precup]{barreto2012policy}
A. Barreto, J. Pineau, and D. Precup.
\newblock Policy iteration based on stochastic factorization.
\newblock Submitted, 2013.

\bibitem[Barreto et~al.(2011)Barreto, Precup, and
  Pineau]{barreto2011reinforcement}
A. Barreto, D. Precup, and J. Pineau.
\newblock Reinforcement learning using kernel-based stochastic factorization.
\newblock In \emph{Advances in Neural Information Processing Systems (NIPS)},
2011.

\bibitem[Barreto et~al.(2012)Barreto, Precup, and Pineau]{barreto2012online}
A. Barreto, D. Precup, and J. Pineau.
\newblock On-line reinforcement learning using incremental kernel-based
  stochastic factorization.
\newblock In \emph{Advances in Neural Information Processing Systems (NIPS)},
  2012.

\bibitem[Barto et~al.(1983)Barto, Sutton, and Anderson]{barto83neuronlike}
A. Barto, R. Sutton, and C. Anderson.
\newblock Neuronlike adaptive elements that can solve difficult learning
  control problems.
\newblock \emph{IEEE Transactions on Systems, Man, and Cybernetics},
  13:\penalty0 834--846, 1983.

\bibitem[Bellman(1957)]{bellman57dynamic}
R.~E. Bellman.
\newblock \emph{Dynamic Programming}.
\newblock Princeton University Press, 1957.

\bibitem[Bentley(1975)]{bentley75multidimensional}
J. Bentley.
\newblock Multidimensional binary search trees used for associative searching.
\newblock \emph{Communications of the {ACM}}, 18\penalty0 (9):\penalty0
  509--517, 1975.

\bibitem[Bertsekas and Tsitsiklis(1996)]{bertsekas96neuro-dynamic}
D. Bertsekas and J. Tsitsiklis.
\newblock \emph{Neuro-Dynamic Programming}.
\newblock Athena Scientific, 1996.

\bibitem[Beygelzimer et~al.(2006)Beygelzimer, Kakade, and
  Langford]{beygelzimer06cover}
A. Beygelzimer, S. Kakade, and J. Langford.
\newblock Cover trees for nearest neighbor.
\newblock In \emph{Proceedings of the International Conference on Machine
  Learning (ICML)}, 2006.

\bibitem[Bhat et~al.(2012)Bhat, Moallemi, and Farias]{bhat2011nonparametric}
N. Bhat, C. Moallemi, and V. Farias.
\newblock Non-parametric approximate dynamic programming via the kernel method.
\newblock In \emph{Adv. in Neural Information Processing Systems ({NIPS})},
  2012.

\bibitem[Brafman and Tennenholtz(2003)]{brafman2003rmax}
R.~I. Brafman and M. Tennenholtz.
\newblock R-{MAX}: a general polynomial time algorithm for near-optimal
  reinforcement learning.
\newblock \emph{Journal of Machine Learning Research}, 3:\penalty0
213--231,~2003.

\bibitem[Bush et~al.(2009)Bush, J., and M.]{bush2009manifold}
K.~Bush, Pineau J., and M. Avoli
\newblock Manifold embeddings for model-based reinforcement learning of
  neurostimulation policies.
\newblock In \emph{Proceedings of the ICML/UAI/COLT Workshop on Abstraction in
  Reinforcement Learning}, 2009.

\bibitem[Bush and Pineau(2009)]{bush2009manifold2}
K. Bush and J. Pineau.
\newblock Manifold embeddings for model-based reinforcement learning under
  partial observability.
\newblock In \emph{Adv. in Neural Information Processing Systems
  (NIPS)}, 2009.

\bibitem[Cohen and Rothblum(1991)]{cohen91nonnegative}
J. Cohen and U. Rothblum.
\newblock Nonnegative ranks, decompositions and factorizations of nonnegative
  matrices.
\newblock \emph{Linear Algebra and its Applications}, 190:\penalty0 149--168,
  1991.

\bibitem[Cutler and Breiman(1994)]{cutler94archetypal}
A. Cutler and L. Breiman.
\newblock Archetypal analysis.
\newblock \emph{Technometrics}, 36\penalty0 (4):\penalty0 338--347, 1994.

\bibitem[Denardo(1967)]{denardo67contraction}
E. Denardo.
\newblock Contraction mappings in the theory underlying dynamic programming.
\newblock \emph{SIAM Review}, 9\penalty0 (2):\penalty0 165--177, 1967.

\bibitem[Dietterich(2000)]{dietterich2000hierarchical}
T. Dietterich.
\newblock Hierarchical reinforcement learning with the {MAXQ} value function
  decomposition.
\newblock \emph{Journal of Artificial Intelligence Research}, 13:\penalty0
  227--303, 2000.

\bibitem[Durand and Bikson(2001)]{durand2001supression}
D. Durand and M.~Bikson.
\newblock Suppression and control of epileptiform activity by electrical
  stimulation: a review.
\newblock \emph{Proceedings of the IEEE}, 89\penalty0 (7):\penalty0 1065
  --1082, 2001.

\bibitem[Engel et~al.(2005)Engel, Mannor, and Meir]{engel2005reinforcement}
Y. Engel, S. Mannor, and R. Meir.
\newblock Reinforcement learning with {G}aussian processes.
\newblock In \emph{Proceedings of the International Conference on Machine
  learning (ICML)}, 2005.

\bibitem[Ernst et~al.(2006)Ernst, Stan, Gon\c{c}alves, and
  Wehenkel]{ersnt2006clinical}
D.~Ernst, G. Stan, J.~Gon\c{c}alves, and L.~Wehenkel.
\newblock Clinical data based optimal {STI} strategies for {HIV}: a
  reinforcement learning approach.
\newblock In \emph{Proceedings of the {IEEE} Conference on Decision and
  Control (CDC)}, 2006.

\bibitem[Ernst et~al.(2005)Ernst, Geurts, and Wehenkel]{ernst2005tree}
D. Ernst, P. Geurts, and L. Wehenkel.
\newblock Tree-based batch mode reinforcement learning.
\newblock \emph{Journal of Machine Learning Research}, 6:\penalty0 503--556,
  2005.

\bibitem[Farahmand(2011)]{farahmand2011regularization}
A. Farahmand.
\newblock \emph{Regularization in reinforcement learning}.
\newblock PhD thesis, Univ. of Alberta,~2011.

\bibitem[Geurts et~al.(2006)Geurts, Ernst, and Wehenkel]{geurts2006extremely}
P. Geurts, D. Ernst, and L. Wehenkel.
\newblock Extremely randomized trees.
\newblock \emph{Machine Learning}, 36\penalty0 (1):\penalty0 3--42, 2006.

\bibitem[Golub and Loan(1993)]{golub93matrix}
G. Golub and C.~Van Loan.
\newblock \emph{Matrix Computations}.
\newblock Johns Hopkins University Press, second edition, 1993.

\bibitem[Gomez et~al.(2006)Gomez, Schmidhuber, and
  Miikkulainen]{gomez2006efficient}
F.~Gomez, J.~Schmidhuber, and R.~Miikkulainen.
\newblock Efficient non-linear control through neuroevolution.
\newblock In \emph{Proceedings of the European Conference on Machine
  Learning}, 2006.

\bibitem[Gomez(2003)]{gomez2003robust}
F. Gomez.
\newblock \emph{Robust non-linear control through neuroevolution}.
\newblock PhD thesis, The University of Texas at Austin, 2003.
\newblock Technical Report {AI-TR-03-303}.

\bibitem[Gordon(1995)]{gordon95stable2}
G. Gordon.
\newblock Stable function approximation in dynamic programming.
\newblock Technical Report CMU-CS-95-103, Computer Science Department, Carnegie
  Mellon University, 1995.

\bibitem[Grunewalder et~al.(2012)Grunewalder, Lever, Baldassarre, Pontil, and
  Gretton]{grunewalder2012modelling}
S. Grunewalder, G. Lever, L. Baldassarre, M. Pontil, and A.
  Gretton.
\newblock Modelling transition dynamics in {MDPs} with {RKHS} embeddings.
\newblock In \emph{Proceedings of the
  International Conference on Machine Learning ({ICML})}, 2012. 

\bibitem[Hastie et~al.(2002)Hastie, Tibshirani, and
  Friedman]{hastie2002elements}
T. Hastie, R. Tibshirani, and J. Friedman.
\newblock \emph{The Elements of Statistical Learning: Data Mining, Inference,
  and Prediction}.
\newblock Springer, 2002.

\bibitem[Jerger and Schiff(1995)]{jerger95periodic}
K.~Jerger and S. Schiff.
\newblock Periodic pacing an in vitro epileptic focus.
\newblock \emph{Journal of Neurophysiology}, \penalty0 (2):\penalty0 876--879,
  1995.

\bibitem[Jong and Stone(2006)]{jong2006kernelbased}
N. Jong and P. Stone.
\newblock Kernel-based models for reinforcement learning in continuous state
  spaces.
\newblock In \emph{Proceedings of the International Conference on Machine
  Learning---Workshop on Kernel Machines and Reinforcement Learning}, 
  2006.

\bibitem[Jong and Stone(2009)]{jong2009compositional}
N. Jong and P. Stone.
\newblock Compositional models for reinforcement learning.
\newblock In \emph{Proc. of the European Conference on Machine Learning
  and Knowledge Discovery in Databases},  2009. 

\bibitem[Kaufman and Rousseeuw(1990)]{kaufman90finding}
L. Kaufman and P. Rousseeuw.
\newblock \emph{Finding Groups in Data: an Introduction to Cluster Analysis}.
\newblock John Wiley and Sons, 1990.

\bibitem[Kroemer and Peters(2011)]{kroemer2011nonparametric}
O. Kroemer and J. Peters.
\newblock A non-parametric approach to dynamic programming.
\newblock In \emph{Advances in Neural Information Processing Systems
({NIPS})}, 2011.

\bibitem[Kveton and Theocharous(2012)]{kveton2012kernel}
B. Kveton and G. Theocharous.
\newblock Kernel-based reinforcement learning on representative states.
\newblock In \emph{Proceedings of the {AAAI} Conference on Artificial
  Intelligence}, 2012.

\bibitem[Lagoudakis and Parr(2003)]{lagoudakis2003least}
M. Lagoudakis and R. Parr.
\newblock Least-squares policy iteration.
\newblock \emph{Journal of Machine Learning Research}, 4:\penalty0 1107--1149,
  2003.

\bibitem[Michie and Chambers(1968)]{michie68boxes}
D.~Michie and R.~Chambers.
\newblock {BOXES}: An experiment on adaptive control.
\newblock \emph{Machine Intelligence 2}, pages 125--133, 1968.

\bibitem[Moore and Atkeson(1993)]{moore93prioritized}
A. Moore and C. Atkeson.
\newblock Prioritized sweeping: Reinforcement learning with less data and less
  time.
\newblock \emph{Machine Learning}, 13:\penalty0 103--130, 1993.

\bibitem[Munos and Szepesv\'ari(2008)]{munos2008finite}
R. Munos and Cs. Szepesv\'ari.
\newblock Finite-time bounds for fitted value iteration.
\newblock \emph{Journal of Machine Learning Research}, 9:\penalty0 815--857,
  2008.

\bibitem[Ng et~al.(2003)Ng, Kim, Jordan, and Sastry]{ng2003autonomous}
A. Ng, H. Kim, M. Jordan, and S. Sastry.
\newblock Autonomous helicopter flight via reinforcement learning.
\newblock In \emph{Advances in Neural Information Processing Systems ({NIPS})},
  2003.

\bibitem[Ormoneit and Glynn(2002)]{ormoneit2002kernelbased2}
D.~Ormoneit and P.~Glynn.
\newblock Kernel-based reinforcement learning in average-cost problems.
\newblock \emph{IEEE Transactions on Automatic Control}, 47\penalty0
  (10):\penalty0 1624--1636, October 2002.

\bibitem[Ormoneit and Sen(2002)]{ormoneit2002kernelbased}
D.~Ormoneit and S.~Sen.
\newblock Kernel-based reinforcement learning.
\newblock \emph{Machine Learning}, 49 (2--3):\penalty0 161--178, 2002.

\bibitem[Puterman(1994)]{puterman94markov}
M. Puterman.
\newblock \emph{Markov Decision Processes---Discrete Stochastic Dynamic
  Programming}.
\newblock John Wiley \& Sons, Inc., 1994.

\bibitem[Puterman and Shin(1978)]{puterman78modified}
M. Puterman and M. Shin.
\newblock Modified policy iteration algorithms for discounted {M}arkov decision
  problems.
\newblock \emph{Management Science}, pages 1127--1137, 1978.

\bibitem[Rasmussen and Kuss(2004)]{rasmussen2004gaussian}
C. Rasmussen and M. Kuss.
\newblock Gaussian processes in reinforcement learning.
\newblock In \emph{Advances in Neural Information Processing Systems (NIPS)},
2004.

\bibitem[Ravindran(2004)]{ravindran2004algebraic}
B.~Ravindran.
\newblock \emph{An Algebraic Approach to Abstraction in Reinforcement
  Learning}.
\newblock PhD thesis, University of Massachusetts, Amherst, MA, 2004.

\bibitem[Rummery and Niranjan(1994)]{rummery94on-line}
G.~Rummery and M.~Niranjan.
\newblock On-line {Q}-learning using connectionist systems.
\newblock Technical Report CUED/F-INFENG/TR 166, Cambridge
  University, 1994.

\bibitem[Sch\"{o}lkopf and Smola(2002)]{scholkopf2002learning}
B.~Sch\"{o}lkopf and A. Smola.
\newblock \emph{Learning with Kernels}.
\newblock MIT Press, 2002.

\bibitem[Sorg and Singh(2009)]{sorg2009transfer}
J. Sorg and S. Singh.
\newblock Transfer via soft homomorphisms.
\newblock In \emph{Autonomous Agents \& Multiagent Systems/Agent Theories,
  Architectures, and Languages}, 2009.

\bibitem[Strehl and Littman(2008)]{strehl2008analysis}
A. Strehl and M. Littman.
\newblock An analysis of model-based interval estimation for {M}arkov decision
  processes.
\newblock \emph{Journal of Computer and System Sciences}, 74\penalty0
  (8):\penalty0 1309--1331, 2008.

\bibitem[Sutton(1996)]{sutton96generalization}
R. Sutton.
\newblock Generalization in reinforcement learning: Successful examples using
  sparse coarse coding.
\newblock In \emph{Advances in Neural Information Processing Systems}, 1996.

\bibitem[Sutton and Barto(1998)]{sutton98reinforcement}
R. Sutton and A. Barto.
\newblock \emph{Reinforcement Learning: An Introduction}.
\newblock MIT Press, 1998.

\bibitem[Szepesv{\'a}ri(2010)]{szepesvari2010algorithms}
Cs. Szepesv{\'a}ri.
\newblock \emph{Algorithms for Reinforcement Learning}.
\newblock Synthesis Lectures on Artificial Intelligence and Machine Learning.
  Morgan {\&} Claypool Publishers, 2010.

\bibitem[Tanner and White(2009)]{tanner2009rl-glue}
B. Tanner and A. White.
\newblock {RL-Glue}: Language-independent software for reinforcement-learning
  experiments.
\newblock \emph{Journal of Machine Learning Research}, 10:\penalty0 2133--2136,
  2009.

\bibitem[Taylor and Parr(2009)]{taylor2009kernelized}
G. Taylor and R. Parr.
\newblock Kernelized value function approximation for reinforcement learning.
\newblock In \emph{Proceedings of the International Conference on Machine
  Learning ({ICML})}, 2009. 

\bibitem[Vavasis(2009)]{vavasis2009complexity}
S. Vavasis.
\newblock On the complexity of nonnegative matrix factorization.
\newblock \emph{{SIAM} Journal on Optimization}, 20:\penalty0 1364--1377, 2009.

\bibitem[Whitt(1978)]{whitt78approximations}
W. Whitt.
\newblock Approximations of dynamic programs, {I}.
\newblock \emph{Mathematics of Operations Research}, 3\penalty0 (3):\penalty0
  231--243, 1978.

\bibitem[Wieland(1991)]{wieland91evolving}
A. Wieland.
\newblock Evolving neural network controllers for unstable systems.
\newblock In \emph{Proceedings of the International Joint Conference on Neural
  Networks}, 1991.

\bibitem[Xu et~al.(2005)Xu, Xie, Hu, and Lu]{xu2005kernel}
X. Xu, T. Xie, D. Hu, and X. Lu.
\newblock {Kernel Least-Squares Temporal Difference Learning}.
\newblock \emph{Information Technology}, pages 54--63, 2005.

\bibitem[Ye(2011)]{ye2011simplex}
Y. Ye.
\newblock The simplex and policy-iteration methods are strongly polynomial for
  the {M}arkov decision problem with a fixed discount rate.
\newblock \emph{Mathematics of Operations Research}, 36\penalty0 (4):\penalty0
593--603, 2011.

\end{thebibliography}

\end{document}